\documentclass{report}

\usepackage{algorithm2e}
\usepackage{latexsym}
\usepackage{amssymb}
\usepackage{epsfig}
\usepackage{subfig}
\usepackage{geometry}
\usepackage{rotating}
\usepackage{titlesec}

\usepackage{graphicx}
\usepackage{url}
\usepackage{textcomp}
\usepackage{color}
\usepackage[all]{xy}
\usepackage{appendix}
\usepackage{multirow}
\usepackage{float}
\usepackage{tabularx}				
\usepackage[some]{background}		
\usepackage{amssymb}  
\usepackage{amsmath}   

\usepackage{amsthm} 

\definecolor{red}{rgb}{0.8,0.2,0.2}
\definecolor{blue}{rgb}{0,0,0.5}
\definecolor{green}{rgb}{0,0.7,0}
\definecolor{violet}{rgb}{0.5,0.2,0.5}
\definecolor{orange}{rgb}{0.8,0.5,0.2}

\usepackage{amssymb}  
\usepackage{amsmath}   

\usepackage{amsthm} 
\newtheorem{theorem}{Theorem}
\newtheorem{proposition}[theorem]{Proposition}

\newtheorem{example}{Example}
\newtheorem{definition}{Definition}

\usepackage{color}
\definecolor{orange}{RGB}{255,127,0}
\definecolor{brown}{RGB}{150,70,0}
\definecolor{green}{RGB}{127,255,127}
\definecolor{darkgreen}{RGB}{0,127,0}
\definecolor{blue}{RGB}{127,127,255}
\definecolor{lightblue}{RGB}{150,150,255}
\definecolor{darkblue}{RGB}{0,0,127}
\definecolor{red}{RGB}{255,90,90}
\definecolor{grey}{RGB}{127,127,127}
\definecolor{pink}{RGB}{255,180,180}

\usepackage{listings}
\usepackage{color}

\definecolor{dkgreen}{rgb}{0,0.6,0}
\definecolor{gray}{rgb}{0.5,0.5,0.5}
\definecolor{mauve}{rgb}{0.58,0,0.82}

\newcommand{\comment}[1]{}

\usepackage{graphicx}

\usepackage{xspace}	
\newcommand{\xaxis}{$x$-axis\xspace}
\newcommand{\yaxis}{$y$-axis\xspace}

\newcommand{\JROC}{JROC\xspace}

\titleformat{\chapter}{\normalfont\huge\bfseries}{\thechapter}{20pt}{\huge}
\titleformat{\section} {\normalfont\Large\bfseries}{\thesection}{1em}{} 
\titleformat{\subsection}{\normalfont\large\bfseries}{\thesubsection}{1em}{} 

\long\def\comment#1{}

\author{MAGUEDONG DJOUMESSI Celestine Periale}
\title{MODEL RE-FRAMING BY FEATURE CONTEXT CHANGE}

\begin{document}

\renewcommand{\contentsname}{Contents}
\renewcommand{\tablename}{Table}
\renewcommand{\listtablename}{List of Tables}
\renewcommand{\figurename}{Figure}
\renewcommand{\listfigurename}{List of Figures}
\renewcommand{\appendixname}{Appendices} 
\renewcommand{\appendixtocname}{Appendices} 
\renewcommand{\appendixpagename}{Appendices}

\SetBgScale{1}
\SetBgContents{
\includegraphics[scale=1.5]{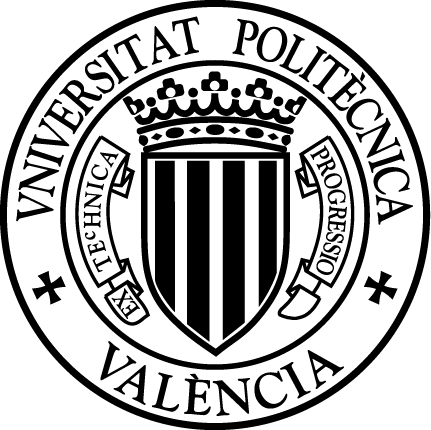}}
\SetBgColor{gray}
\SetBgAngle{0}
\SetBgPosition{8cm,-7cm}
\SetBgOpacity{0.07}

\begin{titlepage}
\begin{center}
	\begin{center}
	
    \BgThispage			
    
    \begin{table}[h]
            \begin{tabularx}{\textwidth}{cXc}
                \includegraphics[width=6cm]{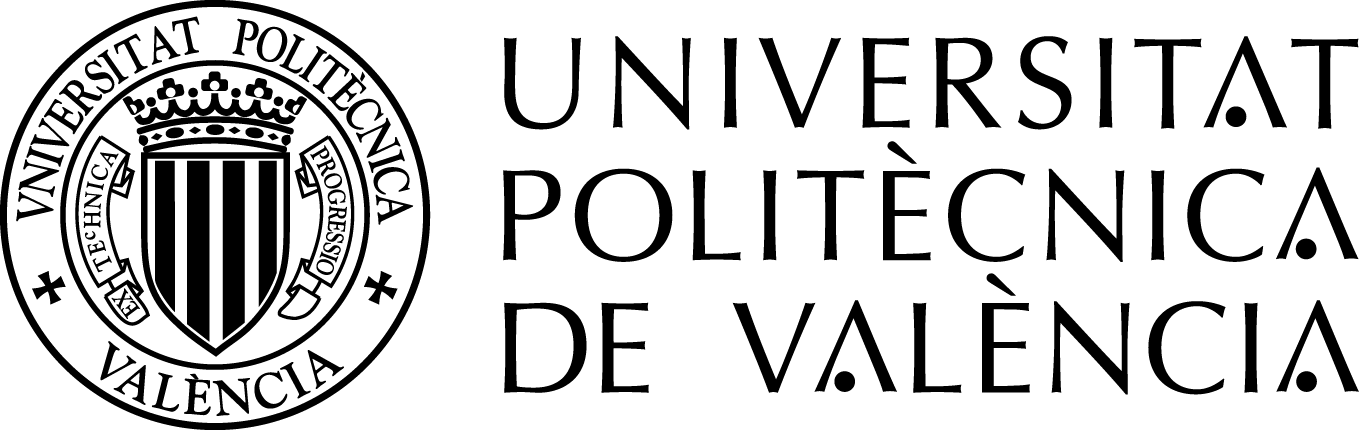} & & \includegraphics[width=7cm]{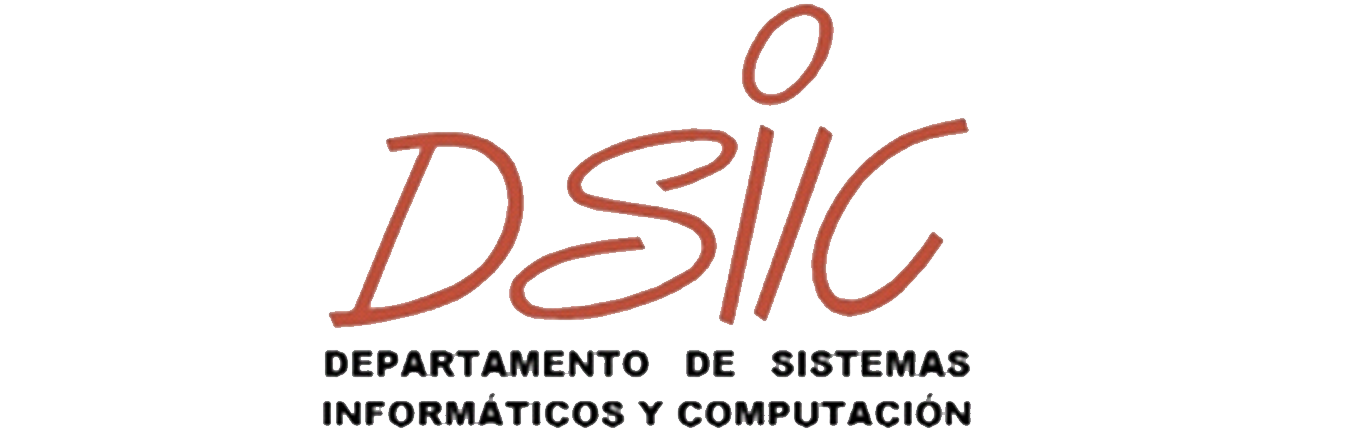} 
            \end{tabularx}
    \end{table}
	\end{center}
    \vspace{1cm}

    {\LARGE \textit{M\'aster Universitario en Ingenier\'ia del Software, M\'etodos Formales y Sistemas de Informaci\'on}}
    
    \vspace{0.5cm}
    
    {\LARGE \textit{MSc in Software Engineering, Formals Methods and Information Systems}}
    \vspace{1cm}
    
\begin{center}
\rule{1\textwidth}{1pt}
\end{center}

\LARGE
MODEL REFRAMING BY FEATURE CONTEXT CHANGE\\

\begin{center}
\rule{1\textwidth}{1pt}
\end{center}

Maguedong Djoumessi Celestine Periale\\

\vspace{0.5cm}

Dirigida por: / Supervised by:\\

\vspace{0.5cm}

Jos\'e Hern\'andez Orallo\\

\vspace{2cm}

Valencia, Julio de 2013 / Valencia, July 2013\\

\end{center}

\end{titlepage}

 \thispagestyle{empty}

\vspace*{1cm}

\begin{flushright}

In love memory of my Mum:\\
Donfack Marie Julienne

\end{flushright}

\vspace*{\fill}

\chapter*{Abstract }
\setlength{\parskip}{\baselineskip}

\setcounter{page}{1} 
\pagenumbering{roman}

Many solutions to cost-sensitive classification (and regression) rely on some or all of the following assumptions: we have complete knowledge about the cost context at training time, we can easily re-train whenever the cost context changes, and we have technique-specific methods (such as cost-sensitive decision trees) that can take advantage of that information. In this work we address the problem of selecting models and minimising joint cost (integrating both misclassification cost and test costs) without any of the above assumptions. We introduce methods and plots (such as the so-called \JROC plots) that can work with any off-the-shelf predictive technique, including ensembles, such that we re-frame the model to use the appropriate subset of attributes (the feature configuration) during deployment time. In other words, models are trained with the available attributes (once and for all) and then deployed by setting missing values on the attributes that are deemed ineffective for reducing the joint cost. As the number of feature configuration combinations grows exponentially with the number of features we introduce quadratic methods that are able to approximate the optimal configuration and model choices, as shown by the experimental results.

\vspace{1cm}

\noindent \textbf{Keywords:} test cost, misclassification cost, missing values, re-framing, ROC analysis, operating context, feature configuration, feature selection.

\chapter*{Resumen }

Muchas de las soluciones para la clasificaci\'on y regresi\'on sensible al coste se basan en alguna de las siguientes hip\'otesis: que tenemos un conocimiento completo sobre el contexto de coste en tiempo de entrenamiento, que podemos volver a entrenar con facilidad cada vez que cambia el contexto de costes, y que tenemos los  m\'etodos para una t\'ecnica especifica (tales como \'arboles de decisiones sensibles a los costes) que pueden aprovechar esa informaci\'on. En este trabajo se aborda el problema de la selecci\'on de modelos y la minimizaci\'on de los costes conjuntos (integrando tanto el coste de clasificaci\'on err\'onea como los costes de pruebas de atributos) sin ninguno de los supuestos anteriores. Introducimos m\'etodos y gr\'aficos (como los gr\'aficos \JROC) que pueden funcionar con cualquier t\'ecnica predictiva com\'un, incluyendo {\em ensembles}, de tal manera que nos adapta el modelo para el subconjunto apropiado de atributos (la configuraci\'on de los atributos) durante el tiempo de despliegue. En otras palabras, los modelos son entrenados con los atributos disponibles (una vez y para siempre) y luego desplegados mediante el establecimiento de valores faltantes en los atributos que se consideran ineficaces para reducir el conjunto. Como el n\'umero de combinaciones de los atributos crece exponencialmente con el n\'umero de atributos se introducir\'an m\'etodos cuadr\'aticos que son capaces de aproximar  la opci\'on de configuraci\'on \'optima y el modelo \'optimo, como se muestra con los resultados experimentales.

\vspace{1cm}

\noindent \textbf{Keywords:} costes de test de atributos, costes de clasificaci\'on erronea, valores faltantes, reframing, an\'alisis {ROC}, contexto de operaci\'on, configuraci\'onn de atributos. 
\chapter*{Remerciements}
\setlength{\parskip}{\baselineskip}

Au terme de ce si riche parcours, je tiens \`a rendre tout d'abord grace \`a Dieu sans qui rien n'aurait \'et\'e possible et aussi \`a  remercier toutes ces personnes qui ont su m'accompagner et m'apporter tout le soutien morale, physique et financier necessaire pour y arriver. Je pense ici \`a:
\begin{itemize}
\item L'ame de ma tr\'es ch\`ere maman DONFACK Marie-Julienne, qui a donn\'e jusqu'\`a la derni\`ere de ses forces pour que je puisse recevoir une \'education de qualit\'e. Merci, mille merci maman!
\item Mon papa et fr\`eres et soeurs qui ont \'et\'e d'un tres grand soutien moral \`a chaque fois que j'en avais besoin.
\item Mon parrain et pa'a Jose Antonio Varela Ferrandis qui m'a accompagn\'e et ``proteg\'e'' tout au long de mon sejour \`a Valencia.
\item A toutes mes mamis et papis de Benimaclet qui par leur amour mon procur\'e une veritable famille en Espagne. 
\end{itemize}
A tous ceux-ci ainsi qu'\`a tous les anonymes qui te pres ou de loin et d'une mani\`ere ou d'une autre ont contribu\'e \`a ce que je puisse arriver au terme de ce parcours, je dirige mes sinc\`eres remerciements. 
\vspace{-0.6cm}

\chapter*{Agradecimientos}

\setlength{\parskip}{\baselineskip}

\vspace{-0.6cm}

Al final de este curso tan rico, quiero en primer lugar dar las gracias a Dios, sin el nada hubiera sido posible, y tambi\'en agradecer a todas aquellas personas que pudieron que me han acompa\~nado y me han dado todo el apoyo moral, físico y financiero necesario para llegar hasta aqu\'i. Pienso:
\begin{itemize}
\item En el alma de mi muy querida madre Donfack Marie Julienne, que entreg\'o hasta el \'ultimo extremo de su fuerza para que puediera recibir una educaci\'on de calidad. Gracias, mil gracias mam\'a!
\item A mi padre y mis hermanos y hermanas que han sido de un gran apoyo moral cada vez que lo necesitaba.
\item Mi padrino y pa'a Jos\'e Antonio Varela Ferrandis que me ha acompa\~nado y protegido durante toda mi estancia en Valencia.
\item A todos mis mamis y papis de Benimaclet que por su amor me han ofrecido una familia de verdad en Espa\~na.
\vspace{-0.1cm}
\item A Jos\'e Hern\'andez Orallo, que con la calidez de sus guiones, su paciencia y devoci\'on me ha dirigido durante todo el desarollo de mi tesis.
\item A la Universidad Politecnica de Valencia, a todos mis profesores. Con la calidez de su ense\~nanza, salgo de este curso grata y llena de nuevos conocimientos. 
\item A mis compa\~neros de la universidad y de casa, que me han ayudado en adaptarme y integrarme en la cultura espa\~nola.
\end{itemize}
A todos ellos, y todos los que no he nombrado que de lejos o de cerca, de una manera u otra, han contribuido a lo que puedo llegar hasta el final de este curso, dirijo mis m\'as sinceros agradecimientos. 
\vspace{-0.6cm}

\vspace*{1cm}

\begin{flushright}

Celestine P. MAGUEDONG D.\\
Valencia, 2013

\end{flushright}

\tableofcontents

\listoftables
\listoffigures

\chapter{Introduction}

\setlength{\parskip}{\baselineskip}

\pagenumbering{arabic}
\setcounter{page}{1} 

\vspace{0.2cm}

The main statement of this thesis is about machine learning (ML)/data mining (DM) --- especially about cost-sensitive learning--- which consist in extracting useful knowledge from large and detailed collections of data. Because of the wide range of ways in which companies and institutions can easily, efficiently and cheaply collect data (e.g. web, social networks \dots), this subject has become of increasing importance and interest. This interest has inspired a rapidly developing research field with developments both on a theoretical, as well as on a practical level with the availability of many commercial and non-commercial tools. In fact, data mining has many advantages across different industries, allowing large historical data to be used as the background for prediction. The interpretation and evaluation of patterns obtained by data mining produces new knowledge that decision-makers can use. Thus, ML/DM provide a way to extract knowledge and obtain information that can be useful for prediction and be used as a support in a decision making.
Unfortunately, in data mining approaches, two important assumptions slow the good application of those technology. first the context in which data are obtained can change, for instance, when the features used for describing a patient in a hospital, can be different in one year to the other. Second, the cost to obtain those data can be very high sometimes. As a result, many approaches produce knowledge which is not well adapted to the context or that is not exhaustive.
Those limitations have generated a relatively recent interest in richer and data mining approaches in order to improve, more, optimise the result when contexts change.

\section{Motivation}

Reuse of learnt knowledge is of critical importance in the majority of knowledge-intensive application areas, particularly because the operating context can be expected to vary from training to deployment. In machine learning this is most commonly studied in relation to variations in class and cost skew in classification. While this is evidently useful in many practical situations, there is a clear and pressing need to generalise the notion of operating context beyond the narrow framework of skew-sensitive and cost-sensitive classification. The approach is based around the new notion of model reframing, which can be applied to inputs (features), outputs (predictions) or parts of models (patterns), in this way generalising, integrating and broadening the more traditional and diverse notions of model adjustment in machine learning and data mining. These ideas have led to the following project:

\begin{center} 
\url{www.reframe-d2k.org}
\end{center}

One kind of context in the above project is related to the way inputs (i.e., features) can vary from training to deployment. Among these changes, we can mention that some attributes can disappear or may have different application costs.

For instance, a model $M$ can be learnt from a dataset features $X= (x_1, x_2, \dots, x_m)$ and output attribute Y. On deployment time, we may find that we only have a (possible much smaller) dataset of features $X$\textquotesingle = ($x_1$\textquotesingle, $x_2$\textquotesingle, \dots , $x_m$\textquotesingle). If all the features in $X\textquotesingle$ are also in $X$, there is no problem for applying $M$. But even in this case, some attributes may have higher costs than others on deployment time. On other occasions some features may be not present. In these cases we may use some information relating attributes in $X$ with attributes in $X$\textquotesingle. For instance, $x_2$ may not be present in $X$\textquotesingle, but we may know that the correlation of $x_2$ and $X\textquotesingle$ is high, so one could be used instead of the other. For nominal attributes, we may have functional dependencies or association rules. Also, we may have some kind of hierarchy of $X$ using Principal Component Analysis, or other kind of information that helps us to relate $X$ and $X$\textquotesingle, to be able to apply $M$ with $X$\textquotesingle.
This information can be used in the reframing process or can be embedded in M as a more versatile model, possibly taking a more general notion of feature, or using attribute selection lattices.
A simpler approach is to consider that those features of $X$ which are not in $X\textquotesingle$ can be just processed as missing values. In fact, we can also use missing values for existing attributes if the attributes test cost is too high compared to what we gain by using the attribute.

\vspace{0.3cm}
\section{Preview and others approaches}

The feature space (including both input and output variables) characterises a data mining problem \cite{li2001feature}. In predictive (supervised) problems, the quality and availability of features determines the predictability of the dependent variable, and the performance of data mining models in terms of misclassifcation or regression error. Good features, however, are usually difficult to obtain. It is usual that many instances come with missing values, either because the actual value for a given attribute was not available or because it was too expensive (e.g., in medical domains, where attributes usually correspond to diagnostic tests). This is usually interpreted as a utility or cost-sensitive learning dilemma \cite{turney2000types,elkan2001foundations}, in this case between misclassification (or regression error) costs and attribute tests costs. Both misclassification cost (MC) and test cost (TC) can be integrated into a single measure, known as joint cost (JC).

One possible option to affront this dilemma is known as missing value imputation \cite{zhu2011missing}, but this approach is not usually appropriate when test costs are considered. First, expensive attributes (e.g., in diagnosis) are usually missing for many other instances as well and it is difficulty to infer from other instances or attributes. Second, imputing missing values 
``is regarded as unnecessary for cost-sensitive learning that also considers the test costs'' \cite{zhang2005missing}.

The most common option is usually to train models that are able to do reasonably good predictions with the available attributes or to find a trade-off (in terms of minimising joint cost) about how many (and which) attributes need to be used. Retraining with the necessary attributes (or with feature selection)
 does not seem to be a good option, because for $n$ attributes we typically have $2^n$ possible combinations. Also, only using the training instances for which the same subset of attributes is available can bias each `partial' model. As a result, one common option is to use techniques that lead to models that can use any subset of attributes. Decision trees are the usual choice \cite{ling2004decision,zhang2005missing,lomax2013survey} because the use of attributes can be customised in many different ways to (for example taking to consideration missing values without imputing them or considering misclassification cost). But this alternative limits the option of machine learning techniques to be used (from now limited to the decision trees). We could also try --- if not already done --- to design cost-sensitive versions for many other families of techniques, such as Bayesian models, neural networks, logistic regression, kernel methods, etc., with varying success. This would lead to two problems. On one hand, we would need to have a library of specific cost-sensitive algorithms for classification and regression, which would also limit our range of options and the use of the ultimate learning techniques (until cost-sensitive versions appear and are implemented). On the other hand, even assuming that this is possible, we would require some tools to properly select which model is better, as we do not know in advance what the misclassification (or regression error) cost and test cost configuration is. In fact, each instance may have a different subset of missing values and a different cost configuration, so this choice must be very specific.

In this work, we explore an alternative, more general approach using off-the-shelf machine learning methods. The procedure is simple: we use any data mining technique that accepts missing values during training and prediction and learn a predictive model with our training data as usual. Then, 
we evaluate the model (on a validation dataset) by exploring the lattice of attribute subsets, with a very straightforward mechanism: we set missing values on purpose for each combination in the lattice. From here, we know how well our model behaves for any attribute subset. 
From here, once the model needs to be deployed on unlabelled data, whenever a new instance appears (with a possibly particular cost configuration) we decide which attributes the model requires to get the lowest expected joint cost\footnote{Given an example with some non-missing and some missing values we may decide increase the number of non-missing values. Given a case for which we have not still retrieved any of the attributes (tests) we decide how many (and which) attributes we are going to ask for.}. This is done by calculating the expected joint cost for each point in the lattice. In this sense, each prediction is associated with a possibly different operating condition, and the best attribute subset is chosen.

Interestingly, we can use the previous approach for more than one model, and see that some models dominate for some operating conditions over the rest. This is exactly the way ROC analysis works (for classification \cite{SDM00,flach2003decision,rocai2004,Fawcett06,Mamitsuka2006} and for regression \cite{RROC2013}).

\vspace{0.2cm}

\section{Objectives}

\vspace{0.2cm}
 
The goal of our investigation is then to introduce new methods to make optimal choices in terms of the joint cost when using off-the-shelf data mining models. An optimal choice is understood as selecting the right model with the right subset of attributes given a cost context. 

We will introduce graphical plots and procedures to make this selection and also to reduce the number of combinations in the lattice that need to be explored in order to make a good selection.

In order to improve those objectives, several sub-objectives should be accomplished:
\begin{itemize}
\item A graphical evaluation of how a model improves, and most especially, degrades if we start removing attributes, finding the concept of dominance here.
\item The analysis of those context changes related to attribute (feature) costs and the introduction of solutions to create more versatile models and reframing transformations.
\item The evaluation of the previous approaches with several datasets, using common repositories or the reframe domains.
\end{itemize}

\noindent In this work we will focus on classification, but many of the ideas could be extended to regression as well.

\vspace{0.3cm}

\section{Thesis plan}

\vspace{0.2cm}

\noindent This work is organised as follows:

\begin{itemize}
\item \textbf{Chapter \ref{prelim}: Previous Works.} This chapter describes some general aspects of machine learning with some of its actual challenges, and present some others works related to the current one.

\item \textbf{Chapter \ref{reducingAttributes}: Reframing the model with missing values on purpose.} In this chapter, a description and explanation of our approach is made.

\item \textbf{Chapter \ref{hull}: Approximating the \JROC hull.} A scaling-up of the approach described in the previous chapter is developed in this section.

\item \textbf{Chapter \ref{experiments}: Experiments.} This chapter presents all the experiments performed during this work and the result obtained, using a set of datasets of the UCI repository and different machine learning methods. 
\end{itemize}

\noindent This work discusses, in \textbf{chapter \ref{conclusion}}, the conclusions obtained from the experiment results, some future work and the dissemination made.

\noindent Finally, some  \textbf{appendices} follow with more details about the results of chapter \ref{experiments} and a description of the implementation of \JROC.

\chapter{Previous Works}\label{prelim} 

\setlength{\parskip}{\baselineskip}

In this chapter, we give an overview of machine learning, and some of its difficulties and challenges are presented. we also see some previous works and approaches related to our investigation which face those difficulties. Finally, the tools used during our investigation are introduced. 
\vspace{0.3cm}

\section{Machine learning}
\vspace{0.2cm}
In recent times the level of technological advancement in the use of machine learning as a reliable means of retrieving  information, and the ever increasing amounts of available data makes the analysis of data necessary for many areas of technological progress. Furthermore, the use of machine learning and machine learning tools has proven to help solve some complex problems, providing good solutions, in the form of classification and regression models \cite{FerriFH04}.

Machine learning is the means by which knowledge is acquired and the ability to use it. This implies that learning  involves the process of finding and describing data in the form of structural patterns. For instance, data could be of client complaints in an organisation or other service options open to clients. The output of such learning could be whether a particular customer complaint was genuine or not. Another example is that the data could contain examples of customers who have switched to another service provider in the telecommunication industry and some that have not. The output of learning could be the prediction of whether a particular customer will switch to another service provider or not. There are two types of learning: supervised and unsupervised. The previous examples are supervised. In this work, we will focus on supervised learning.
 
Supervised learning requires the response for each instance such that it can be used by the system to guide the learning. The whole process includes the collection of data to be used for data mining, identifying the target variable, dividing up of the data into training and testing data and constructing and evaluating the model. The training data is used by the data mining algorithm to ‘learn’ the data and build a model while the test data is used to evaluate the performance of the model on new data. For instance, decision trees and neural nets are two common types of supervised learning. This type of learning always requires a target variable to predict. 

However, some of the problems of supervised learning are overfitting \cite{FerriFH04} \cite{witten}, which means that the model perfectly works on the training data but not on the evaluation (test) data. This problem can also occur independtly of overfitting. It happens when the model learns the training data with some attributes and in the test time there are some missing or new features.
Thus, learning, inference, and prediction in the presence of missing data or ``new'' data are pervasive problems in machine learning and statistical data analysis.

\vspace{0.3cm}

\section{Features selection}
\vspace{0.2cm}
There are some aspects of machine learning which involve the application of datasets with a large number of features (tens or hundreds of thousands of variables available) such as text processing, gene expression, array analysis, and combinatorial chemistry. In such case the model built could be overfitted. Feature selection has in recent times become the focus of much research \cite{guyon2003introduction, molina2002feature}, with approaches to address such a situation. Feature selection, also known as variable selection, attribute selection or variable subset selection, is the process of choosing a subset of relevant features within the original set of features which will be used in model construction. The principal assumption when using a feature selection technique is that the data may contains many redundant or irrelevant features. A feature is redundant if it does not give enough information once given the already selected features. Irrelevant features  are the ones which provide useless information for all contexts. In addition, feature selection has some others advantages when constructing predictive models: improved model interpretability, reduced training times, and enhanced generalisation by reducing overfitting.

\vspace{0.3cm}

\section{Missing attributes: a challenge}
\vspace{0.2cm}
In addition to the large number of variables or attributes that machine learning often has to face, missing values is another challenge. Missing data can be introduced in many situations and domains, and for various reasons. For instance, in bioinformatics, certain regions of a gene micro-array may fail to give measurements of the underlying gene expressions due to scratches, finger prints, or manufacturing defects; or in a hospital, participants in a clinical study may drop out during the course of the study leading to missing observations at subsequent time points; and moreover, a doctor may not be able to order all applicable tests to a patient. All these examples could lead to a large number of missing data.

\vspace{0.3cm}

\subsection{Missing attribute imputation}
\vspace{0.2cm}

In the past there have been many approaches to account for missing data. However, one well known attempt to deal with missing data is known as ``imputation'' \cite{zhu2011missing}.
In statistics, imputation is the process of replacing missing data with substituted values (when substituting for a data point, it is known as ``unit imputation''; when substituting for a component of a data point, it is known as ``item imputation''). Because of the fact that data can be ``obstructive'' during training and lead to the construction of a suboptimal or defective model, the concept of imputation is seen by some researchers as a way to avoid challenges involved with list-wise deletion of instances that have missing values. In practice, when one or more values are missing from a case, most statistical packages default to removing any example that has a missing value (the classifier algorithm discard all the instances which do not have all the information), which may introduce bias or affect the representation of the results (pattern). But imputation preserves all cases by replacing missing data with a probable value based on other available information. Once all missing values have been imputed, the data set can then be trained and evaluated using standard methods for complete data. Nevertheless, imputing missing values ``is regarded as unnecessary for cost-sensitive learning that also considers the test costs'' \cite{zhang2005missing}.

\vspace{0.3cm}

\section{Cost-sensitive methods}
\vspace{0.2cm}

As a means of minimise the cost made when predicting the classification of unseen examples, there have been many works in the last decade that investigate approaches to learning or revising classification procedures that attempt to reduce the cost of misclassified examples rather than the number of misclassified examples \cite{turney2000types}.  These ideas are underpinned by the concept that in many problems, the cost of all errors (false positive or false negative) is not equal. Thus, the cost of making an error can depend upon both the predicted class of the example and the actual class of an example. Therefore, in contrast to the general case, here the main purpose is not to maximise the accuracy. For instance, in a bank, credit card fraud detection aims to maximise the total transaction amount of correctly detected frauds minus the cost to investigate all (correctly and incorrectly) detected frauds; And because undetected fraud causes a loss of the whole transaction amount, it is by far more profitable to detect frauds with high transaction amount than those whose amount is not even higher than the cost to investigate.

Turney \cite{turney2000types} presents an excellent survey on different types of costs in cost-sensitive learning, among which misclassification costs and test costs are distinguished as most important.  
  
Many modifications to machine learning algorithms attempt to reduce these costs, as in classification rules \cite{pazzani1994reducing}, when a prior domain theory is available \cite{elkan2001foundations}.

\subsection{Cost-sensitive decision trees}
\vspace{0.2cm}

In real practice, as we said, many real-world data sets for machine learning and data mining contain missing values and previous research regards it as a problem and attempts to impute missing values before training and testing. Some other works study this issue in cost-sensitive learning.

In cost-sensitive learning, which attempts to minimise the total cost of tests and misclassifications, missing data can be useful for cost reduction, so imputing missing values might be unnecessary. Some recent approaches in this field are cost-sensitive decision tree learning algorithms which should utilise only known values, using decision trees as the base learner, that consider both test costs and misclassification costs. In this case, some attributes (during the tests) might be too expensive for obtaining their values. Thus it can be cheaper not to include their values, avoiding so expensive and risky tests (as in patient diagnosis for example  \cite{ling2004decision,zhang2005missing}). This reasoning is called: ``missing is useful'' \cite{zhang2005missing}, as the fact that values actually reduce the total cost of tests and, therefore, it is not meaningful to impute their values, as they will not reduce the misclassification cost. Cost-sensitive decision trees algorithms have different approaches:

\begin{itemize}
\item As values are missing for certain reasons --- unnecessary and too expensive to test --- an option is to replace missing data with a special value, mostly called ``null'' in databases. And this null value will then be treated just as a regular known value in the model construction (tree building in this case) and during the test processes. Since this strategy has been proposed in machine learning \cite{ali, date}, its performance and efficiency  in cost-sensitive learning has not been shown, since the use of the same value to all missing values (null in this case) may not be good enough since it could introduce bias.
\item The second option is called the C4.5 strategy. This approach consists in choosing, during training, an attribute by the probability of missing values of that attribute. And during the test process, a test example with missing value is divided into branches according to the portions of training examples falling into those branches. The class of the test example is the weighted classification of all leaves. C4.5’s missing-value strategy has been shown to be efficient in cost-sensitive learning\cite{batista}. 
\end{itemize}

\vspace{0.2cm}

\noindent Both approaches consider missing values in a way to make methods cost-sensitive. But they are not optimal in such cases where you have to deal with the cost of the features and missing features at the same time. Also, they are restricted to decision trees. This investigation brings a new vision to tackle this kind of cases. The approach we discuss in this work does not look at imputation either as a solution nor an option since it is a disadvantage for the cost-effectiveness in cost-sensitive learning. Our idea is to build a model independent approach (not based on a special technique such as decision tree, i.e, it can be applied to any technique.) 

Before the presentation of our approach, it is necessary to present and define some tools used during the process.
\vspace{0.3cm}

\section{Tools}\label{tools}
\vspace{0.2cm}
The main tools used for the implementation of this work are {R} and {RWeka}.

\subsection{R}
\vspace{0.2cm}

R is at the same time a language and  an environment used for statistical computing and graphical representation about statistics results. Besides, the fact that R provides a great number of statistical methods (linear and nonlinear modeling, classical statistical tests, classification, regressions, clustering, \dots) and graphical techniques, gives the possibility to be highly extended by integrating many packages and libraries available.

\vspace{0.3cm}

\subsection{RWeka} 
\vspace{0.2cm}

RWeka is an R interface to Weka, while Weka is a collection of machine learning algorithms for data mining tasks written in Java. It contains tools to pre-process data, for classification or regression, clustering, association rules, and visualisation. Weka is also recommended and used for developing new machine learning schemes. The package RWeka contains the interface code, and the Weka jar is in a separate package RWekajars. For the experiments of our work we used Weka (thus RWeka, because the implementation and execution are made in R language and environment).

As we will work with a learned model that we want to apply to a new cost context, we need to find a way for removing some attributes from the model, in order to reduce the Test Cost. A tricky way to do this is by setting them as ``missing''.
In Weka the missing values are represented with a ``?''; so for an attribute that we do not want to represent in the reframed model (maybe because of its higher test cost), we just represent the entire column with ``?''.

\subsection{Weka algorithms used in this work} \label{algo}

The main algorithms which we used for the experimental part of this work are:
\begin{itemize}
\item SMO is a Support Vector Machine (SVM), a supervised learning model with associated learning algorithms that analyse data and recognise patterns, used for classification and regression. The basic SVM takes a set of input data and predicts, for each given input, which of two possible classes forms the output, making it a non-probabilistic binary linear classifier. But it also works for more than two classes.

\item IBk is a K-nearest neighbour (KNN), a machine learning classification algorithm. KNN is lazy (it defers computation until classification is needed) and supervised.
\item J48 is a decision tree learning algorithm that graphically displays the classification process of a given input for a given output class labels.

\item \label{ada} Adaboost is an ensemble method for constructing a ``strong'' classifier as a linear combination of simple ``weak'' classifiers \cite{bauer1999}. 

\item \label{bagg} Bagging \cite{bauer1999} is a classificcation method which consists in training a number of base learners each from a different bootstrap sample (original data partition) by calling a base learning algorithm (decision trees, KNN, \dots). After obtaining the base learners, it executes a vote and the most-voted class is predicted. 

\end{itemize}

\noindent We also add that, for the representation and the evaluation for each model or algorithm used in our approach, we have used the ROC curves and some evaluation of classifiers methods. We see some basic ideas next.

\section{Classifier visualisation and evaluation}
\subsection{{ROC} Analysis}
A receiver operating characteristics ({ROC}) graph is a technique for visualising, organising and selecting classifiers based on their performance \cite{Fawcett06}. First, {ROC} graphs have been introduced to represent the trade-off between hit rates and false alarm rates. Later, it has also been demonstrated that {ROC} curves can also be very useful as a visualisation method for classifiers evaluation and comparison. This is due partly because of the fact that simple classification accuracy is often a poor metric for measuring performance, and partly because {ROC} curves have properties well adapted for domains with unbalanced class distributions and unequal classification error costs.

\subsection{Classifiers evaluation}

Nowadays there is a wide range of classifiers techniques. This has led the user to one preoccupation: how can he evaluate how good is the learned model? Performance metrics are of a fundamental importance to give an consistent answer to this question. In classification, it has been shown that in data mining problems, the use of accuracy to compare classifiers is not adequate, and many works have been developed to address this problem \cite{PRL09}. In general, we have to choose the most adequate measure (or set of measures) for a specific application (Accuracy, F-measure, Rank Rate, Area Under the ROC Curve (AUC), Squared Error, \dots). These metrics can be grouped in three principal groups \cite{PRL09}: the ones based on a threshold and qualitative understanding of error, the ones based on a probabilistic understanding of error, and the ones based on how well the model ranks the examples (where {AUC} is found).

\chapter{Reframing the model with missing values on purpose }\label{reducingAttributes}
\setlength{\parskip}{\baselineskip}
There has been an extensive work in the past decades on how the performance of a predictive technique evolves with different feature subsets. This is the core of feature selection techniques. In fact, model performance can even be increased by using a subset of the original attributes. Also, if we think about costs, most works on minimising costs have taken this approach \cite{ling2004decision,zhang2005missing,lomax2013survey}.

However, we can also consider that the model has already been trained (with possibly all the attributes) and we may just want to apply the model with fewer available attributes, e.g., when missing values appear or when we cannot afford `buying' some of the tests included in the model. It is important to say that we consider models that may have been developed by experts or by automated predictive analysis tools, or both. What we do is to {\em reframe} the model to a situation with fewer attributes.

\vspace{0.3cm}

\section{Definition of the approach}\label{approach}

We will focus on classification problems, characterised by a multivariate input domain $\mathbb{X}$, i.e., a tuple of elements of sets $X_1, X_2, \dots, X_m$, where $m$ is the number of features or input attributes, possibly containing the null value, and a univariate output domain $\mathbb{Y} \subset \{l_1, l_2, \dots, l_c\}$, where $c$ is the number of classes or labels of the output attribute. The domain space $\mathbb{D}$ is then $\mathbb{X} \times \mathbb{Y}$. 
Examples or instances are just pairs $\left\langle x,y \right\rangle \in \mathbb{D}$, and datasets are subsets (actually multi-sets) of $\mathbb{D}$. The length of a dataset will usually be denoted by $n$. 
A {\em crisp} classification model $\hat{f}$ is a function $\hat{f}: \mathbb{X} \rightarrow \mathbb{Y}$. 
We just represent the true value by $y$ and the estimated value by $\hat{y}$. Subindices will be used when referring to more than one example in a dataset. Given an example $i$, the values of the $m$ input attributes are denoted by $x_{i,1},  x_{i,2}, \dots, x_{i,m}$. 

Throughout this work we will use several classifiers from Weka \cite{weka}. We are especially interested in using the techniques as they are, being able to use techniques that are, in principle, inattentive to the use of all the attributes, such as kernel methods, ensembles, etc. In particular, we will use SMO (a support vector machine), IBk (a k-nearest neighbour), J48 (a decision tree), Adaboost (an ensemble method with decision stumps) and Bagging (an exemple method with decision trees)presented in the previous chapter. All of them will be used with their default parameters.

Once this common setting for classification is set, we may wonder how models are created and deployed. In fact, models are usually learned under some contextual information but possibly deployed several times under changing conditions.
Reuse of learned knowledge is of critical importance in the majority of knowledge-intensive application areas, particularly because the operating context can be expected to vary from training to deployment and we need to make the best decision according to that context \cite{SDM00,flach2003decision}. One kind of context is related to the way inputs (i.e., features) can vary from training to deployment. Among these changes, we can mention two important ones: attributes may not be available (missing values) or may have different test costs. Another type of context depends on how class distribution and misclassification costs affect the output variable. Note that these context changes may happen for each problem instance individually. For instance, in a medical domain, some tests may not be applicable to some patients (as can be contraindicated or risky), other tests may be more or less expensive depending on the patient (her insurance policy). Also, for the output variable, a wrong diagnosis usually has asymmetric costs, as a false negative is usually worse (and economically more expensive in the long term) than a false positive. 

These two types of costs (test costs and misclassification costs) are highly intertwined. In fact, as Turney \cite{turney2000types} points out, we can only rationally determine whether it is worthwhile to pay the cost of test when we know the cost of misclassification errors. If the cost of misclassification errors is much greater than the cost of tests, then it is rational to purchase all tests that seem to have some predictive value. But if the cost of misclassification errors is much less than the cost of tests, then it is not rational to purchase any tests.

Let us define these types of cost formally:

\begin{definition}\label{def:M}
A misclassification cost function is any function $M:{\mathbb{Y}} \times {\mathbb{Y}} \rightarrow \mathbb{R}$ which compares elements in the output domain. For convenience, the first argument will be the estimated value, and the second argument the actual value. 
\end{definition}
As ${\mathbb{Y}}$ is a discrete set, typically we refer to $M$ as the misclassification cost {\em matrix}. We will assume that the diagonal of the matrix is zero (i.e., $\forall y \:: M(y,y) = 0$) and that the other elements of the matrix are greater than or equal to 0.

We can have a different matrix for each example, denoted by $M_i$. 
From above, we define the misclassification cost $MC$ of an example $i$ as $MC_i \triangleq M_i(\hat{y}_i,y_i)$. 
For a complete dataset we can just calculate the average $MC$ as the Frobenius product between the confusion matrix and the cost matrix, divided by $n$.

\begin{definition}\label{def:T}
The test cost vector is a real vector of size $m$, i.e., $(t_1, t_2, \dots, t_m)$, where $m$ is the number of attributes. 
The test cost function $T_j$ is any function as:
\begin{eqnarray*}
T_j(x) \triangleq \left\{   \begin{array}{l l}
                                             t_j    & \quad \text{if $x$ is not null}\\
                                             0      & \quad \text{otherwise}
                       \end{array} \right.																	
\end{eqnarray*}						
\end{definition}

We can have a different test cost function for example, denoted by $T_{i,j}$. 
From above, we define the test cost $TC$ of an example $i$ as $TC_i \triangleq \sum_{j=1}^m T_{i,j}(x_{i,j})$. 
For a complete dataset we can just calculate the average $TC$ as the dot product between the use vector (how many times each attribute has been used) and the test cost vector, divided by $n$.

We want to integrate both the misclassication cost and the test cost in one single measure of cost:

\begin{definition}\label{def:JC}
The {\em joint} cost is:
\begin{eqnarray*}
JC_i \triangleq \alpha \cdot MC_i + (1-\alpha) \cdot TC_i															
\end{eqnarray*}						
with $\alpha \in [0,1]$.
\end{definition}
The value $\alpha$ will be better explained later on, but clearly sets more relevance to misclassification or test costs. If $\alpha = 1$ only the misclassification cost matters, and if $\alpha = 0$  only the test cost matters.
$M$, $T$ and $\alpha$ configure the cost context. With $m$ attributes and $c$ classes, there are $m+c(c-1)-1$ degrees of freedom (assuming the cost matrix has a zero diagonal).

\begin{example}\label{ex:1}
Consider the iris dataset \cite{UCIrep2013}, created by R.A. Fisher, which is composed of four attributes: $SL$, $SW$, $PL$ and $PW$ and three classes: {\em setosa}, {\em versicolour} and {\em virginica}. 

Assume that we have an example where the test cost vector is $(3, 2, 10, 5)$ and the misclassification cost matrix $M$ is defined as follows:

\begin{center}
{\center
\begin{tabular}{c|ccc}
            & setosa & versicolour & virginica \\ \hline
setosa      & 0      &         20  & 15        \\
versicolour & 5      &         0   & 15        \\
virginica   & 30     &         15  & 0         \\
\end{tabular}
}
\end{center}

\noindent where columns represent the actual value and rows the predicted value. 
Consider also that we have three models. Model 1 requires attributes $SL$ and $PL$ and predicts $virginica$, model 2 requires attributes $SL$ and $SW$ and predicts $setosa$, and model 3 requires all attributes and predicts $versicolour$. 
If the true label is $versicolour$, then we have $JC = MC + TC = 15 + (3 + 10) = 28$ for model 1, $JC = MC + TC = 20 + (3 + 2) = 25$ for model 2, and $JC = MC + TC = 0 + (3 + 2 + 10 + 5) = 20$ for model 3. 
\end{example}

In the previous example, model 3 is better than the other two for this example. Of course, in general, we need to make the decision of which model to use without knowing the actual label, and that will depend on the reliability of the models and the class frequencies. This is then a decision problem, for which we need to determine the model with lowest expected cost.

Interestingly, we may wonder what would happen if we remove attribute $PW$ for model 3. Even though we are told that model 3 works with that attribute, it is not difficult to guess what the model would do without it. There are several ways to do it: set it to null and see what happens, or consider several values for the attribute and get the most frequently predicted class. 
Imagine that, by using any of these two methods, model 3 still predicts $versicolour$. Our cost would have been lowered down to $15$.

So, the question we want to address in this work is not only what model to choose but also the subset of attributes that we will use (`buy').
\vspace{0.3cm}

\section{Reframing the model with missing values on purpose }\label{reducingAttributes2}
Re-training can be a bad choice on many occasions: when we have an expert (human-made) model, when
we are using ensembles or other techniques with high training costs, when the training data is no longer
available, or when the cost context changes recurrently, even for each example.
What can we do instead of re-training? What we do is to {\em reframe} the original model to a situation with
fewer attributes, a different {\em feature configuration} . But, how do models behave when we remove attributes
from them?

First, we need to clarify how we can get predictions from a model that takes $m$ attributes when we only provide $m' < m$. There are two possible ways of reframing a model in order to do this: 

\begin{enumerate}
\item Setting the attribute to null. Many models can just work with missing values for test instances. However, on some occasions the model cannot take null values (e.g., logistic regression is usually one of these techniques). Nonetheless, it highly depends on the implementation of the technique (or the model).
\item \label{item-range} Instead, we can invent or negotiate over the attribute  \cite{bella2011using}. This means that if it is a nominal attribute, we can just ask the model to give a prediction for all the possible values for the attribute, get the predictions, and calculate the most frequently predicted class. If it is a numerical attribute, we can just use a sampling or discretisation and then behave similarly. If we have information about the attribute value distribution, we can also use it, as in missing value imputation.
\end{enumerate}

\noindent This second approach is more powerful (and related to missing value imputation and feature selection). 
In fact, on occasions, we may even realise that we get the same prediction for whatever value of the attribute (i.e., this is said to be a non-negotiable attribute in terms of \cite{bella2011using}) so we can clearly save the cost of getting the value for this attribute. 
However, we will work with the first way, as using a null value works for many DM/ML techniques and libraries, without further modification of our models. In our case, it just worked smoothly with Weka \cite{weka}.

\begin{figure}
\centering
\includegraphics[width=0.6\textwidth]{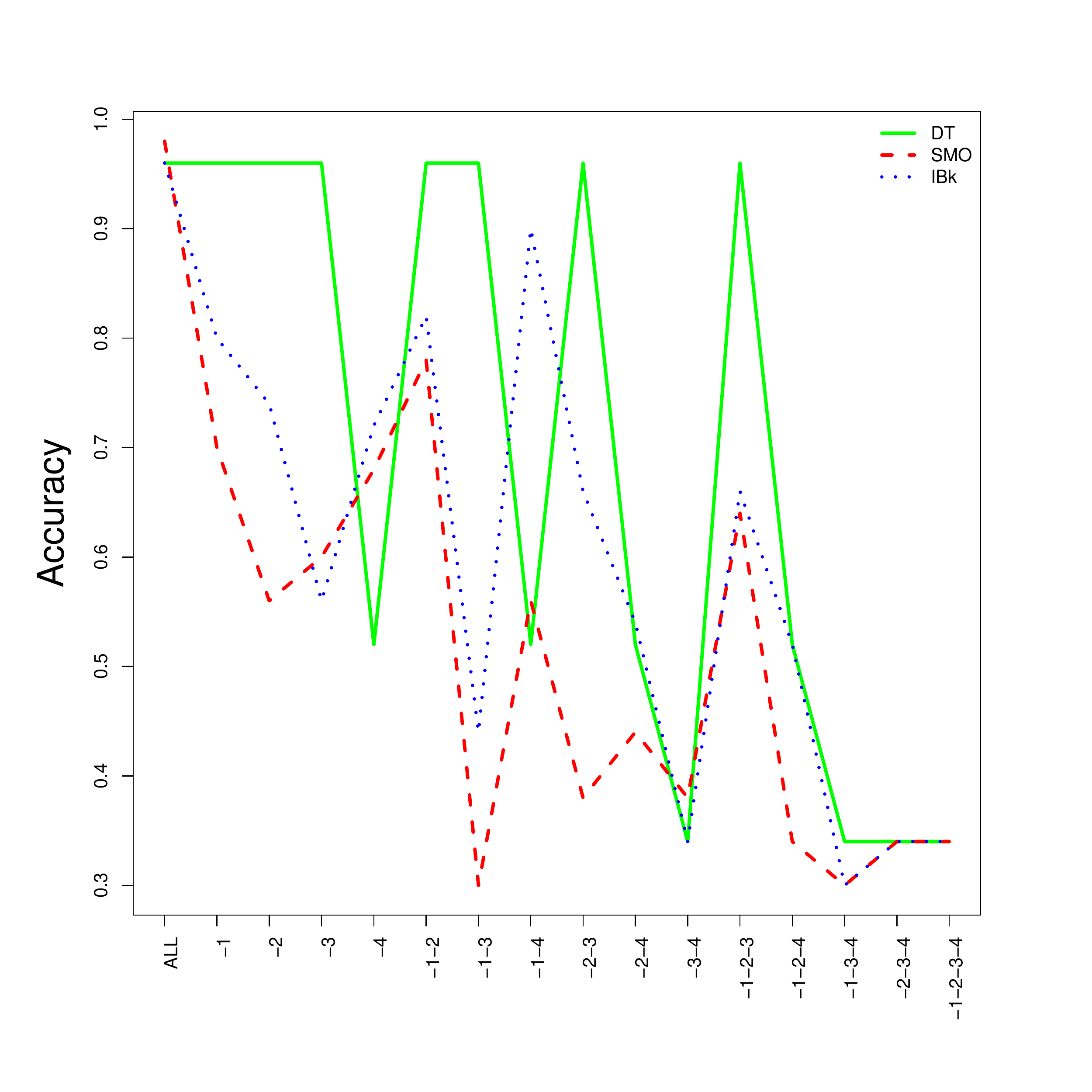} \hfill
\includegraphics[width=0.6\textwidth]{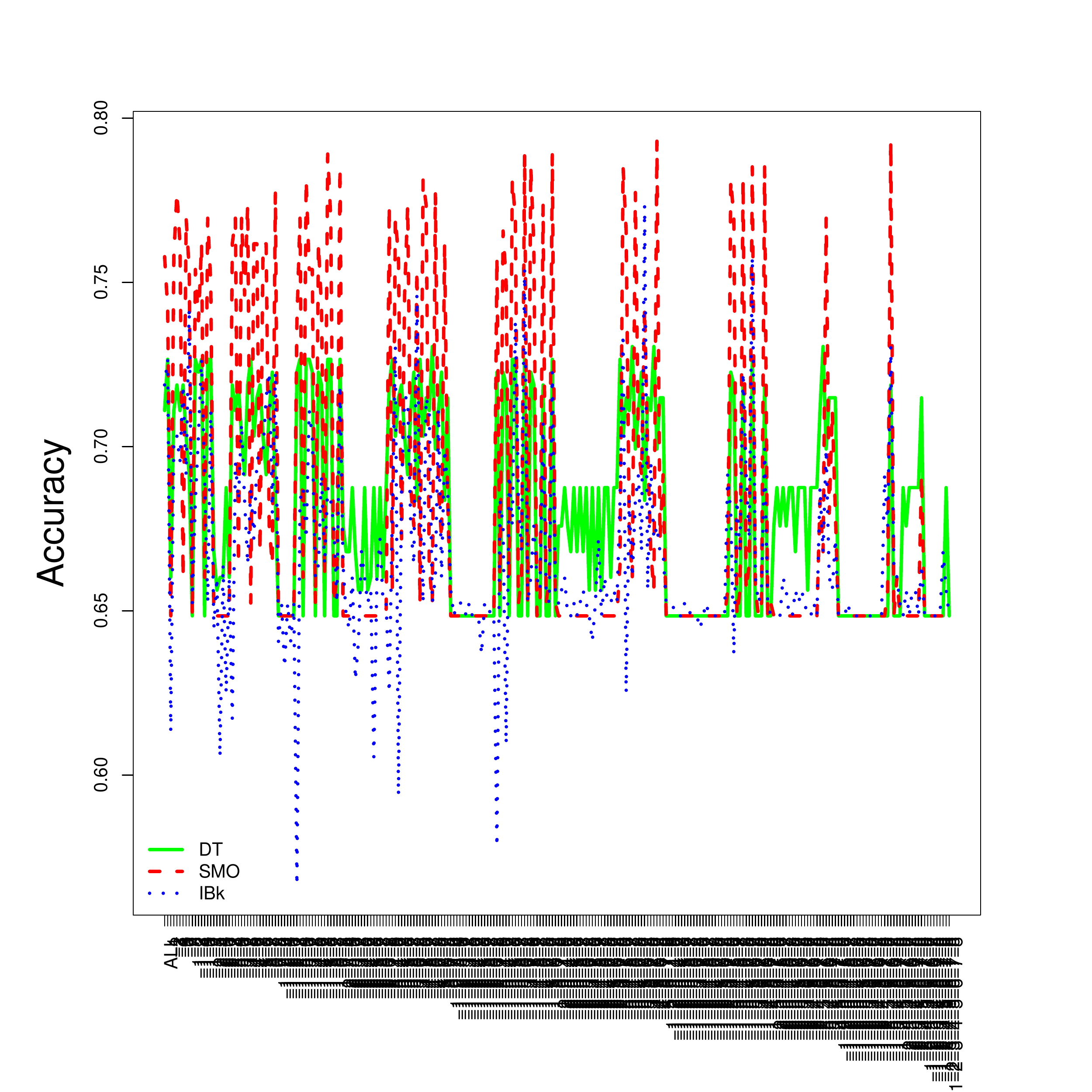} 
\caption{Evolution of accuracy according to attribute selection for three models: decision trees, SMO and kNN.}
\label{fig:accuracy}
\end{figure}

Once we know a simple procedure to reduce the attribute set, let us analyse how models behave.
Figure \ref{fig:accuracy} shows the evolution of accuracy\footnote{We show accuracy, but we could show other measures such as AUC or MSE \cite{PRL09,JMLR12}}, for all the possible subsets of the iris and the diabetes dataset (the subset lattice). The models (decision trees, SMO and KNN) are learned over the whole training dataset using all the attributes. Then, some attributes are removed by setting a null value on them systematically. Up: iris dataset (with $4$ attributes and hence $2^4=16$ combinations). Down: Pima Indian diabetes dataset (with $8$ attributes and hence $2^8=256$ combinations). Models are trained on 2/3 of the data and evaluated with the rest. We see many interesting things here:
 First, the general pattern is to get more accuracy as more attributes are used. But, obviously, some attributes are more important than others, leading to a sawtooth picture.
 Second, and more interestingly, the minimum is found at the majority-class classifier, i.e., if we are not given any information about any attribute, the best thing that we can do is to predict the majority class (or the class with lowest expected loss if misclassification costs are taken into account). 
 Third, now surprisingly, we see that for some models and problems (Figure \ref{fig:accuracy}, down), the maximum is not obtained with all the attributes. In fact, it is obtained at several other places, one of them with four attributes removed (of the possible eight).

We can show the specific values of $MC$, $TC$ and the aggregate $JC$ for a given context of $M$, $T$ and $\alpha$. 
We will consider a `uniform' operating context:

\begin{definition}\label{def:uniform}
The uniform operating context $\theta_U$ is defined by a uniform test cost vector $(1/m, 1/m, \dots, 1/m)$ and a uniform misclassification cost matrix 
$\forall y_1,y_2 \:: M(y_1,y_2) = c/(c-1)$ if $y_1 \neq y_2$ and 0 otherwise. Also, $\alpha=0.5$.
\end{definition}

The parameters of this context have the property that given a problem whose classes are perfectly balanced, the expected $MC$ of a random classifier is $1$ and the expected $TC$ of a classifier using all the attributes is $1$. As a consequence, $JC=1$. For this context, any model with $JC > 1$ is clearly a model to be discarded. In fact, as a random classifier does not need to use any of the attributes, any $JC > 0.5$ is also discardable for this context.

Figure \ref{fig:allcostsU} shows the evolution of $MC$, $TC$ and the aggregate $JC$ for the uniform context described above. Up: iris dataset. The configuration which minimises the $JC$ is given by the use of only attribute 4 (removing -1-2-3). Down: Pima Indian diabetes dataset.  The configuration that minimises the $TC$ is given by only two attributes (removing six). We can see that the information shown is very similar to that evolution of Figure \ref{fig:accuracy}. However, for other operating contexts, things might be different.

\begin{figure}
\centering
\includegraphics[width=0.6\textwidth]{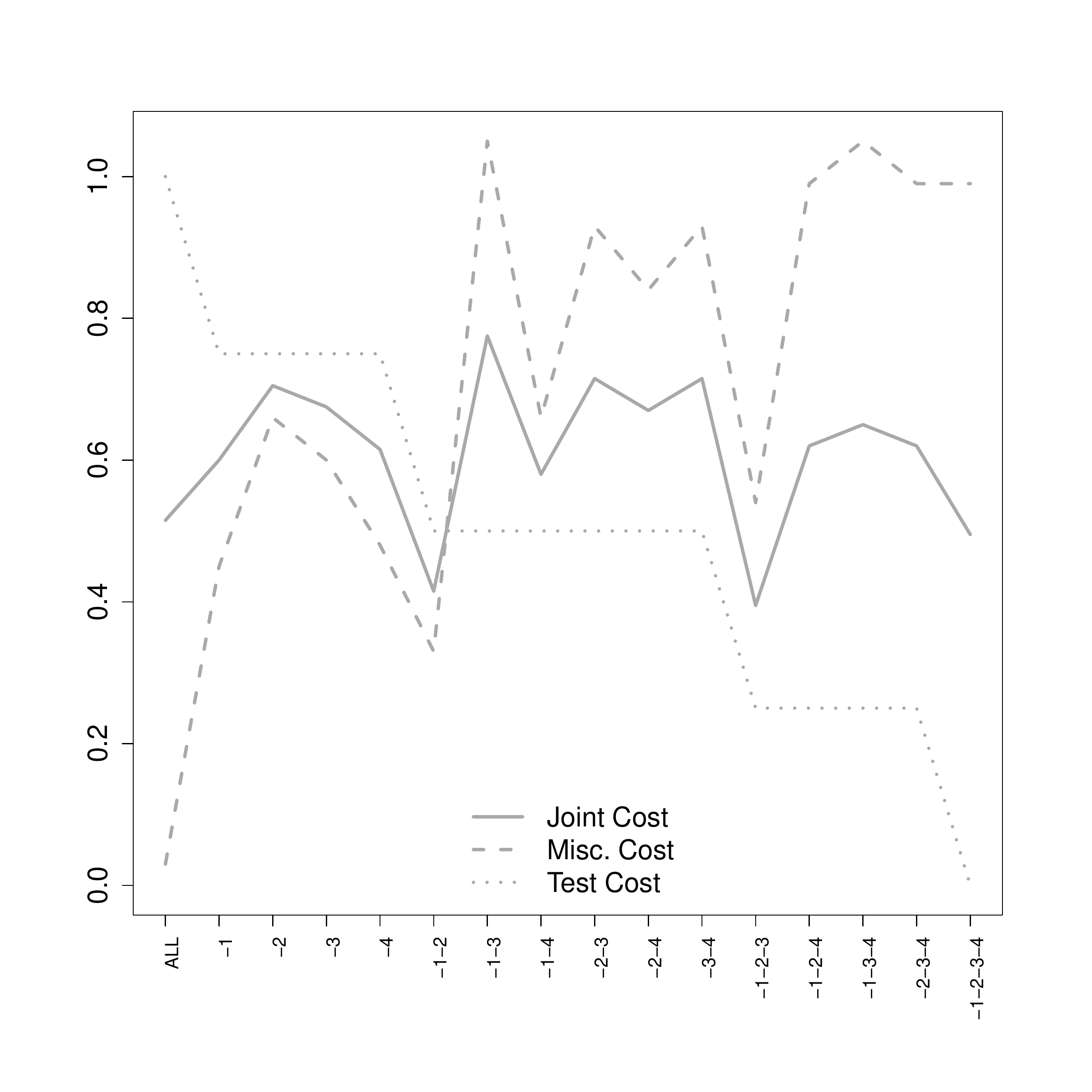} \hfill
\includegraphics[width=0.6\textwidth]{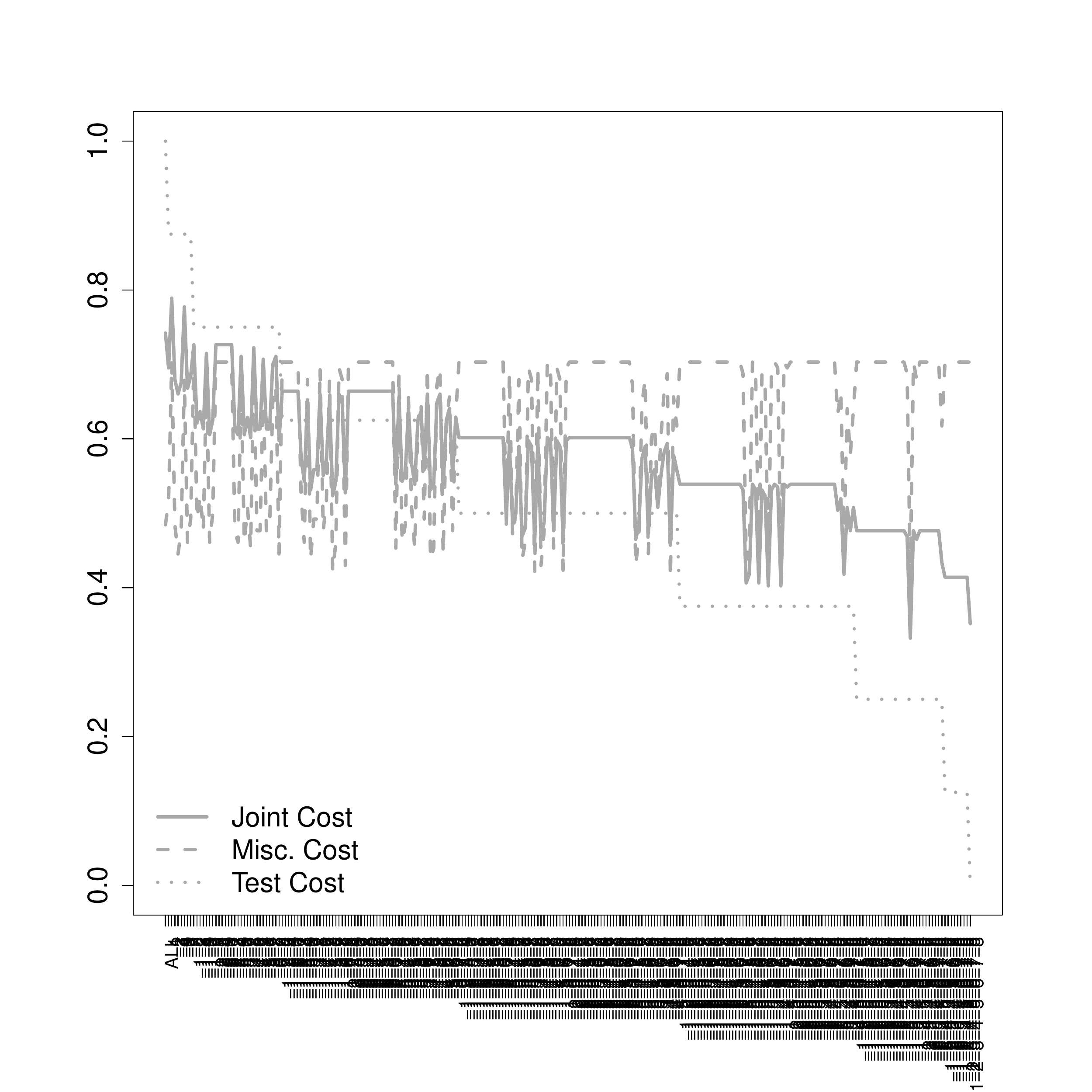} 
\caption{ Evolution of $MC$, $TC$ and $JC$ according to attribute selection for a SMO (SVM) model using the uniform context $\theta_U$ .}
\label{fig:allcostsU}
\end{figure}

Now let us consider another operating context, defined as follows for the problems ``iris'' and ``Pima Indian diabetes''.

The operating context $\theta_1$ for ``iris'' is just the one in example \ref{ex:1}. The operating context $\theta_2$ for ``Pima Indian diabetes'' is defined as a test cost vector is $(2, 50, 5, 5, 20, 3, 10, 1)$ where the most expensive tests correspond to `plasma glucose concentration', `2hour serum insulin' and `diabetes pedigree function'. The misclassification cost matrix $M$ is defined as follows:

\begin{center}
{\center
\begin{tabular}{c|cc}
             & negative (0)  & positive (1) \\ \hline
negative (0) & 0             &         200          \\
positive (1) & 50            &         0           \\
\end{tabular}
}
\end{center}
\noindent where columns represent the actual value and rows the predicted value. The value of $\alpha$ is 0.5.

With these operating contexts, Figure \ref{fig:allcosts0102},shows the same plots as Figure \ref{fig:allcostsU} ( Up: iris dataset using context $\theta_1$. The configuration which minimises the JC is given by the attribute 4. Down: Pima Indian diabetes dataset using context $\theta_2$. The configuration which minimises the JC is given by removing all the attributes.) .

\begin{figure}
\centering
\includegraphics[width=0.6\textwidth]{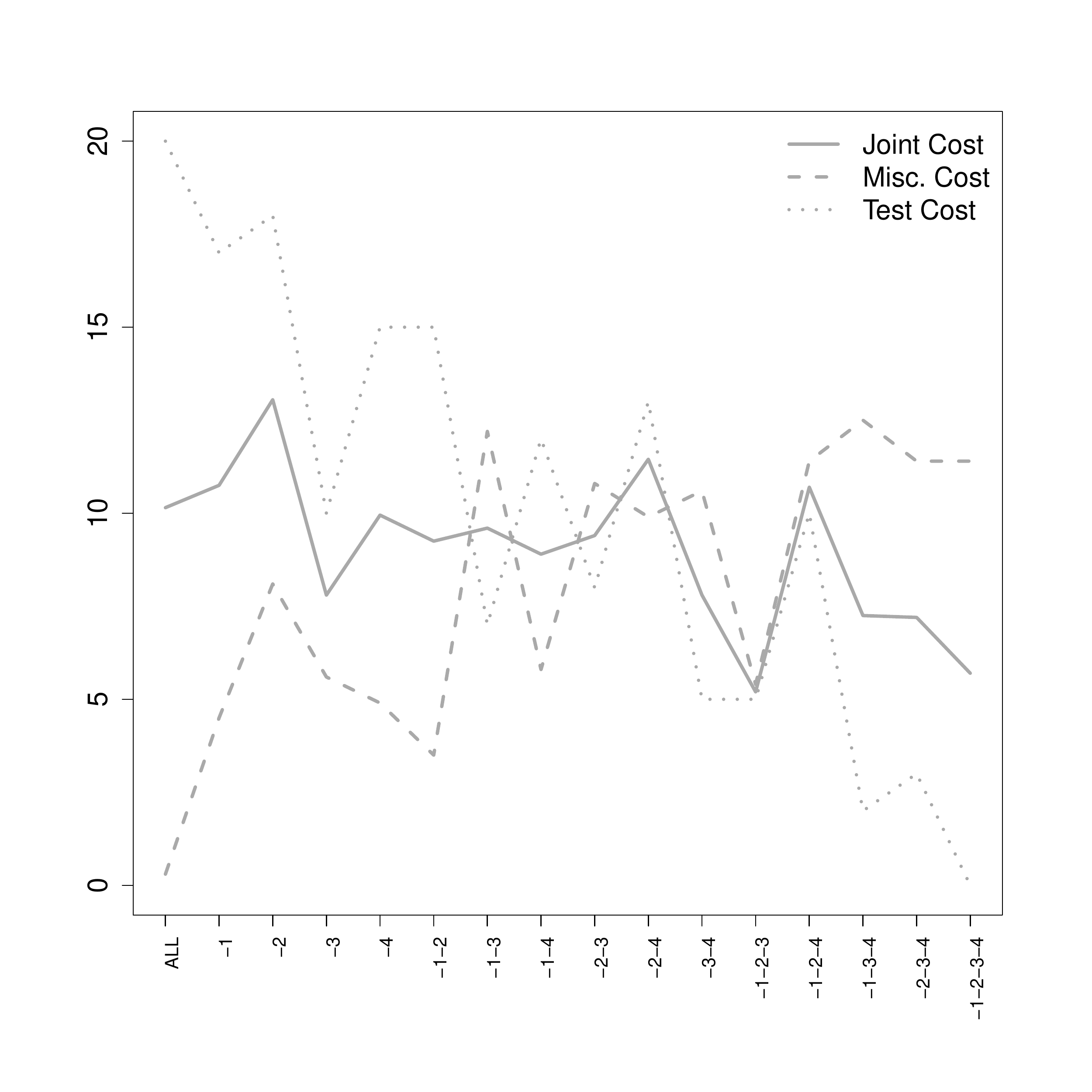} \hfill
\includegraphics[width=0.6\textwidth]{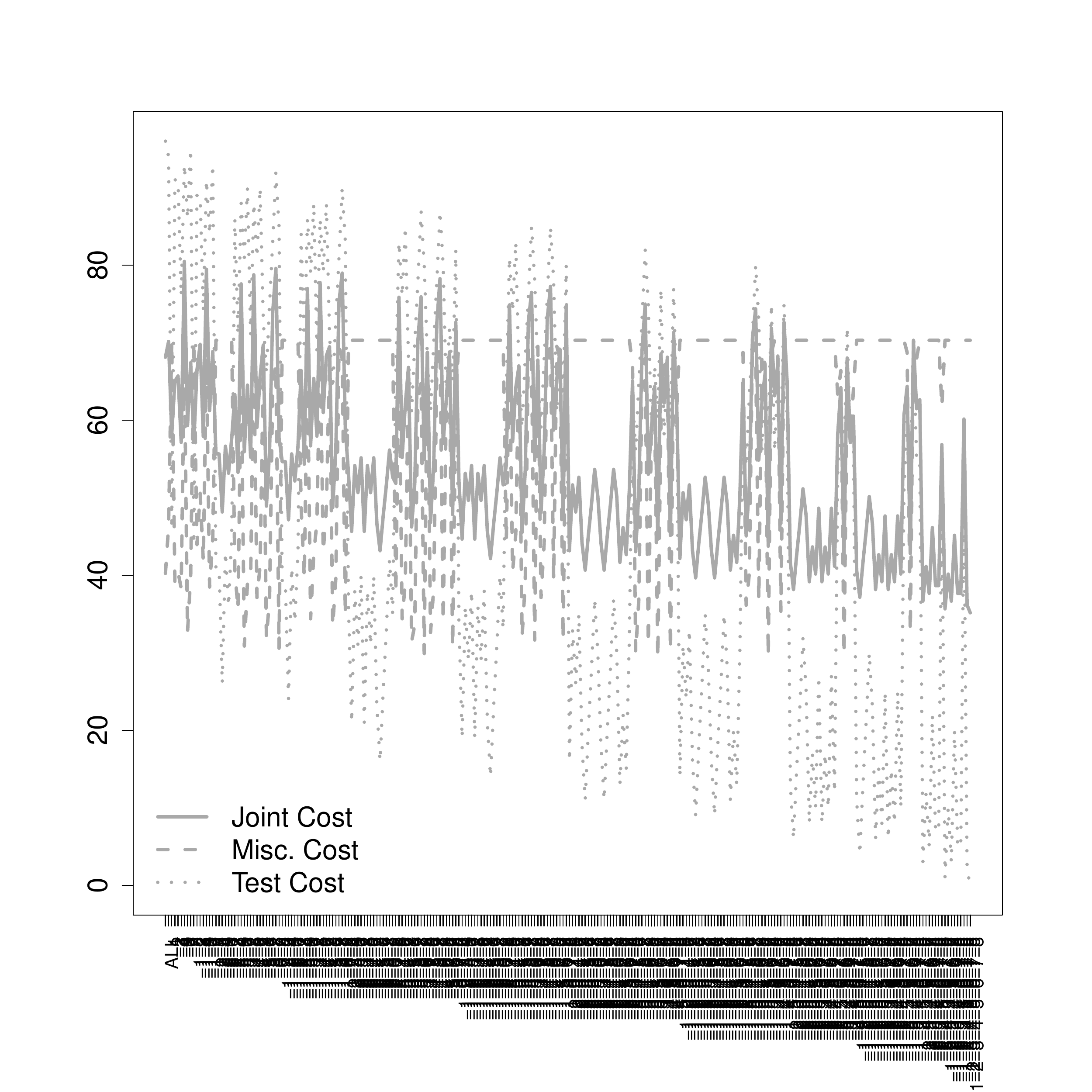} 
\caption {Evolution of $MC$, $TC$ and $JC$ according to attribute selection for a SMO (SVM) model using non-uniform context.}
\label{fig:allcosts0102}
\end{figure}

\vspace{0.3cm}

\section{The MC/TC tradeoff: \JROC plots}

The plots seen in the previous section are very informative for a given operating context. If the plots are drawn on a validation set, we will just choose the model and attribute configuration which minimises the $JC$. However, there are some problems with the previous plots: if we have several models, the plot gets too crowded. Also, the {\em curves} are usually too sawtooth. Finally, we need to change the curves whenever we change the operating context.

While some of the above problems are difficult to solve completely, most especially because we have $m+c(c-1)-1$ degrees of freedom, we can see a more convenient alternative that minimises these problems. We call these  \JROC plots.

\begin{definition}\label{def:jroc}
A \JROC plot shows the test cost ($TC$) on the $\xaxis$ and misclassification cost ($MC$) on the $\yaxis$.
\end{definition}

Figure \ref{fig:JROCU} shows $\JROC$ plots for iris and Pima Indian diabetes. 
For iris, as it has four attributes, we see $2^4 \times 3$ points, $2^4$ for each model.
For diabetes, as it has eight attributes, we see $2^8 \times 3$ points, $2^8$ fore each model.
Those models and configurations which go closer to the bottom left corner are better than those that are placed on the top right area of the plot. 
There is always a point with 0 $TC$ and a high misclassification cost, usually matching the majority class model. 
However, as mentioned earlier on, the minimum $MC$ is not always achieve with maximum $TC$. 
In this particular case, for iris we see that decision trees and kNN (IBk) perform better, as the points which are most on the bottom left are of these models. However, for diabetes, it seems that SMO and kNN get closer to the desired bottom left corner.

\begin{figure}
\centering
\includegraphics[width=0.6\textwidth]{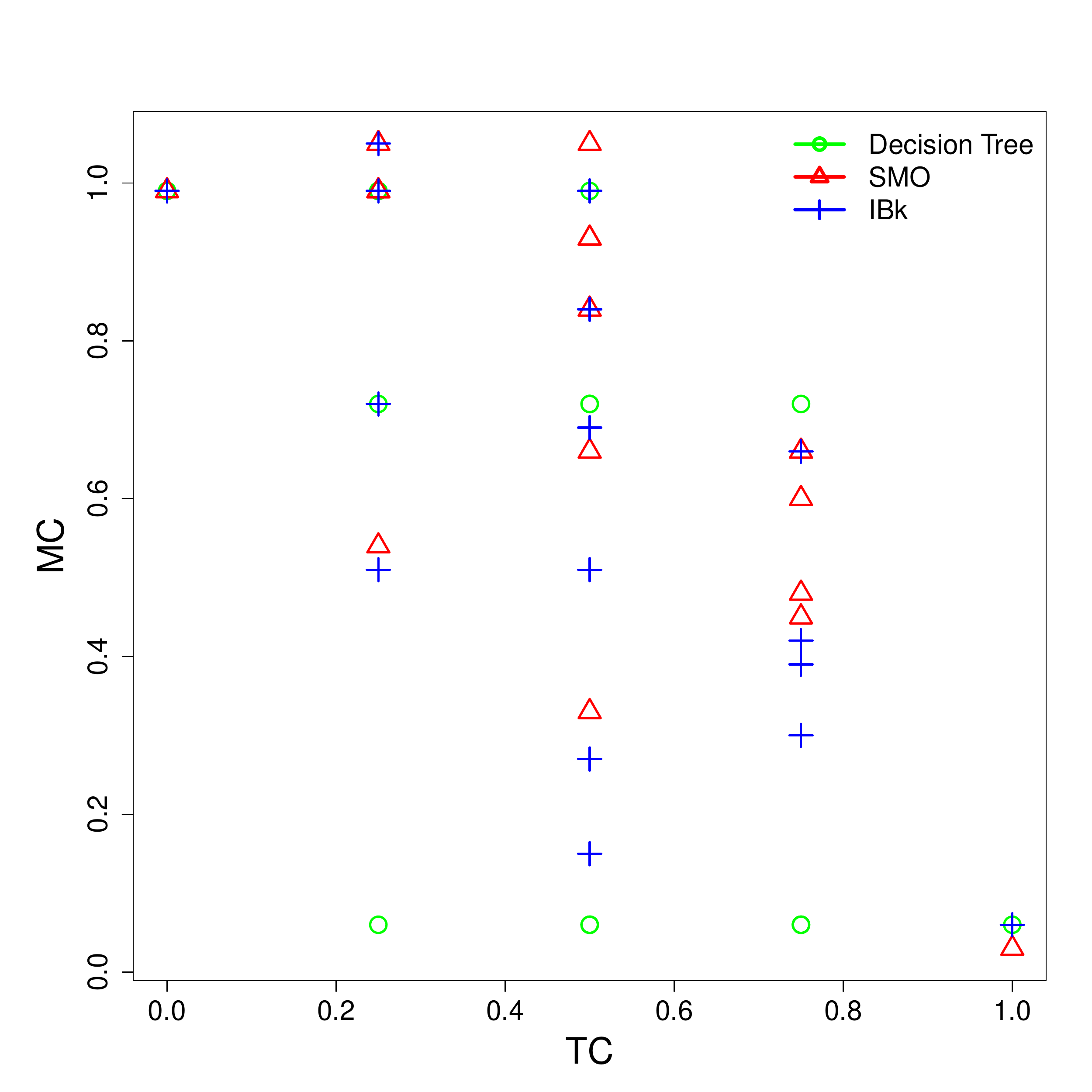} \hfill
\includegraphics[width=0.6\textwidth]{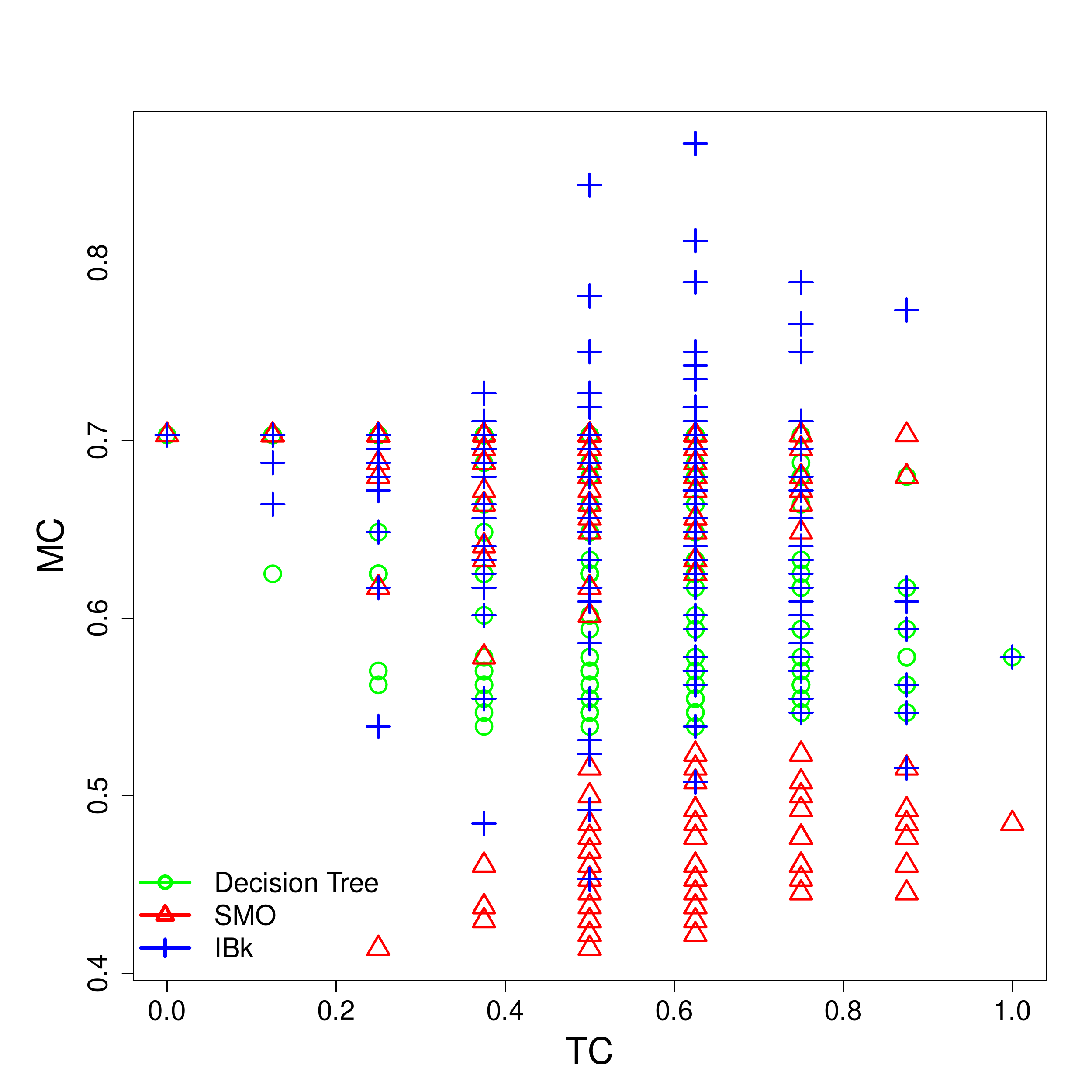} 
\caption{ \JROC plots for the three models: decision trees, SMO and IBk using the uniform operating
context$\theta_U$ .}
\label{fig:JROCU}
\end{figure}

Figure \ref{fig:JROC0102} shows a similar plot with different cost contexts. Here, we also see how the points are now located in different places. Even though the classifiers are the same, Compare to Figure \ref{fig:JROCU}, which uses a different cost context, the distribution of the points is very different. (The test cost ($TC$) on the \xaxis and misclassification cost ($MC$) on the \yaxis. up: iris dataset with the operating context $\theta_1$. Down: Pima Indian diabetes dataset with the operating context $\theta_2$.)

\begin{figure}
\centering
\includegraphics[width=0.6\textwidth]{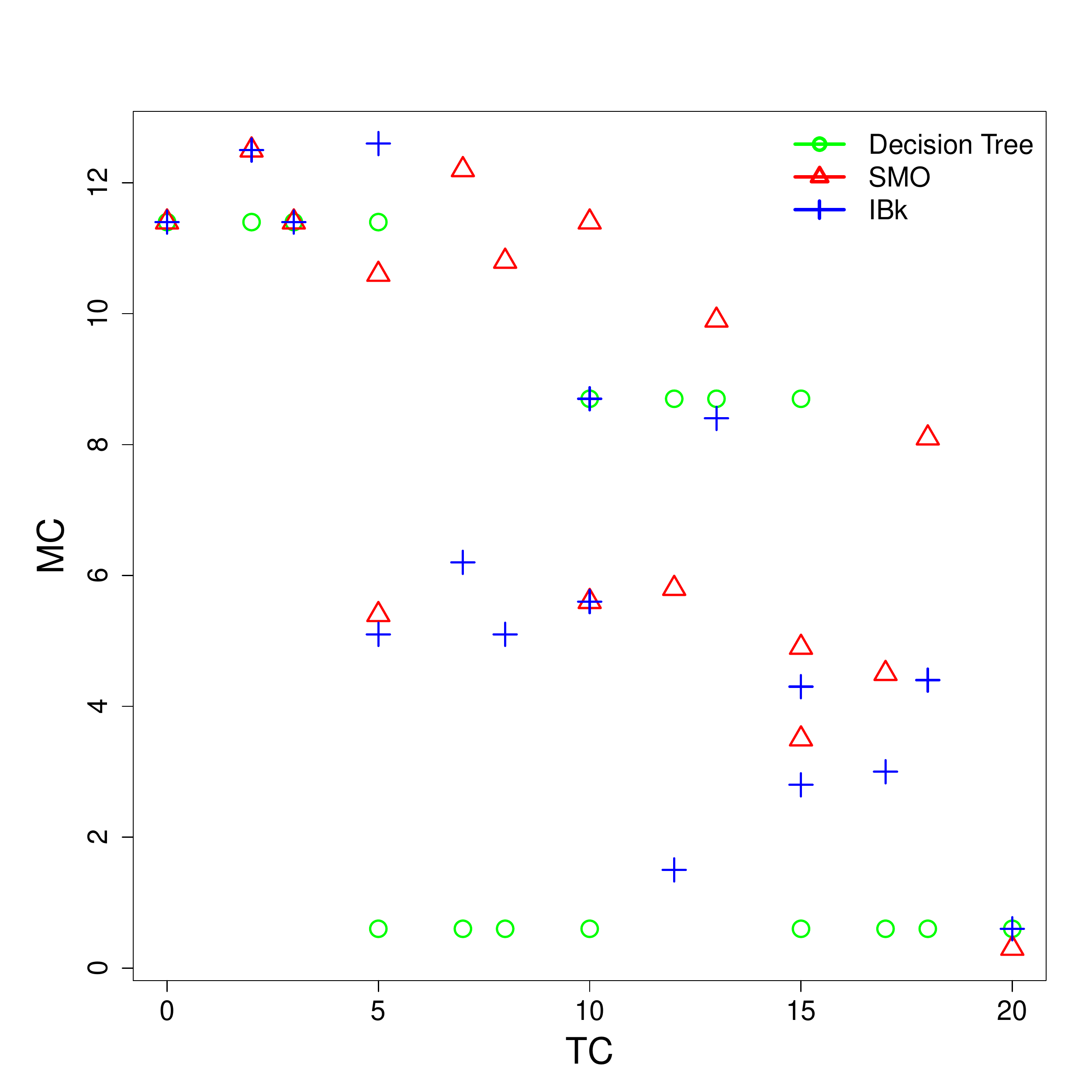} \hfill
\includegraphics[width=0.6\textwidth]{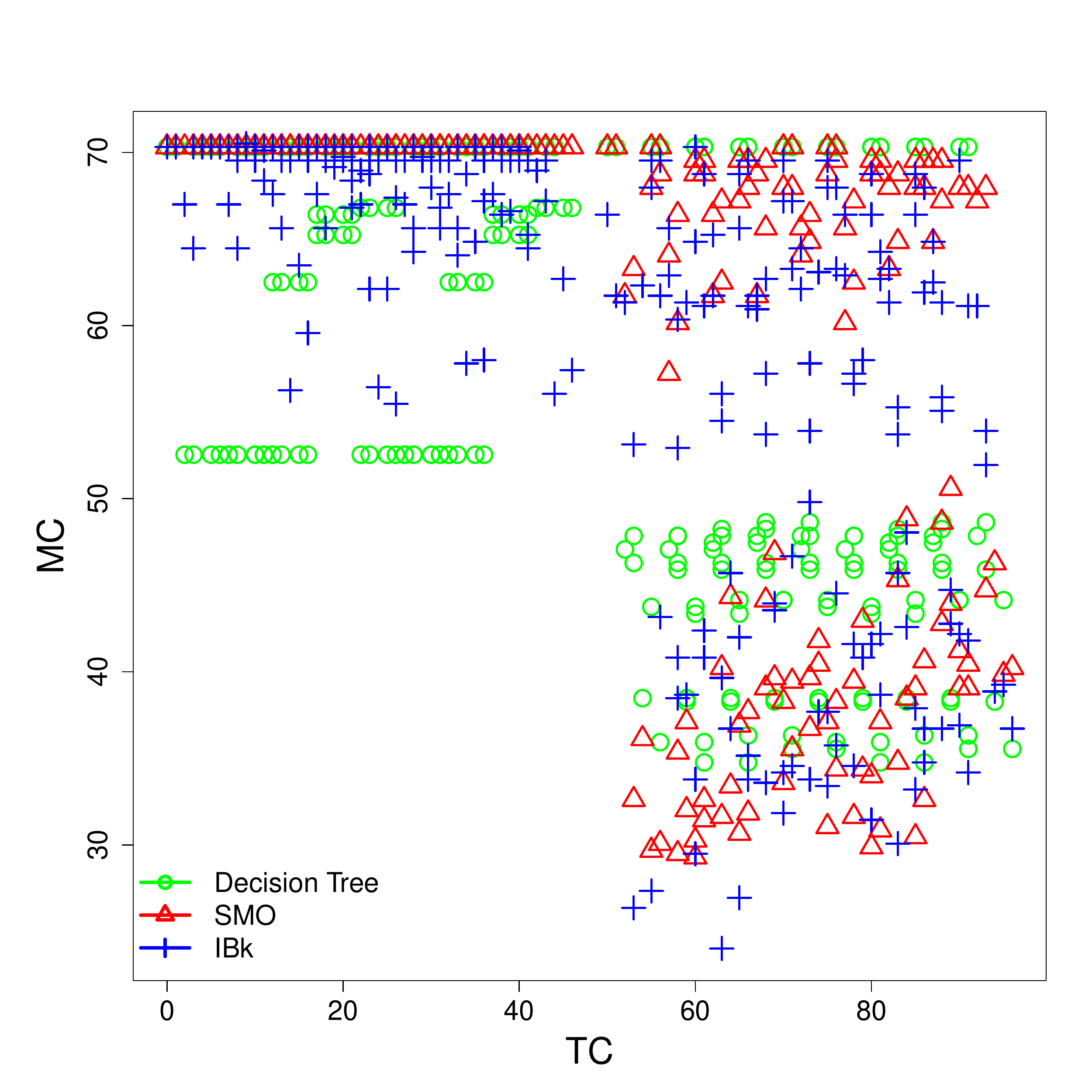} 
\caption {\JROC plots for the three models: decision trees, SMO and IBk using the non-uniform operating context.}
\label{fig:JROC0102}
\end{figure}

Intentionally, we have not shown $\alpha$ on the plots, so, according to definition \ref{def:JC}, we cannot calculate $JC$ unless this value is fixed. The following lemma shows that the same plot can be used for any value of $\alpha$, with the notion of cost {\em isometrics}.

\begin{proposition}
Given a value of $alpha$ the points which are connected by a line with slope $\frac{1-\alpha}{\alpha}$ have the same $JC$.
\end{proposition}
\begin{proof}
From definition \ref{def:JC} we have that $JC = \alpha \cdot MC + (1-\alpha) \cdot TC$. 
Consequently, two points $a$ and $b$ have the same $JC$ iff
$\alpha \cdot MC_a + (1-\alpha) \cdot TC_a = \alpha \cdot MC_b + (1-\alpha) \cdot TC_b$. 
Operating with this equation, we get:
\begin{eqnarray*}
\alpha \cdot MC_a + (1-\alpha) \cdot TC_a & =  & \alpha \cdot MC_b + (1-\alpha) \cdot TC_b \\ 
MC_a + \frac{1-\alpha}{\alpha} \cdot TC_a & = &  MC_b + \frac{1-\alpha}{\alpha} \cdot TC_b \\ 
MC_a - MC_b & = & \frac{1-\alpha}{\alpha} \cdot (TC_b - TC_a)  \\ 
\frac{MC_a - MC_b}{TC_a - TC_b} & = & \frac{1-\alpha}{\alpha}
\end{eqnarray*}
As the last expression is the change in $y$ divided by the change in $x$, the expression $\frac{1-\alpha}{\alpha}$ is the slope of this line.
\end{proof}

If $\alpha = 1$  only the misclassification cost matters and the slope is 0, and if $\alpha = 0$  only the test cost matters and the slope is infinite. It is clear that $\alpha$ only represents one of the $m+c(c-1)-1$ degrees of freedom, but it is able to consider the most important one: the relative relevance between misclassification and test costs.

As in classical ROC analysis, if we slide an isometric line given by a value of $\alpha$ from the point $(0,0)$ in the opposite direction (towards the top-right part of the plot), we will eventually find one  point (or more) on the plot. This is the best point according to the operating condition.

Figure \ref{fig:isometrics} ( Up: iris dataset with the operating context $\theta_1$. Down: Pima Indian diabetes dataset with the operating context $\theta_2$.) shows three different isometrics given by operating conditions $\alpha=0.03$, $\alpha=0.5$  and $\alpha=0.9$ and where they touch on the cloud of points. As we can see, different feature configurations and models are chosen for each operating condition  

\begin{figure}
\centering
\includegraphics[width=0.6\textwidth]{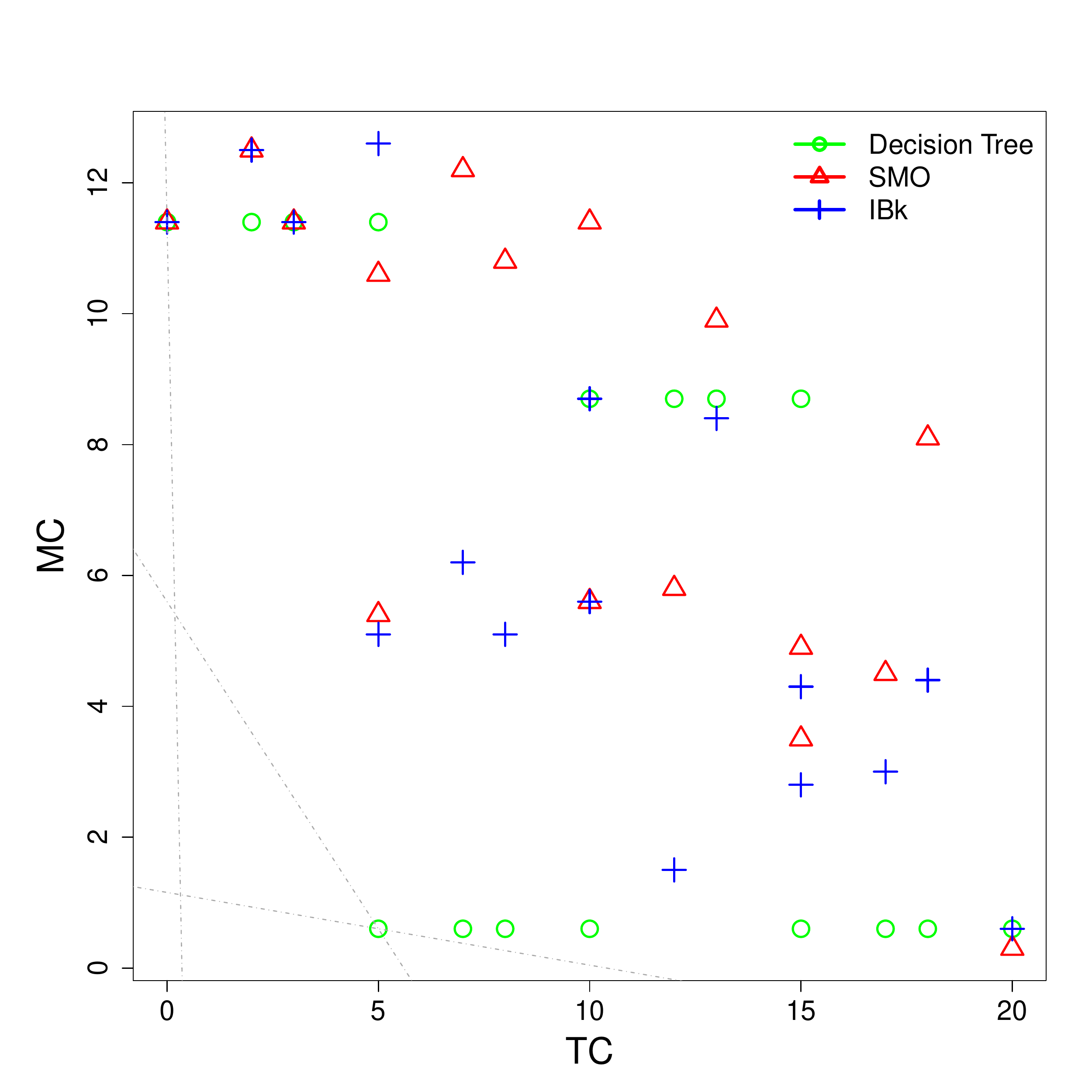} \hfill
\includegraphics[width=0.6\textwidth]{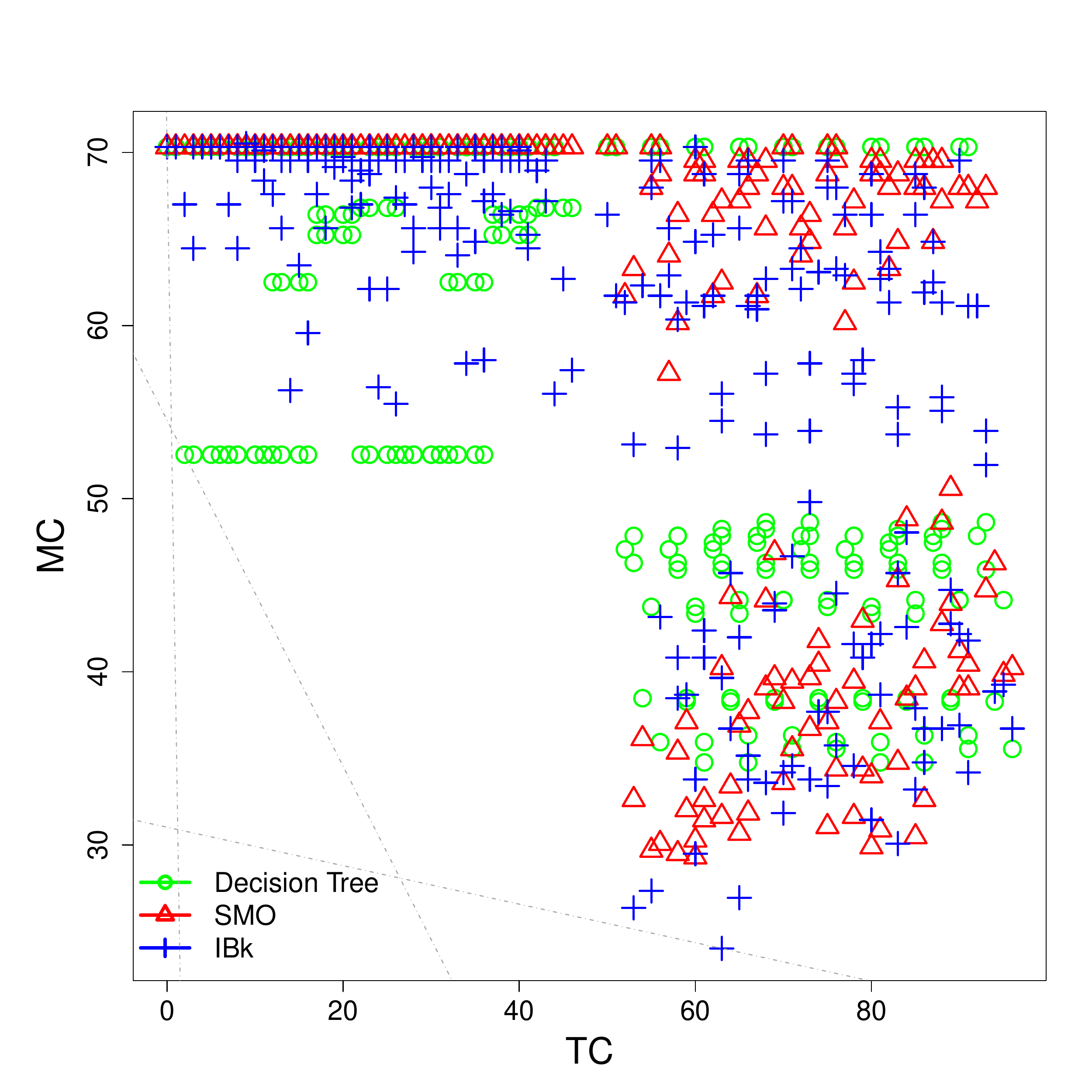} 
\caption {\JROC plot with isometrics representation for operating condition $\alpha=0.03$, $\alpha=0.5$  and $\alpha=0.9$.}
\label{fig:isometrics}
\end{figure}

Finally, if we consider all possible values of $\alpha \in [0,1]$ we see that some points are never chosen. This is exactly the notion of convex hull:

\begin{definition}\label{def:jroc-hull}
A \JROC convex hull of a model is the convex hull of the set of points on the \JROC space that are defined using all the attribute subsets.
\end{definition}

\begin{figure}
\centering
\includegraphics[width=0.6\textwidth]{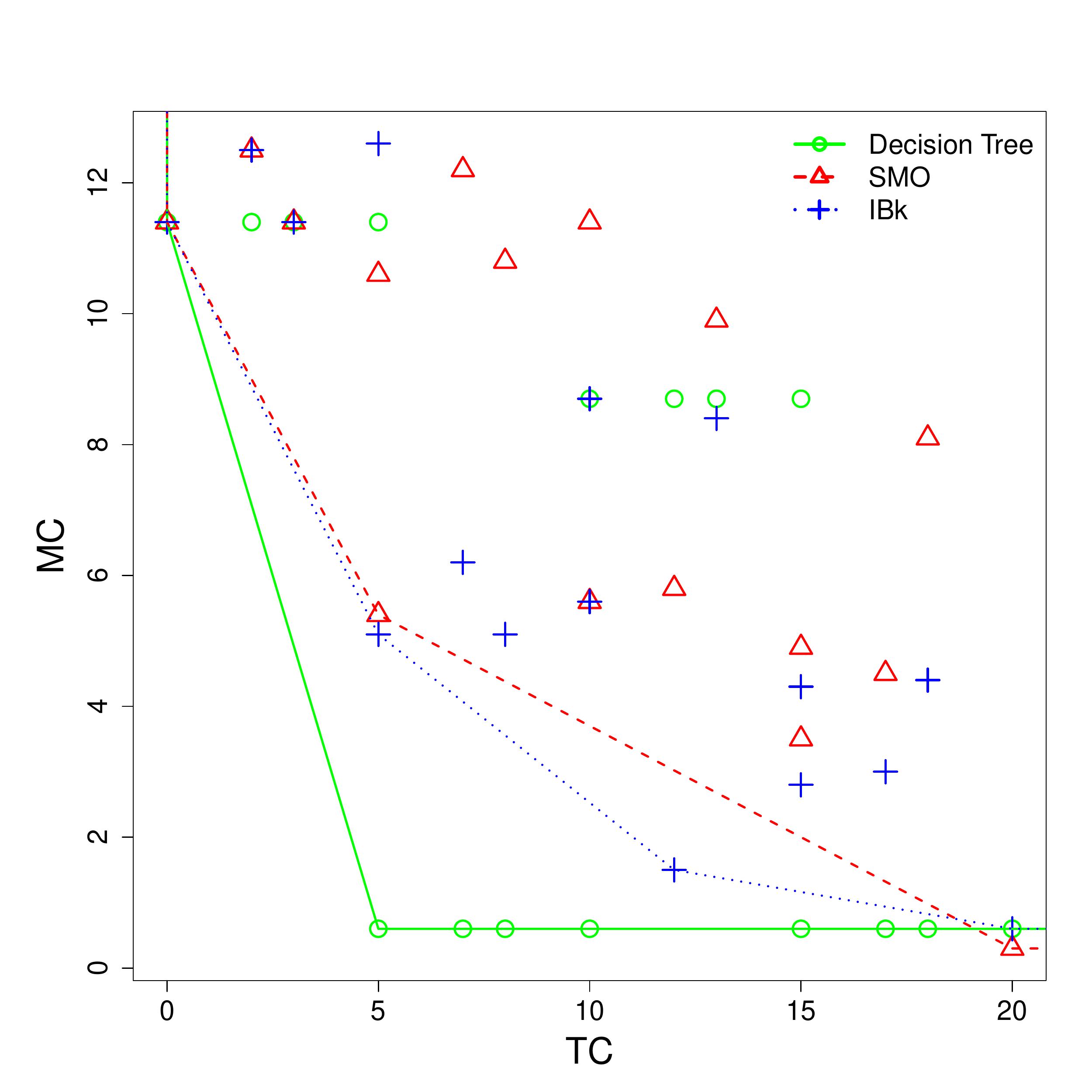} \hfill
\includegraphics[width=0.6\textwidth]{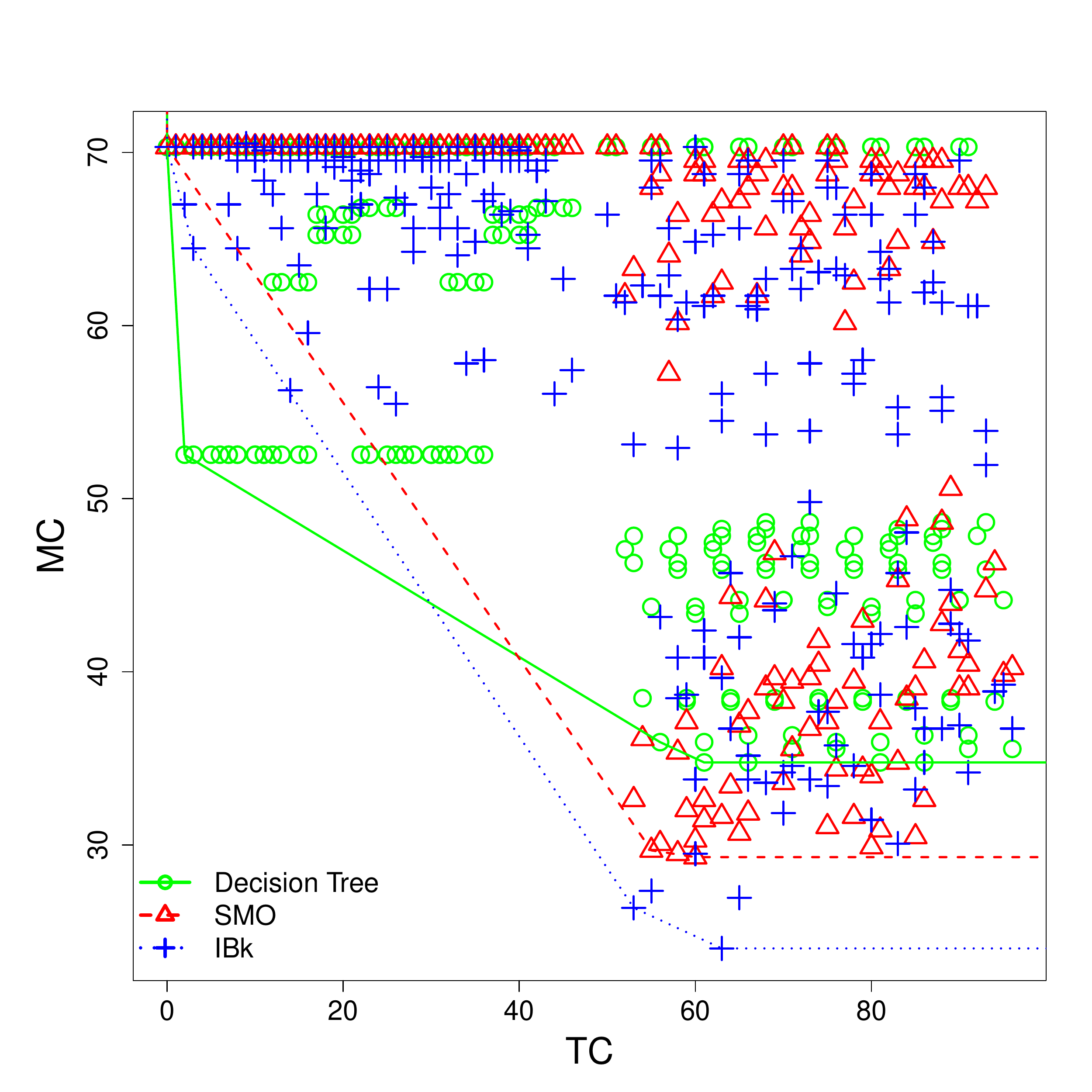} 
\caption {Convex hulls. Up: iris dataset with the operating context $\theta_1$. Down: Pima Indian diabetes dataset with the operating context $\theta_2$.}
\label{fig:hull}
\end{figure}

Figure \ref{fig:hull} shows the convex hull for each of the three models. We can also see the regions of dominance. For diabetes, IBk dominates for high values of $\alpha$, while the decision tree dominates for low values of $\alpha$.

From here we can calculate the regions of dominance for $\alpha$ and choose the best model accordingly, in the very same way as in ROC analysis.

\vspace{0.3cm}

\chapter{Approximating the \JROC hull}\label{hull}

The previous procedure allows us to determine the best model and configuration given the operating condition. We only need to calculate where all the points lie, compute the convex hull and find the one that corresponds for each possible $\alpha$ in application time. 
While this looks easy to do, there is one big issue. As the number of attributes increase, the number of points for a model grows exponentially: a lattice of $m$ elements has $2^m$ nodes. For instance, for a model with 16 attributes, we would have $2^{16} = 65536$ points. Navigating the complete lattice of attribute subsets, calculate their expected $TC$ and $MC$ would be unfeasible. So we need to explore some ways to reduce the number of configurations that are evaluated, while still having a good approximation of the \JROC hulls in order to do the correct decisions and get the optimal cost.

We will consider how to reduce the number of configurations from an exponential growth ($O(2^m)$), given by a {\em full} method, to a quadratic growth ($O(m^2)$). We consider four possible\footnote{There would also be the {\em forward} versions as well. We rule these possibilites out here for the simplicity of exposition, and also because we think that the results would be similar, but they could also be considered in practice.} methods:

\begin{itemize}
\item Backward $MC$-guided (BMC): we start with the $m$ attributes, we evaluate with the $m$ cases removing one attribute, and choose the best one in terms of $MC$, then we evaluate the $m-1$ cases removing one attribute from the previous one. This has an order of $O(m^2)$. In particular, it is easy to see that it leads to exactly $m^2/2)$ points.
\item Backward $TC$-guided (BTC): backward incremental: as $MC$ using $TC$ instead. It has the same order and number of points.
\item Backward $JC$-guided (BJC): as $MC$ using $JC$ instead. It has the same order and number of points.
\item Monte Carlo (RND): a random sample over the lattice. In order to make comparison fair, we will also consider the same number of elements: 
\end{itemize}

\noindent It is easy to show that if the misclassification cost matrix is uniform, then BTC and BJC are equivalent. If the test cost vector then BMC and BJC are equivalent. If both the misclassification cost matrix and the cost vector are uniform (i.e.,  the uniform operating context $\theta_U$) then BMC, BTC and BJC are equivalent.

Figure \ref{fig:bmc} shows the results for the BMC method for our two datasets and operating contexts ( Up: iris dataset with the operating context $\theta_1$. Down: Pima Indian diabetes dataset with the operating context $\theta_2$). If we compare with Figure \ref{fig:hull}, we see that the hull are almost identical.

\begin{figure}
\centering
\includegraphics[width=0.6\textwidth]{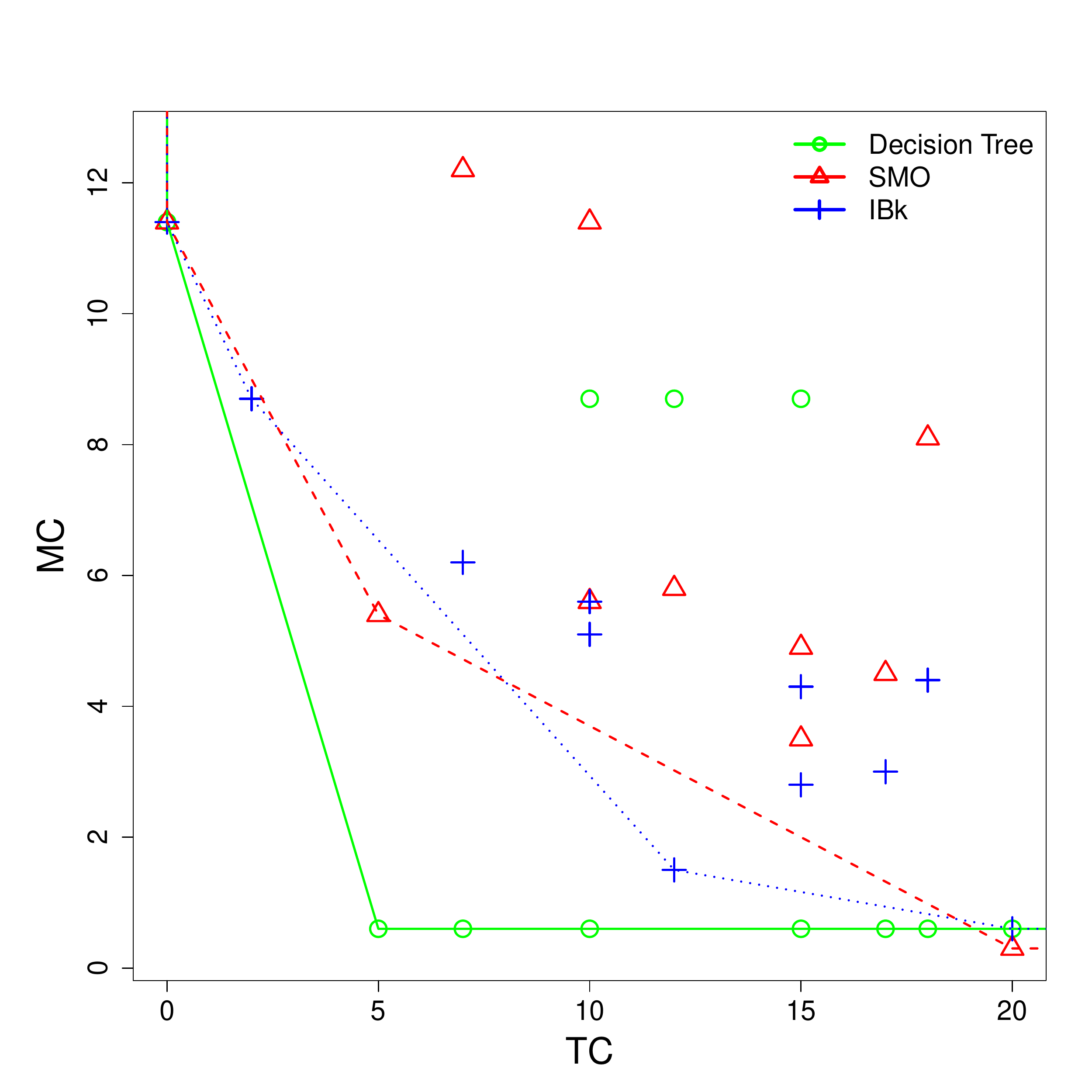} \hfill
\includegraphics[width=0.6\textwidth]{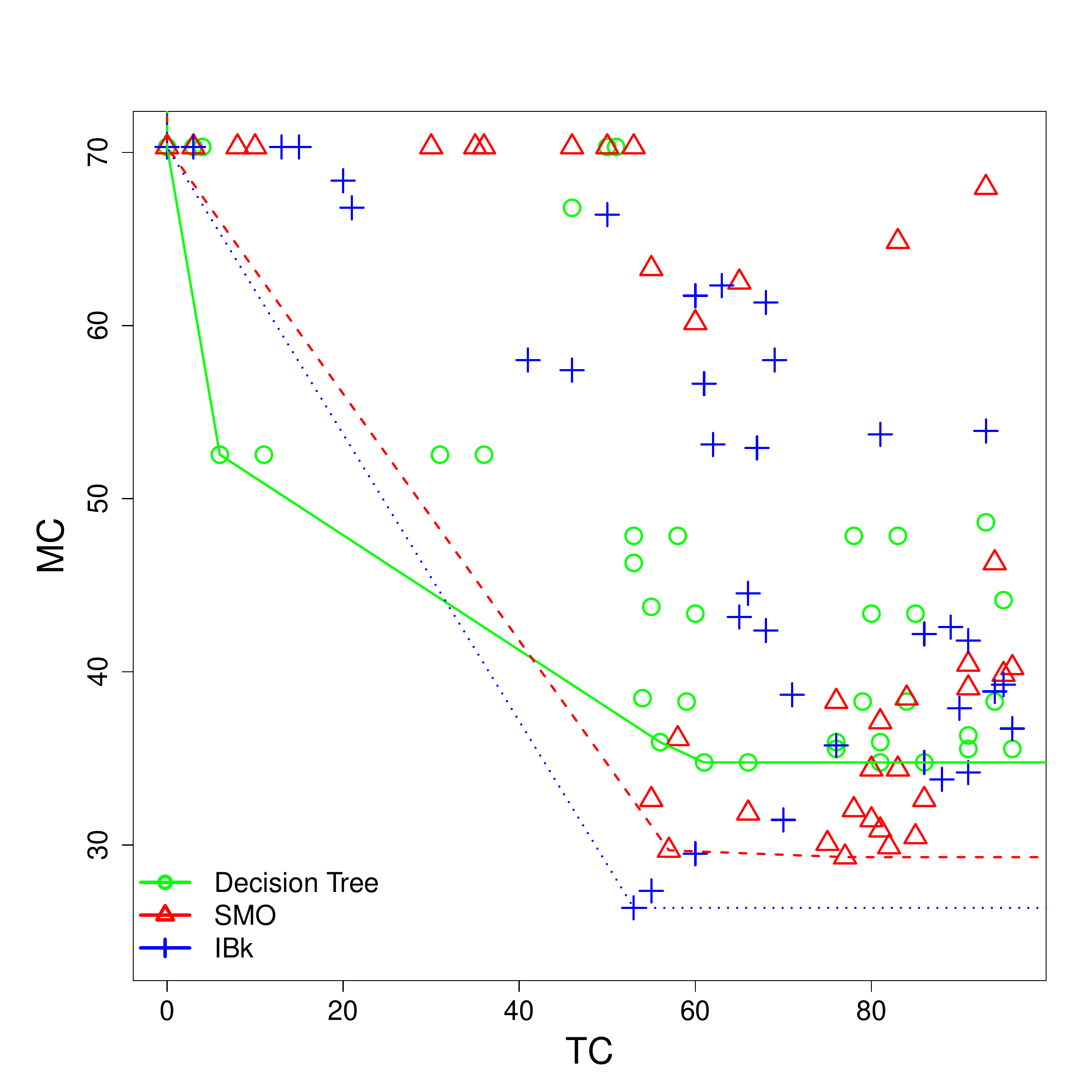} 
\caption{ Hull of the combinations selected by the method BMC. Compare with Convex Hull. Up: iris dataset with the operating context $\theta_1$. Down: Pima Indian diabetes dataset with the operating context $\theta_2$.}
\label{fig:bmc}
\end{figure}

Similarly, we have the results for the BTC, BJC and RND methods on figures \ref{fig:btc}, \ref{fig:bjc} and \ref{fig:rnd}. We note that there are $m(m+1) / 2 + 1$ points for each model (in each graph) instead of $2^m$.
If we compare with Figure \ref{fig:hull}, we see that the hulls are much worse for BTC, and notably worse for BJC and RND. 
However, we see that the results for BMC are good, almost identical to Figure \ref{fig:hull}. 
Does this observation hold in general? The answer of this question (and more) are explored in the experiments.

\begin{figure}
\centering
\includegraphics[width=0.6\textwidth]{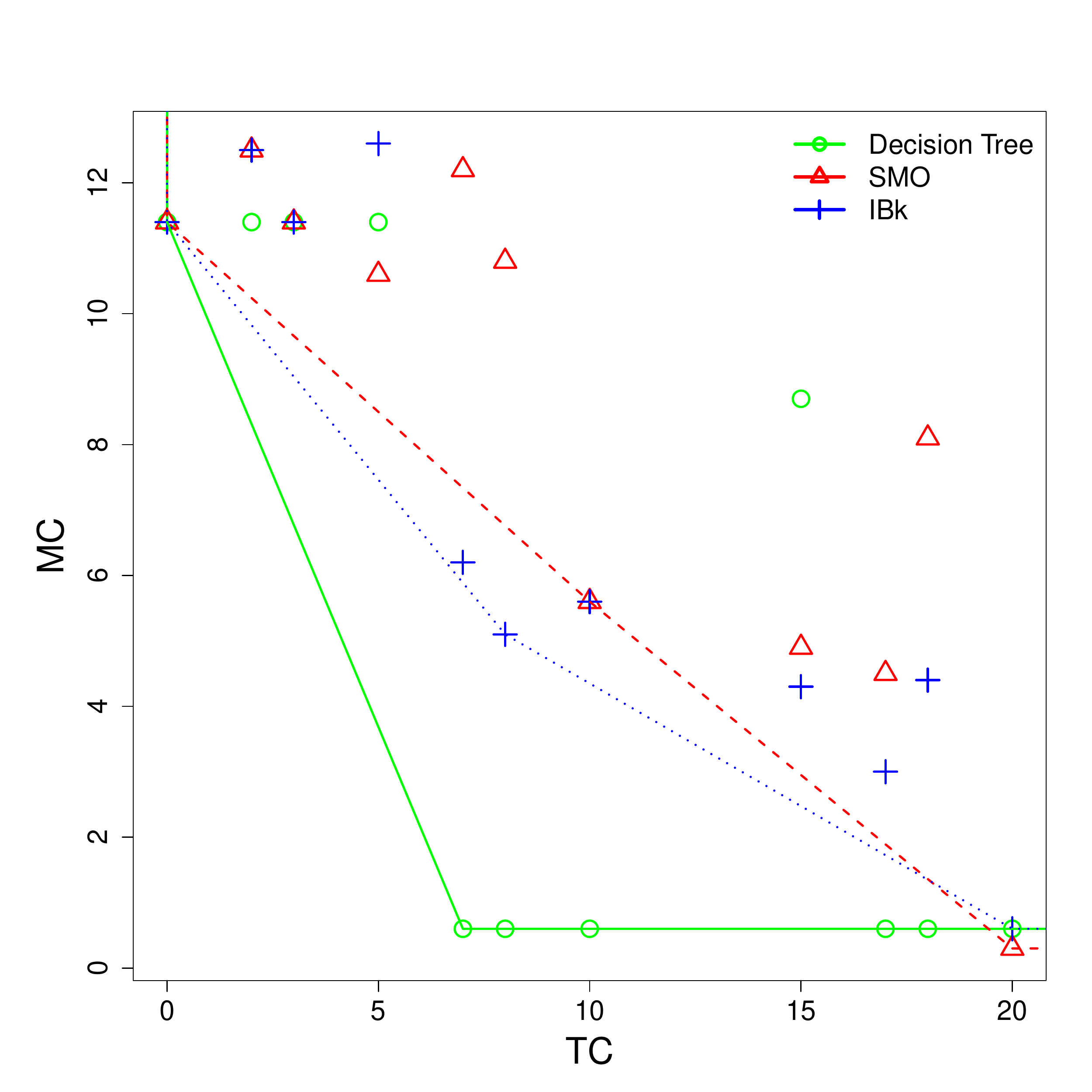} \hfill
\includegraphics[width=0.6\textwidth]{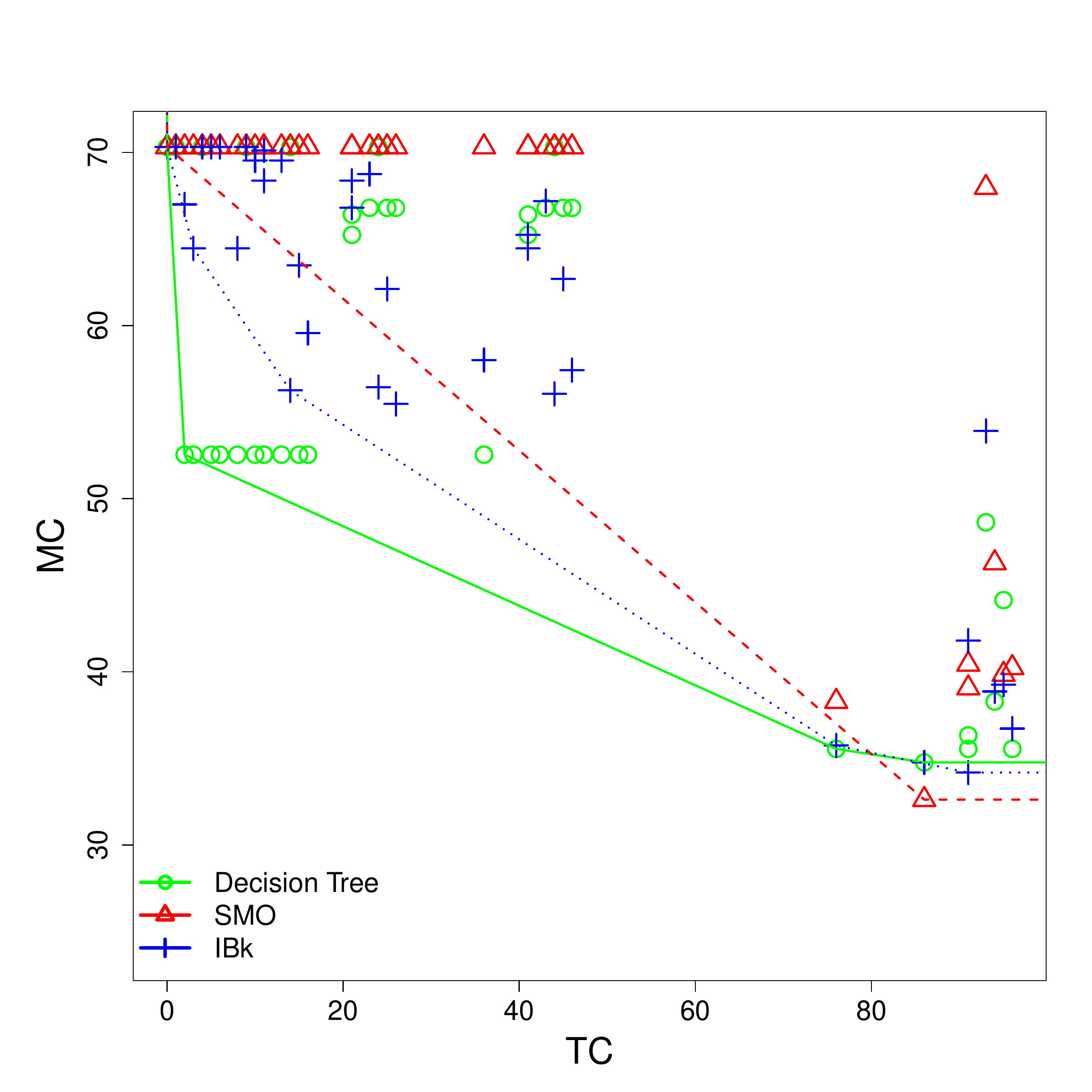} 
\caption{ Hull of the combinations selected by the method BTC. Compare with Convex Hull. Up: iris dataset with the operating context $\theta_1$. Down: Pima Indian diabetes dataset with the operating context $\theta_2$.}
\label{fig:btc}
\end{figure}

\begin{figure}
\centering
\includegraphics[width=0.6\textwidth]{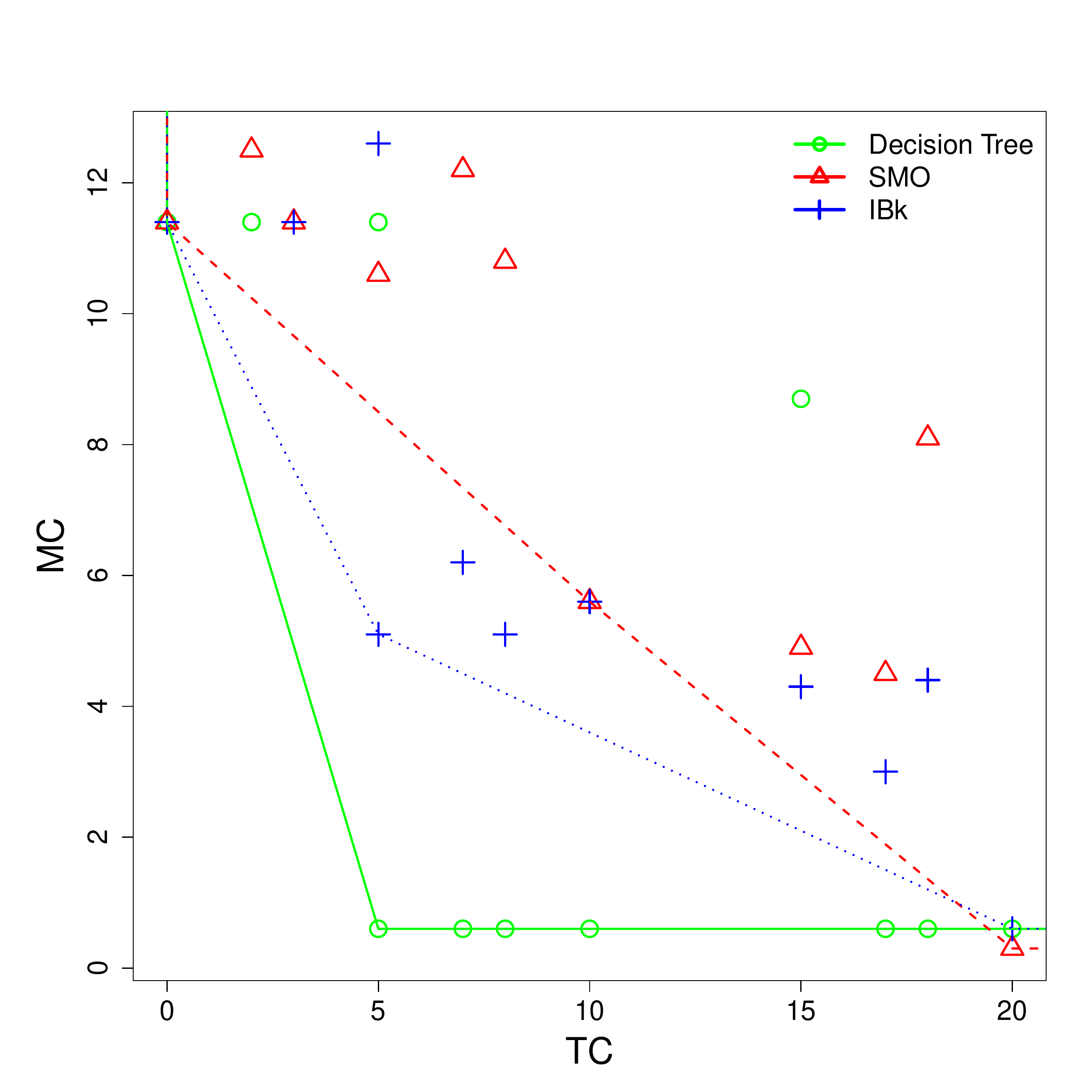} \hfill
\includegraphics[width=0.6\textwidth]{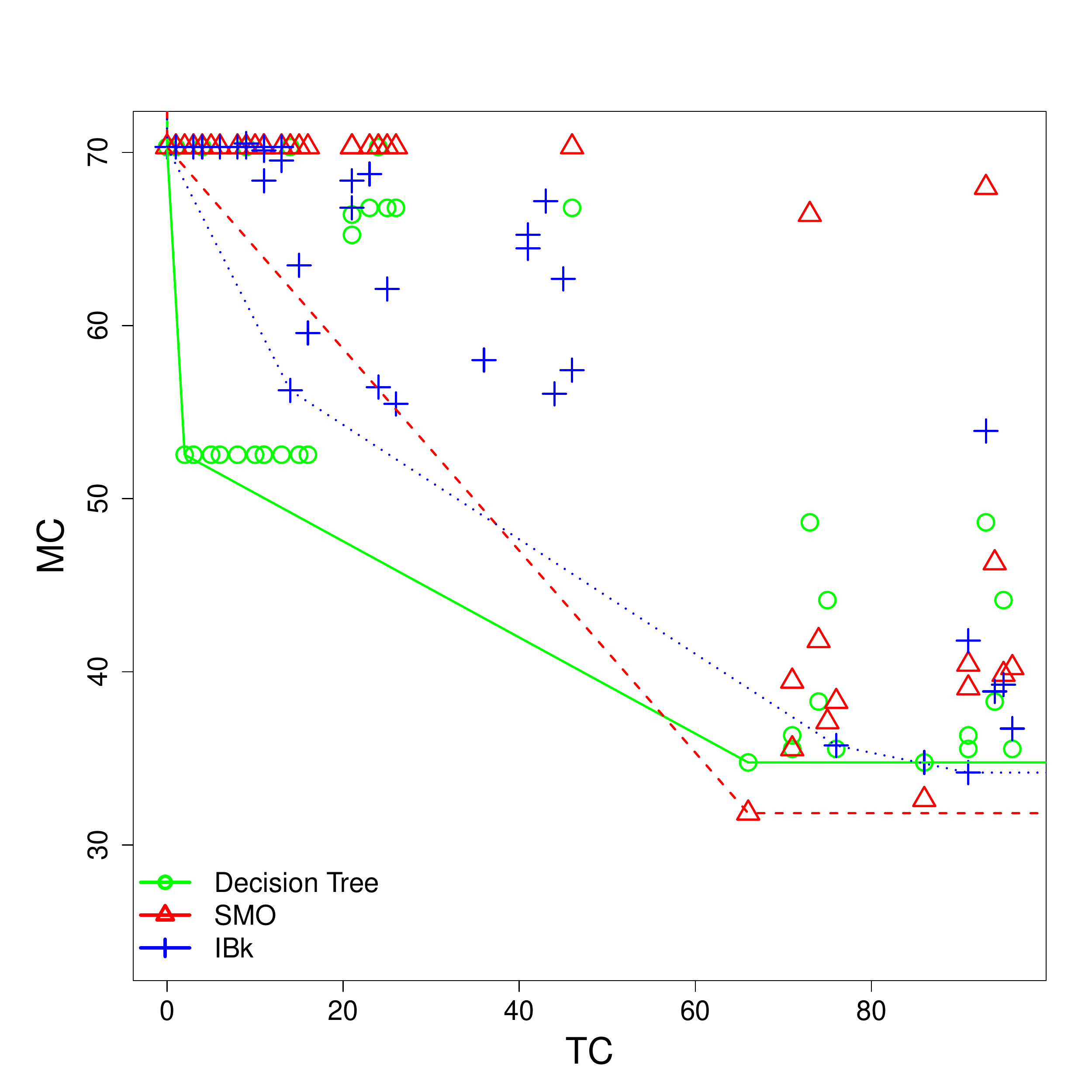} 
\caption { Hull of the combinations selected by the method BJC. Compare with Convex Hull. Up: iris dataset with the operating context $\theta_1$. Down: Pima Indian diabetes dataset with the operating context $\theta_2$.}
\label{fig:bjc}
\end{figure}

\begin{figure}
\centering
\includegraphics[width=0.6\textwidth]{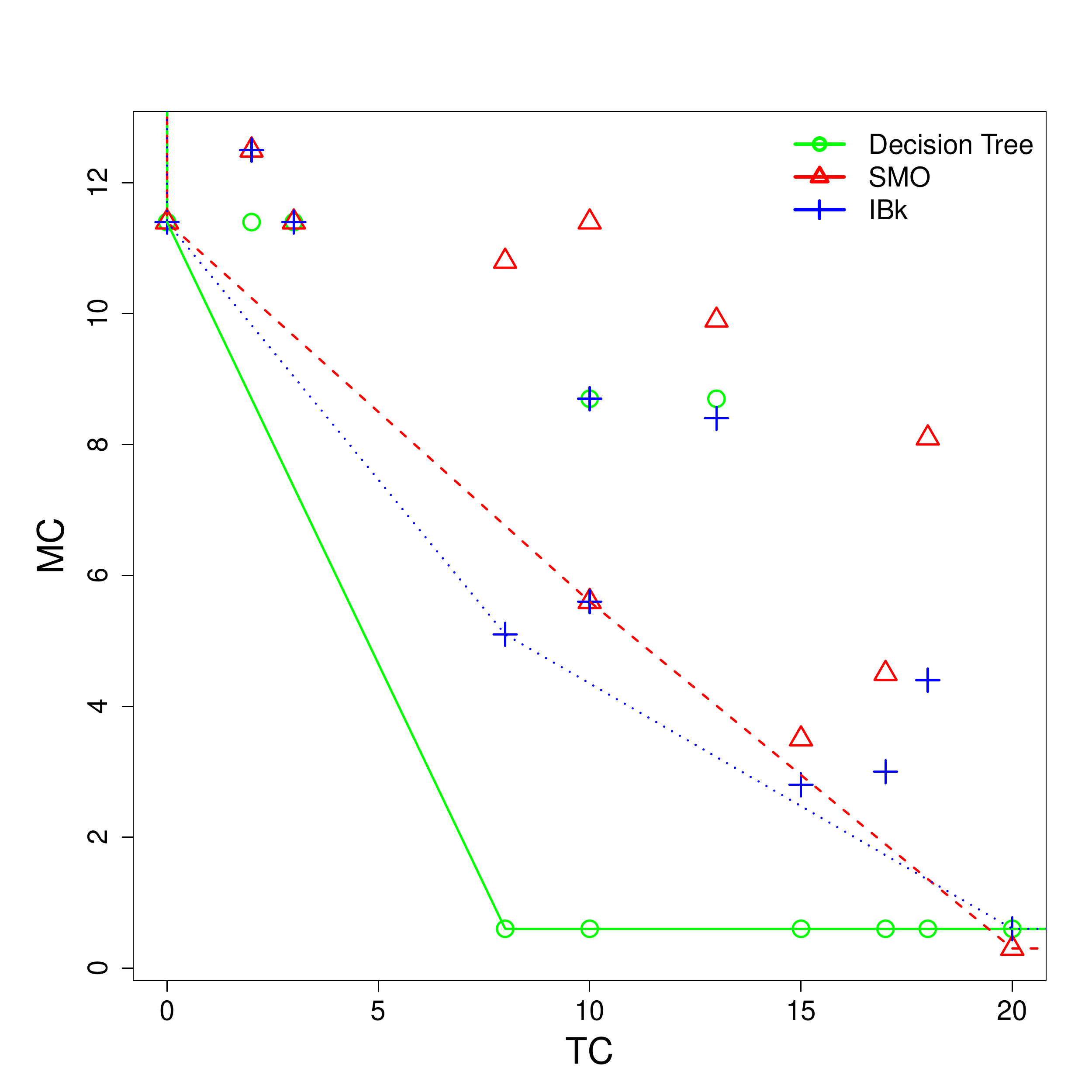} \hfill
\includegraphics[width=0.6\textwidth]{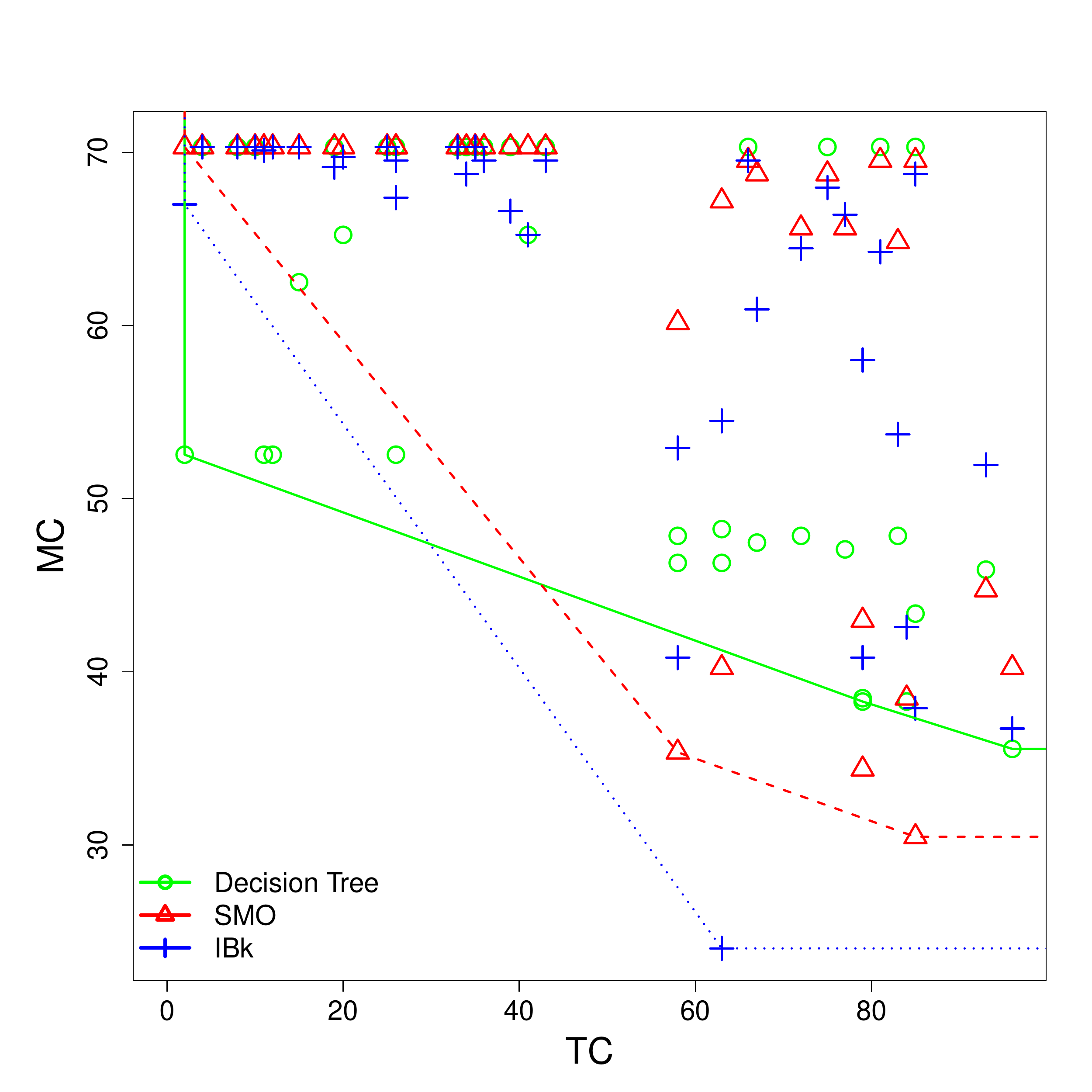} 
\caption{  Hull of the combinations selected by the method RND. Compare with Convex Hull. Up: iris dataset with the operating context $\theta_1$. Down: Pima Indian diabetes dataset with the operating context $\theta_2$.}
\label{fig:rnd}
\end{figure}

\chapter{Experiments}\label{experiments}

In this chapter we present an implementation of our approach described in sections \ref {reducingAttributes} and \ref{hull}, and an analysis of results

We are going to explore whether the \JROC plots are effective, and also whether their quadratic approximation suffers from a degradation.
In order to do that, we consider six datasets of the UCI repository, with number of attributes between 4 and 11, as shown in Table \ref{tab:datasets}.  We could not use larger datasets in this first experiment because of the slowness of the Full method, as the number of elements to explore in the lattice grows exponentially.

\begin{table}
\begin{center}
\begin{tabular}{ccccc}
\#& Name          & $m$& $n$ & $c$ \\ \hline
1 & iris          & 4  & 150 & 3 \\
2 & diabetes      & 8  & 768 & 2 \\
3 & balance-scale & 5  & 625 & 3 \\
4 & breast-w      & 11 & 320 & 2 \\
5 & breast-cancer & 10 & 286 & 2 \\
6 & glass         & 9  & 214 & 5 \\\hline

\end{tabular}
\caption{ Description of the datasets used in the experiments.}\label{tab:datasets}
\end{center}
\end{table}

We consider two different contexts: a uniform context $\theta_u$ and a random context where each value of the misclassification cost matrix and test cost vector are obtained by multiplying the original value of the uniform context by $k$, where $k= e^{\beta \times (k_0-0.5)}$, $k_0$ is obtained as a random number between 0 and 1 using a uniform distribution, and $\beta$ is a factor of how irregular we want the vector and matrix to be. We set $\beta=10$ for the following experiments. Once this function is applied, the test cost vector $T$ and the misclassification cost matrix $M$ are normalised such that $\sum T = 1$ and $\sum M = c^2$.

For each dataset of size $n$, we split it into a work dataset ($2n/3$ of the data) and the remaining data ($n/3$) for test.
With the work dataset, we perform a split of the work dataset into two halves. We train the four models (SMO, kNN, AdaBoost and Bagging) with the first half of the data ($n/3$) and calculate all the points (i.e., $TC$ and $MC$) according to the full method, and the BMC, BTC, BJC and RND methods with the other half.
Next, we choose 5 values of $\alpha \in \{ 0.1, 0.3, 0.5, 0.7, 0.9 \}$ and determine the best configuration of model and attribute subset for each of the five methods. We use these 5 configurations for the test set and calculate the $JC$. 

We repeat the experiment 4 times. This gives us, $4 \times 5=20$ results for each of the 5 methods for each of the 6 datasets.

\section{Uniform context}

First we give the results for the uniform context. Table \ref{tab:results} shows the mean and standard deviation of the results for each dataset and method. We see that Full cannot be improved by the other methods, as it explores all the possibilities. In general, the RND method is worse than the backward methods, except for dataset 1 (the smaller one, iris, where the number of explored configurations is $4 \times 5 + 1 = 11$ in front of a total of 16, which is not a big difference). In fact, for the big datasets, where the difference in explored configuration grows exponentially, we see that the backward methods get close to the Full methods, which gives support to these approximation.

\begin{table}
\centering
\begin{tabular}{rlllll}
  \hline
Dataset & Full & BMC & BTC & BJC & RND \\ \hline
  1 & 0.166 $\pm$ 0.0713 & 0.222 $\pm$ 0.105 & 0.207 $\pm$ 0.0909 & 0.222 $\pm$ 0.105 & 0.181 $\pm$ 0.0779 \\ 
  2 & 0.124 $\pm$ 0.0358 & 0.138 $\pm$ 0.0487 & 0.132 $\pm$ 0.0436 & 0.138 $\pm$ 0.0487 & 0.192 $\pm$ 0.0429 \\ 
  3 & 0.281 $\pm$ 0.156 & 0.286 $\pm$ 0.158 & 0.29 $\pm$ 0.162 & 0.286 $\pm$ 0.158 & 0.344 $\pm$ 0.122 \\ 
  4 & 0.289 $\pm$ 0.138 & 0.295 $\pm$ 0.142 & 0.293 $\pm$ 0.142 & 0.295 $\pm$ 0.142 & 0.358 $\pm$ 0.097 \\ 
  5 & 0.303 $\pm$ 0.123 & 0.328 $\pm$ 0.122 & 0.333 $\pm$ 0.13 & 0.328 $\pm$ 0.122 & 0.374 $\pm$ 0.0943 \\ 
  6 & 0.275 $\pm$ 0.129 & 0.28 $\pm$ 0.132 & 0.278 $\pm$ 0.131 & 0.28 $\pm$ 0.132 & 0.297 $\pm$ 0.126 \\ 
   \hline
\end{tabular}
\caption{JC results (mean and standard deviation) of the 20 results (5 values of alpha with 4 repetitions) for each of the 5 methods (columns: Full, BMC, BTC, BJC, RND) and each of the 6 datasets (rows: 1 to 6), with the uniform context.}\label{tab:results} 
\end{table}

Table \ref{tab:results2} shows the results aggregated for all datasets but showing each value of $\alpha$. This means ($8$ datasets with $10$ repetitions). In this case, we can see that the backward methods are consistently better than the RND method and are reasonable close to the Full method. The influence of $\alpha$ is not particularly clear, the approximation is similar for all of them.

\begin{table}
\centering
\begin{tabular}{rlllll}
  \hline
$\alpha$ & Full & BMC & BTC & BJC & RND \\ \hline
  0.1 & 0.0776 $\pm$ 0.0175 & 0.082 $\pm$ 0.0219 & 0.082 $\pm$ 0.0219 & 0.082 $\pm$ 0.0219 & 0.187 $\pm$ 0.0726 \\ 
  0.3 & 0.204 $\pm$ 0.0366 & 0.233 $\pm$ 0.0476 & 0.22 $\pm$ 0.0482 & 0.233 $\pm$ 0.0476 & 0.259 $\pm$ 0.0696 \\ 
  0.5 & 0.294 $\pm$ 0.0914 & 0.329 $\pm$ 0.0919 & 0.319 $\pm$ 0.0867 & 0.329 $\pm$ 0.0919 & 0.335 $\pm$ 0.085 \\ 
  0.7 & 0.33 $\pm$ 0.121 & 0.345 $\pm$ 0.115 & 0.352 $\pm$ 0.122 & 0.345 $\pm$ 0.115 & 0.359 $\pm$ 0.122 \\ 
  0.9 & 0.292 $\pm$ 0.155 & 0.302 $\pm$ 0.15 & 0.305 $\pm$ 0.154 & 0.302 $\pm$ 0.15 & 0.315 $\pm$ 0.163 \\ 
   \hline
\end{tabular}
\caption{JC results (mean and standard deviation) of the 24 results
(6 datasets with 4 repetitions) 
for each of the 5 methods (columns: Full, BMC, BTC, BJC, RND) and each of the 5 possible values of $\alpha$ (rows 0.1 to 0.9), with the uniform context.}\label{tab:results2}
\end{table}

Finally, we want to see the whole picture and perform a statistical test.
In order to assess the significance of the experimental results we will use a custom procedure, following \cite{japkowicz2011evaluating} and \cite[ch.12]{petersbook}, which in turn is mostly based on \cite{demsar2006statistical}. Since we will not have any baseline method, we will use a Friedman test to tell whether the difference between several methods is significant and then we will apply the Nemenyi post-hoc test. We agree with \cite{garcia2008extension} that the Nemenyi test is a ``very conservative procedure and many of the obvious differences may not be detected", but we prefer to be conservative given our experimental setting and the use of a 0.95 confidence level. In some result tables we will show the means (even though in many cases they are not commensurate) and in some other tables we will show the average ranks (from which the Friedman and Nemenyi tests are calculated). We will also include the critical difference for the Nemenyi test, so we will be able to simply tell whether the difference between two algorithms is significant if the difference between their average ranks is greater than the critical difference.

Table \ref{tab:results3} shows the results of several data, where we are particularly interested in knowing which of the three backward methods is best. As we can see, the three methods behave almost equally. In fact, BMC and BJC are exactly equal, which is a consequence of the use of a uniform context. The 30 rows are given by 6 datasets and 5 possible values of $\alpha$ with the uniform context. The `Avg' row shows the averages of the first 30 rows. Finally, the `AR' row shows the average rank for each method. With these ranks the Friedman test is applied and gives a Friedman statistic of 62.51 which is greather than the critical value of 10.97. Consequently, the null hypothesis is rejected (significance level: 0.05) and we conclude that the methods do not perform equally. In order to see which methods are significantly different from the rest, we look at the critical difference for the Nemenyi post-hoc test, which is 0.2965. This means that the Full method is statistically better than the rest, and that the RND method is statistically worse than the rest, but there is no statistically significant difference between the three methods BMC, BTC and BJC.

\begin{table}
\centering
\begin{tabular}{rrrrrr}
  \hline
 & Full & BMC & BTC & BJC & RND \\ 
  \hline
1 & 0.0953 & 0.1058 & 0.1058 & 0.1058 & 0.1388 \\ 
  2 & 0.2122 & 0.2767 & 0.2642 & 0.2767 & 0.2122 \\ 
  3 & 0.2250 & 0.3362 & 0.3013 & 0.3362 & 0.2562 \\ 
  4 & 0.2047 & 0.2715 & 0.2528 & 0.2715 & 0.2047 \\ 
  5 & 0.0905 & 0.1217 & 0.1092 & 0.1217 & 0.0915 \\ 
  6 & 0.0674 & 0.0674 & 0.0674 & 0.0674 & 0.2220 \\ 
  7 & 0.1532 & 0.2001 & 0.1581 & 0.2001 & 0.2023 \\ 
  8 & 0.1504 & 0.1649 & 0.1660 & 0.1649 & 0.2146 \\ 
  9 & 0.1340 & 0.1340 & 0.1340 & 0.1340 & 0.1834 \\ 
  10 & 0.1169 & 0.1238 & 0.1338 & 0.1238 & 0.1397 \\ 
  11 & 0.0542 & 0.0542 & 0.0542 & 0.0542 & 0.2028 \\ 
  12 & 0.1797 & 0.1797 & 0.1797 & 0.1797 & 0.2380 \\ 
  13 & 0.3125 & 0.3264 & 0.3264 & 0.3264 & 0.3595 \\ 
  14 & 0.4002 & 0.4036 & 0.4068 & 0.4036 & 0.4429 \\ 
  15 & 0.4571 & 0.4649 & 0.4807 & 0.4649 & 0.4774 \\ 
  16 & 0.0738 & 0.0738 & 0.0738 & 0.0738 & 0.2090 \\ 
  17 & 0.2074 & 0.2074 & 0.2074 & 0.2074 & 0.3211 \\ 
  18 & 0.3262 & 0.3457 & 0.3311 & 0.3457 & 0.3779 \\ 
  19 & 0.4186 & 0.4277 & 0.4295 & 0.4277 & 0.4416 \\ 
  20 & 0.4168 & 0.4213 & 0.4217 & 0.4213 & 0.4393 \\ 
  21 & 0.0975 & 0.1117 & 0.1117 & 0.1117 & 0.2225 \\ 
  22 & 0.2419 & 0.3001 & 0.2790 & 0.3001 & 0.3511 \\ 
  23 & 0.3736 & 0.4035 & 0.3958 & 0.4035 & 0.4188 \\ 
  24 & 0.3904 & 0.3988 & 0.4521 & 0.3988 & 0.4321 \\ 
  25 & 0.4137 & 0.4254 & 0.4273 & 0.4254 & 0.4454 \\ 
  26 & 0.0775 & 0.0793 & 0.0793 & 0.0793 & 0.1282 \\ 
  27 & 0.2309 & 0.2309 & 0.2309 & 0.2309 & 0.2309 \\ 
  28 & 0.3792 & 0.3980 & 0.3916 & 0.3980 & 0.3829 \\ 
  29 & 0.4346 & 0.4346 & 0.4346 & 0.4346 & 0.4475 \\ 
  30 & 0.2550 & 0.2550 & 0.2550 & 0.2550 & 0.2974 \\  \hline
  Avg & 0.2397 & 0.2581 & 0.2554 & 0.2581 & 0.2911 \\  \hline
  AR & 1.5000 & 3.0500 & 3.0667 & 3.0500 & 4.3333 \\ 
   \hline
\end{tabular}
\caption{This figure shows the JC means for the 4 repetitions for each of the 5 methods (Full, BMC, BTC, BJC, RND).}\label{tab:results3}
\end{table}

\section{Variable context}

In order to see what happens in a more realistic situation, let us see the results for the non-uniform context. Table \ref{tab:results-UR} shows the mean and standard deviation of the results for each dataset and method. Here we see that not all backward methods are equivalently, but interestingly we see that BMC is now consistently better than RND for all datasets.

\begin{table}
\centering
\begin{tabular}{rlllll}
  \hline
Dataset & Full & BMC & BTC & BJC & RND \\ \hline
  1 & 0.0079 $\pm$ 0.011 & 0.011 $\pm$ 0.013 & 0.015 $\pm$ 0.025 & 0.018 $\pm$ 0.028 & 0.017 $\pm$ 0.04 \\ 
  2 & 0.016 $\pm$ 0.016 & 0.023 $\pm$ 0.024 & 0.016 $\pm$ 0.017 & 0.025 $\pm$ 0.033 & 0.032 $\pm$ 0.033 \\ 
  3 & 0.037 $\pm$ 0.035 & 0.039 $\pm$ 0.035 & 0.037 $\pm$ 0.035 & 0.039 $\pm$ 0.036 & 0.045 $\pm$ 0.044 \\ 
  4 & 0.045 $\pm$ 0.049 & 0.052 $\pm$ 0.057 & 0.05 $\pm$ 0.057 & 0.052 $\pm$ 0.057 & 0.058 $\pm$ 0.057 \\ 
  5 & 0.0026 $\pm$ 0.0036 & 0.004 $\pm$ 0.0056 & 0.0033 $\pm$ 0.0037 & 0.0034 $\pm$ 0.004 & 0.0048 $\pm$ 0.0046 \\ 
  6 & 0.012 $\pm$ 0.017 & 0.02 $\pm$ 0.028 & 0.024 $\pm$ 0.041 & 0.024 $\pm$ 0.041 & 0.023 $\pm$ 0.03 \\ 
   \hline
\end{tabular}
\caption{JC results (mean and standard deviation) of the 20 results (5 values of alpha with 4 repetitions) for each of the 5 methods (columns: Full, BMC, BTC, BJC, RND) and each of the 6 datasets (rows: 1 to 6), with the variable context.}\label{tab:results-UR} 
\end{table}

Again, Table \ref{tab:results2-UR} shows the results aggregated for all datasets but showing each value of $\alpha$. This means ($8$ datasets with $10$ repetitions). Now we see that the results are not especially different according to $\alpha$.

\begin{table}
\centering
\begin{tabular}{rlllll}
  \hline
$\alpha$ & Full & BMC & BTC & BJC & RND \\ \hline
  0.1 & 0.0055 $\pm$ 0.0064 & 0.0081 $\pm$ 0.01 & 0.006 $\pm$ 0.0067 & 0.0061 $\pm$ 0.0067 & 0.012 $\pm$ 0.021 \\ 
  0.3 & 0.014 $\pm$ 0.021 & 0.018 $\pm$ 0.028 & 0.015 $\pm$ 0.023 & 0.022 $\pm$ 0.034 & 0.024 $\pm$ 0.032 \\ 
  0.5 & 0.025 $\pm$ 0.034 & 0.037 $\pm$ 0.047 & 0.04 $\pm$ 0.054 & 0.042 $\pm$ 0.055 & 0.046 $\pm$ 0.06 \\ 
  0.7 & 0.022 $\pm$ 0.027 & 0.025 $\pm$ 0.028 & 0.024 $\pm$ 0.028 & 0.025 $\pm$ 0.028 & 0.031 $\pm$ 0.032 \\ 
  0.9 & 0.033 $\pm$ 0.044 & 0.036 $\pm$ 0.044 & 0.035 $\pm$ 0.044 & 0.039 $\pm$ 0.045 & 0.038 $\pm$ 0.045 \\ 
   \hline
\end{tabular}
\caption{JC results (mean and standard deviation) of the 24 results
(6 datasets with 4 repetitions) 
for each of the 5 methods (columns: Full, BMC, BTC, BJC, RND) and each of the 5 possible values of $\alpha$ (rows 0.1 to 0.9), with the variable context.}\label{tab:results2-UR}
\end{table}

Finally, if we look at the whole picture and using a statistical test, we see in Table \ref{tab:results3-UR} that the backward methods are better than the RND method, but now we find difference between them. In fact, BTC is significantly better than BMC and BJC. The 30 rows are given by 6 datasets and 5 possible values of $\alpha$ with the variable context. The `Avg' row shows the averages of the first 30 rows. Finally, the `AR' row shows the average rank for each method. With these ranks the Friedman test is applied and gives a Friedman statistic of 67.27 which is greater than the critical value of 10.97. Consequently, the null hypothesis is rejected (significance level: 0.05) and we conclude that the methods do not perform equally. In order to see which methods are significantly different from the rest, we look at the critical difference for the Nemenyi post-hoc test, which is 0.2965. This means that the Full method is statistically better than the rest, and that the RND method is statistically worse than the rest. In this case, we see that the BTC method is significantly better than BMC and BJC.

\begin{table}
\centering
\begin{tabular}{rrrrrr}
  \hline
 & Full & BMC & BTC & BJC & RND \\ 
  \hline
1 & 0.0023 & 0.0025 & 0.0025 & 0.0025 & 0.0025 \\ 
  2 & 0.0071 & 0.0071 & 0.0073 & 0.0073 & 0.0100 \\ 
  3 & 0.0160 & 0.0258 & 0.0457 & 0.0457 & 0.0507 \\ 
  4 & 0.0085 & 0.0114 & 0.0117 & 0.0117 & 0.0087 \\ 
  5 & 0.0054 & 0.0057 & 0.0057 & 0.0210 & 0.0124 \\ 
  6 & 0.0078 & 0.0205 & 0.0078 & 0.0082 & 0.0338 \\ 
  7 & 0.0138 & 0.0306 & 0.0138 & 0.0538 & 0.0482 \\ 
  8 & 0.0141 & 0.0191 & 0.0145 & 0.0157 & 0.0267 \\ 
  9 & 0.0171 & 0.0185 & 0.0171 & 0.0180 & 0.0258 \\ 
  10 & 0.0248 & 0.0279 & 0.0254 & 0.0277 & 0.0267 \\ 
  11 & 0.0116 & 0.0129 & 0.0129 & 0.0129 & 0.0171 \\ 
  12 & 0.0358 & 0.0361 & 0.0358 & 0.0358 & 0.0398 \\ 
  13 & 0.0594 & 0.0611 & 0.0594 & 0.0606 & 0.0772 \\ 
  14 & 0.0563 & 0.0585 & 0.0580 & 0.0585 & 0.0665 \\ 
  15 & 0.0206 & 0.0240 & 0.0207 & 0.0254 & 0.0241 \\ 
  16 & 0.0098 & 0.0098 & 0.0098 & 0.0098 & 0.0156 \\ 
  17 & 0.0243 & 0.0299 & 0.0290 & 0.0274 & 0.0343 \\ 
  18 & 0.0481 & 0.0728 & 0.0645 & 0.0728 & 0.0804 \\ 
  19 & 0.0253 & 0.0253 & 0.0253 & 0.0253 & 0.0310 \\ 
  20 & 0.1191 & 0.1220 & 0.1219 & 0.1223 & 0.1264 \\ 
  21 & 0.0003 & 0.0006 & 0.0003 & 0.0005 & 0.0010 \\ 
  22 & 0.0012 & 0.0026 & 0.0025 & 0.0026 & 0.0065 \\ 
  23 & 0.0013 & 0.0028 & 0.0037 & 0.0023 & 0.0039 \\ 
  24 & 0.0037 & 0.0054 & 0.0039 & 0.0052 & 0.0047 \\ 
  25 & 0.0062 & 0.0084 & 0.0063 & 0.0065 & 0.0078 \\ 
  26 & 0.0014 & 0.0026 & 0.0026 & 0.0026 & 0.0014 \\ 
  27 & 0.0025 & 0.0035 & 0.0035 & 0.0035 & 0.0025 \\ 
  28 & 0.0123 & 0.0387 & 0.0548 & 0.0548 & 0.0352 \\ 
  29 & 0.0225 & 0.0278 & 0.0288 & 0.0294 & 0.0474 \\ 
  30 & 0.0216 & 0.0270 & 0.0283 & 0.0283 & 0.0308 \\ \hline
  Avg & 0.0200 & 0.0247 & 0.0241 & 0.0266 & 0.0300 \\  \hline
  AR & 1.2333 & 3.4000 & 2.6500 & 3.5167 & 4.2000 \\ 
   \hline
\end{tabular}
\caption{This figure shows the JC means for the 4 repetitions for each of the 5 methods (Full, BMC, BTC, BJC, RND).}\label{tab:results3-UR}
\end{table}

Although some more definitive conclusions of which method is best in general would require more datasets and repetitions (although the results are significant enough here), these experiments show the potential of the backward methods.

\chapter{Conclusions and future works}\label{conclusion}

In this chapter, we summarise the accomplishments and discuss some interesting future work.

\section{Conclusion}
\vspace{0.2cm}
Missing values and cost minimisation are important issues in machine learning. In this work we have developed  a new approach that tackles the problem of cost minimisation (both misclassification and test cost) reframing an already built model using missing values on purpose. A new manner to address the cost reduction in machine learning problems. Our approach is totally model independent (can be used and works perfectly for any predictive model). Overall, our approach fulfills all the objectives fixed in the introduction:

\begin{enumerate}
\item An evaluation (graphical) of how a model improves, and most especially, degrades if we start removing attributes, finding the concept of dominance here. figure \ref{fig:accuracy} shows precisely this graphical evaluation, depicting the variation of accuracy (of two datasets trained using 3 differents models) in sawteeth form when we start removing attributes.
\item An analysis of those context changes related to attribute (feature) costs and the introduction of solutions to create more versatile models and reframing transformations. In chapter \ref{reducingAttributes}, section \ref{approach}, we define the notion of context and analyse how the operating context could change. We have also defined an approach for a reframing transformation, finding a way to build a versatile model which minimises the costs (Joint Cost), by introducing missing values on purpose. Given the exponential cost of this procedure, we introduce four new quadratic approximations in chapter \ref{hull}.
\item The evaluation of the previous approaches with several datasets, using common repositories or the reframe domains. This point is achieved with the experiments in \textbf{chapter \ref{experiments}}, where we evaluate our approach on 6 different datasets and analyse the results obtained.

\end{enumerate} 

\vspace{0.3cm}
\section{Discussion}
\vspace{0.2cm}

In the introduction we argued that we were looking for a flexible approach that could be used in a variety of circumstances. In fact, we were motivated by  the following considerations:

\begin{enumerate}
\item The method must work for any kind of predictive model, either human-made or trained from data using any off-the-shelf predictive modelling technique. \label{enu:any}
\item Each example may have a different subset of missing values. \label{enu:instances}
\item Retraining the model for each example (using a subset of the examples with similar feature subsets) is not an option (because of \ref{enu:instances} above or other reasons).
\item Both misclassification cost (MC) and test cost (TC) must be considered.
\end{enumerate}

\noindent We have presented some graphical tools and an optimisation method that meets these requirements. In fact, this has to be compared to the usual approach which is specific on decision trees, with several approaches according to \cite{zhang2005missing,lomax2013survey}: 
(a) KV, a tree is rebuilt when missing values are found, (b) Null strategy: replace by an extra label (model is not rebuilt), (c) Internal node: creates nodes for examples with missing values (model is not rebuilt), and (d) C4.5 strategy: probabilistic approach (model is not rebuilt). Option (a) is infeasible if the situation \ref{enu:instances} holds. 

All the above options are specific to decision trees, so they are not able to take advantage of many other off-the-shelf techniques of our preferred data mining suite or machine learning library. This is an important limitation as many of the most powerful machine learning techniques used today, such as ensemble methods (using or not decision trees as base classifiers), support vector machines, etc., are much more difficult to adapt for minimising test costs.

The take-away message of this work is that we can use any machine learning technique, train a model on a dataset with the available attributes and possibly containing missing values, and {\em reframe} it for a different deployment context where we have fewer available attributes, a different distribution of missing values, a different misclassification matrix and test cost vector. While the Full approach is intractable in general, we have introduced some approximations that are just quadratic, which are feasible for hundreds of attributes, which is already a high number of attributes if we are considering test costs. Also, during all the process we can explore the performance of several models using \JROC curves. In fact, these curves are not specific for the methods we have introduced here; they could be used for the traditional methods used for decision trees or for the analysis of any cost context considering both $MC$ and $TC$.
\vspace{0.3cm}

\section{Future work}

This work opens many new avenues of future work. For instance, in section \ref{reducingAttributes2} we discuss that an alternative to the use of missing values is the use of ranges (see bullet \ref{item-range}). The approach presented here could also be compared or explored in combination to the mimetic technique to get models that use fewer attributes \cite{ensemble2002,mimetic2003,Blanco-VegaHR04,icmlaBlanco-VegaHR05,blancoestimating}. Another interesting idea would be the problem of quantification with test costs, which could be applied to both classification and regression \cite{bella2010quantification,bella2013aggregative}. 
We have also been suggested\footnote{Peter Flach, personal communication} to use decision stump ensembles, where the elements in 
the ensemble could be pruned a posteriori when the test cost is known.

More comparison with the area of feature selection could lead to a better understanding of the possibilities of reframing and better methods. For instance, the use of the attribute correlation can be used to an approximate notion of dominance (e.g., if two attributes have high correlation, the cost is expected to be related to the lowest test cost for any of them). As for the relation to other problems, we could also consider that the output domain may be null, as in abstaining classifiers \cite{ferri2004cautious,Pietraszek2005} and the notion of delegation \cite{FerriFH04} could be applied to this case. In fact, a missing value on purpose can be seen as the parallel of a reject option or abstention for the output value.

The notion of \JROC curve could be further explored and extended. For instance, we could figure out other ways of drawing these curves, by using attribute correlation or some other order on the attributes. The issue of representing operating conditions when the the matrix and vector are not fixed could lead to more dimensions, or the inclusion of the cost matrix. At least in the case of binary classifiers we could have 3D surfaces, using, e.g., $TPR$, $FPR$ and $TC$. 
As for any curve representing cost for each operating condition, we wonder whether the area over the \JROC curve means something, as in ROC analysis \cite{Bradley1997,ICML11CoherentAUC}. Also we could ask the question of whether we can draw cost plots as in \cite{drummond-and-Holte2006,ICML11Brier,ROCandCost}.

Finally, there are more more ambitious ideas. We could investigate which attributes to use for each example. We could use reliability measures (especially in probabilistic classifiers) to make better decisions on whether to remove an attribute or not. We could analyse what to do when new attributes appear, using, e.g., the correlation to other attributes to derive the old attributes, or thinking about more general ways of representing the feature space. Finally, we think that there is no reason why most of the ideas introduced here could not work equally well for regression, combining the test cost with any regression loss.

\section{Dissemination}

This work was presented in Brussels, in the REFRAME meeting (the complete information can be found here: \url{http://www.reframe-d2k.org/index.php/Meetings_%26_Workshops} ). 
Part of this report have been published as the ARXIV paper $1305.7111$: ``Test cost and misclassification cost trade-off using reframing'' \cite{periale2013test}.

\newpage
\cleardoublepage

\appendix 
\addappheadtotoc 
\appendixpage

\chapter{Details of the results}\label{Results}

In this appendix, the results of the experiments presented in the section \ref{experiments} are detailed.
In fact the experiments presented in figure \ref{fig:allcostsU} and figure \ref{fig:allcostsU} have been also performed using other machine learning methods: IBk, Bagging and Boosting (see section \ref{ada}). The different results obtained are shown in the following plots.

\begin{figure}
\centering
\includegraphics[width=0.48\textwidth]{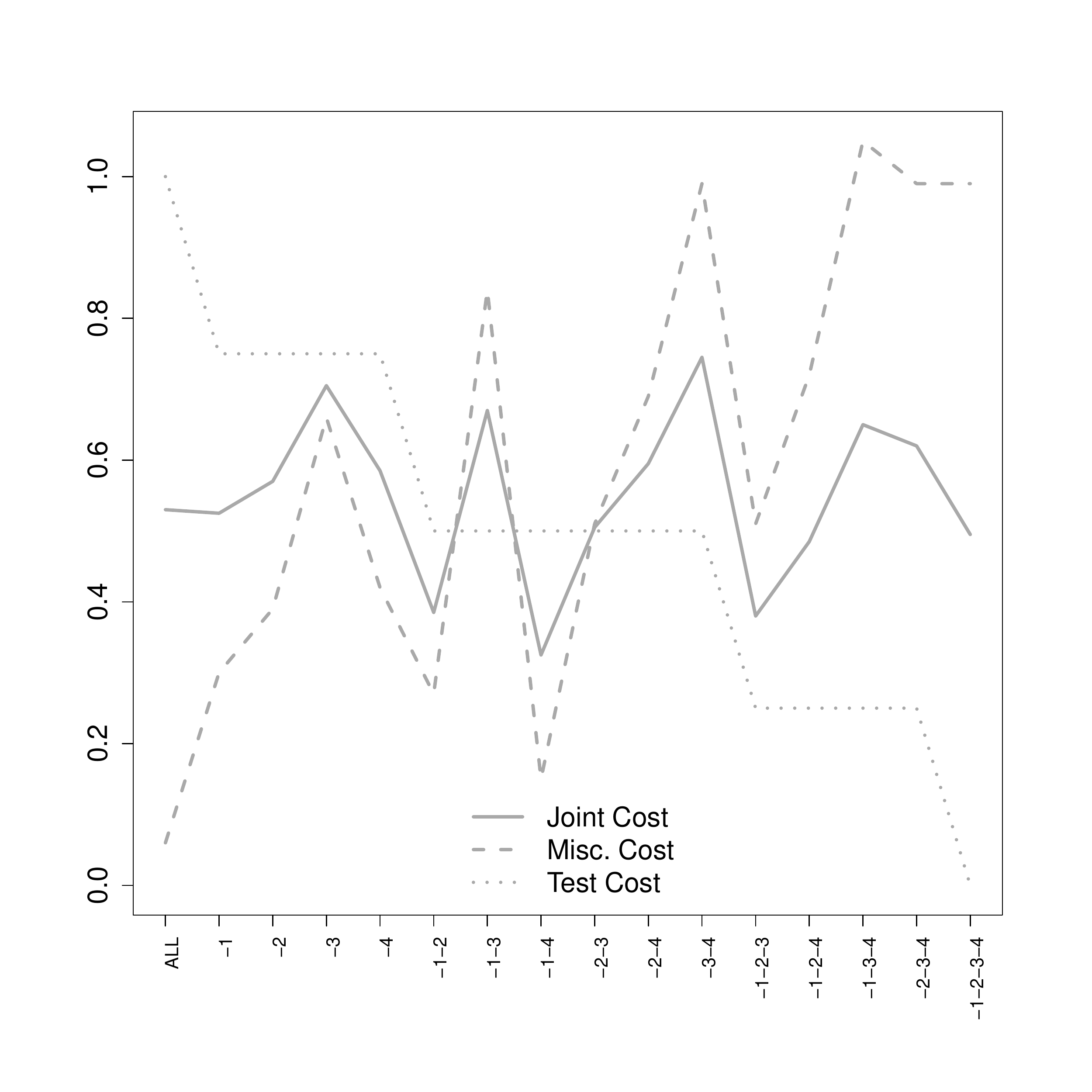} \hfill
\includegraphics[width=0.48\textwidth]{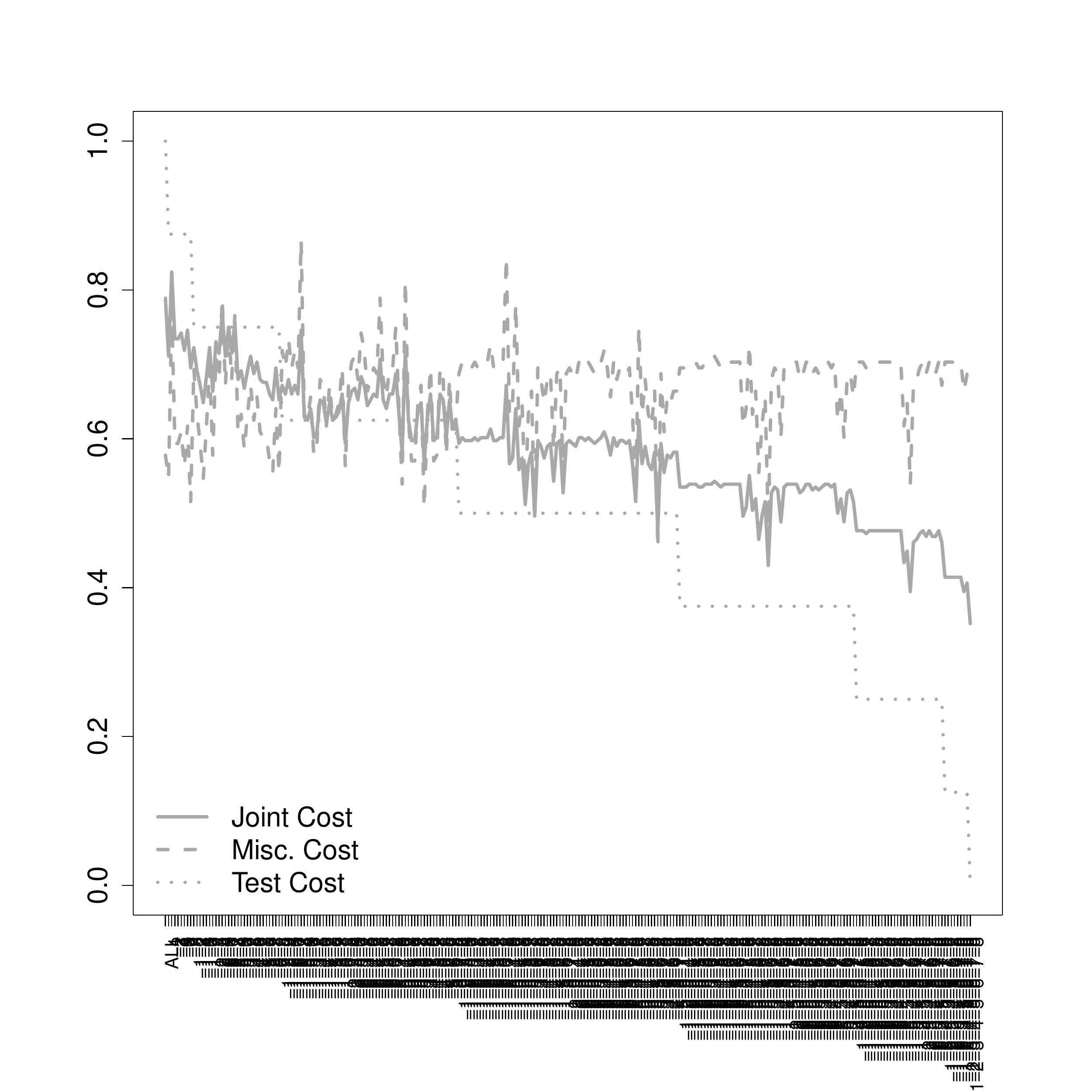} 
\caption {Evolution of $MC$, $TC$ and $JC$ according to attribute selection for a IBk (KNN) model using uniform context. Left: Iris dataset, Right: Diabetes dataset}
\label{fig:allcostsuibk}
\end{figure}

\begin{figure}
\centering
\includegraphics[width=0.48\textwidth]{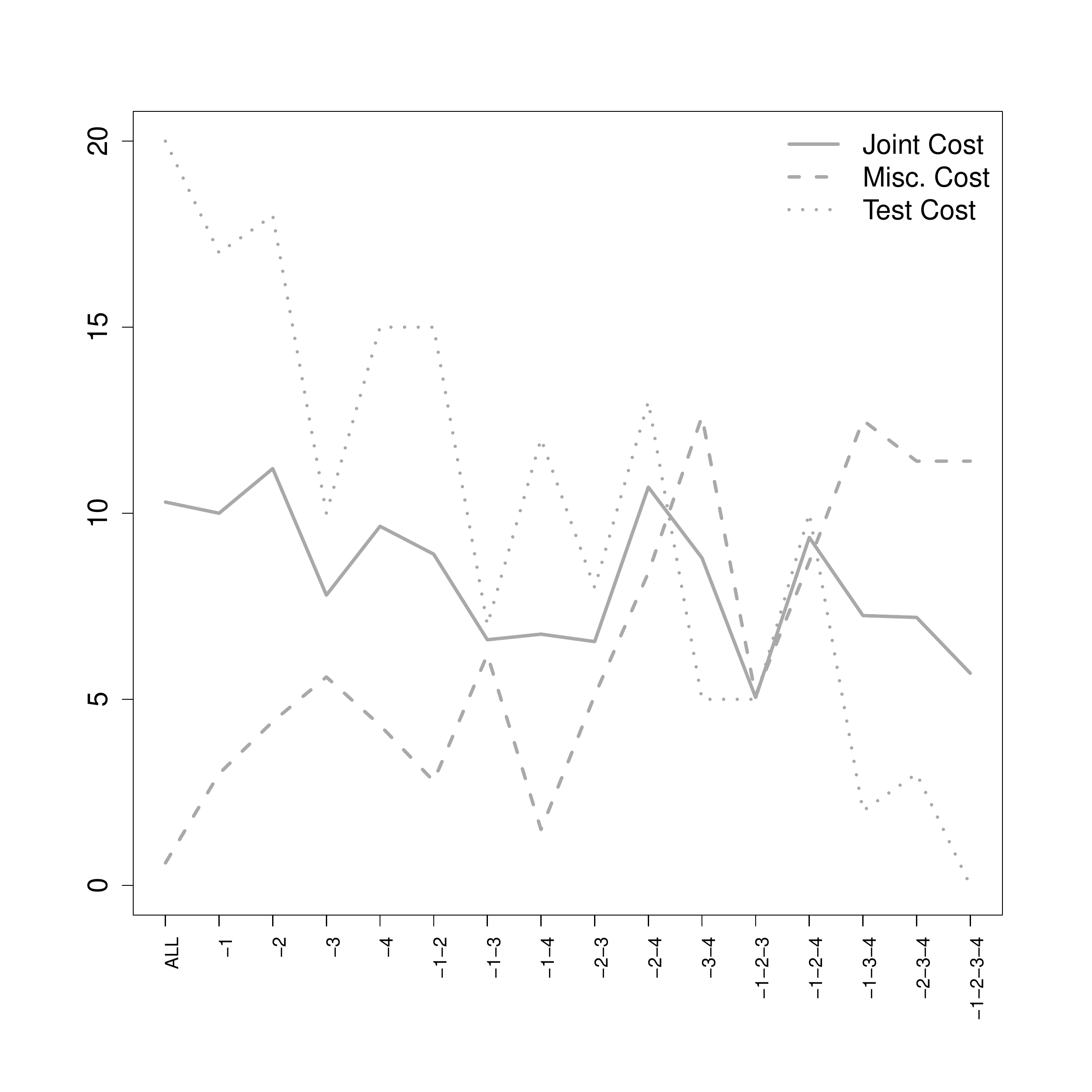} \hfill
\includegraphics[width=0.48\textwidth]{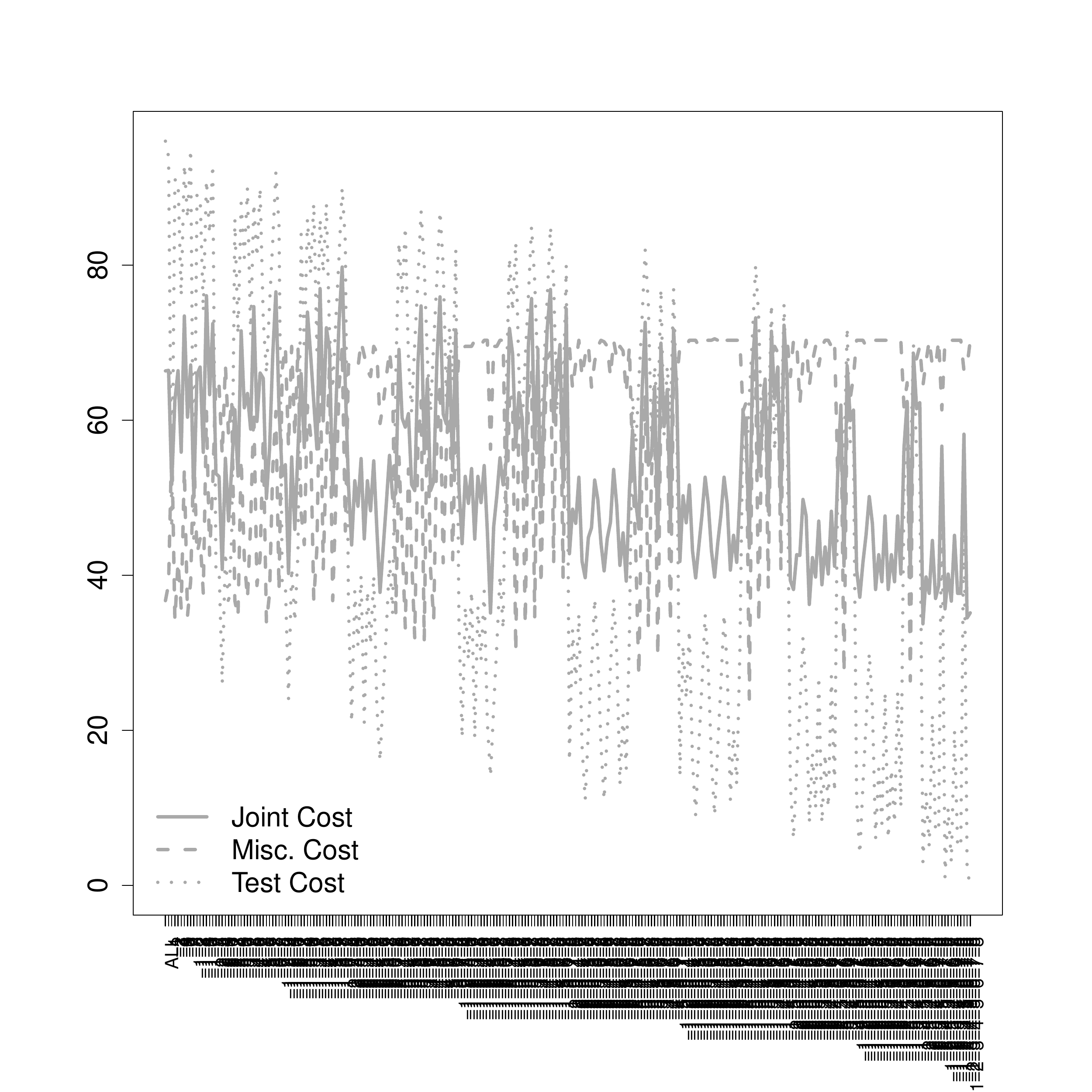} 
\caption {Evolution of $MC$, $TC$ and $JC$ according to attribute selection for a IBk (KNN) model using non-uniform context. Left: Iris dataset, Right: Diabetes dataset}
\label{fig:allcosts0102ibk}
\end{figure}

\begin{figure}
\centering
\includegraphics[width=0.48\textwidth]{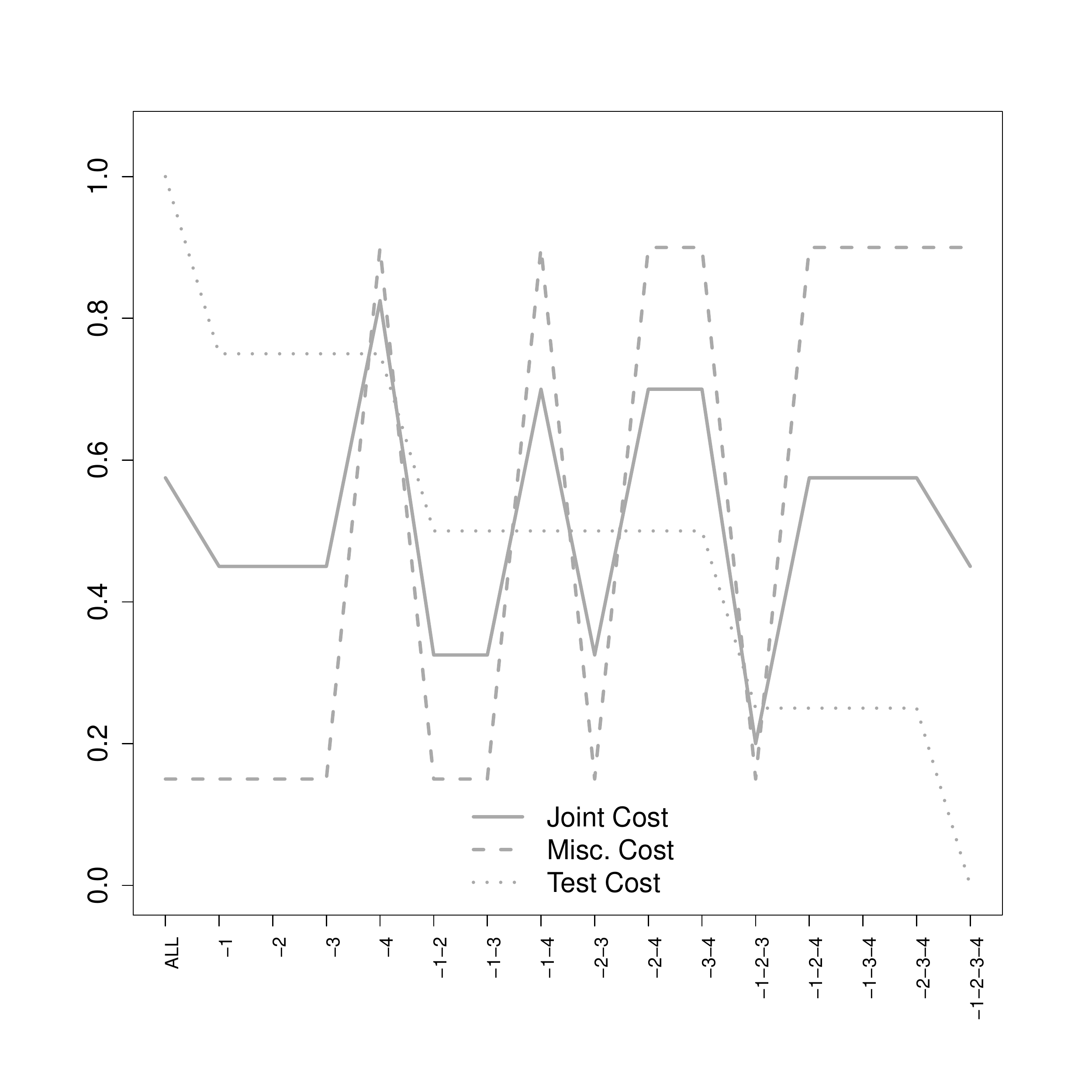} \hfill
\includegraphics[width=0.48\textwidth]{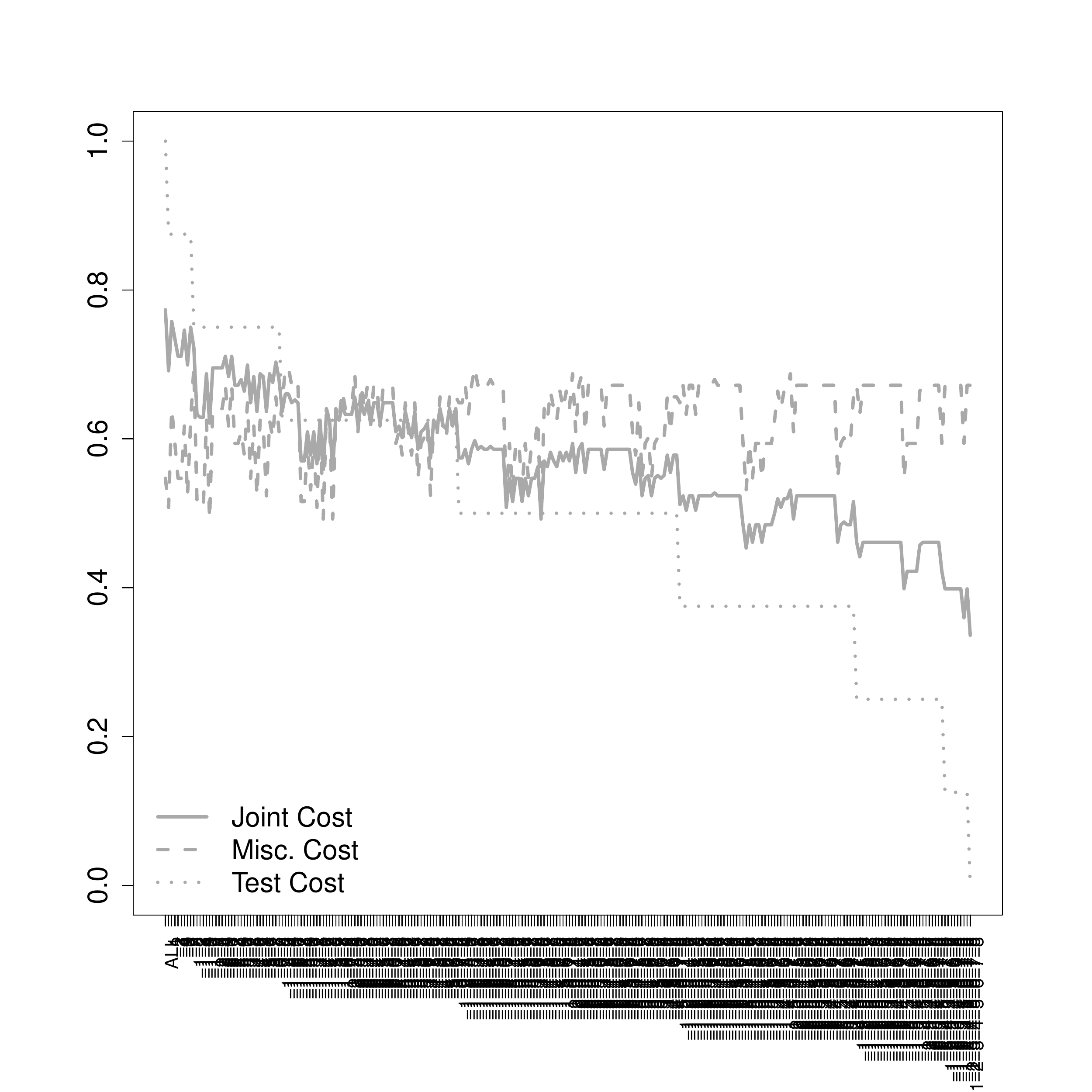} 
\caption {Evolution of $MC$, $TC$ and $JC$ according to attribute selection for a Adaboost model using uniform context. Left: Iris dataset, Right: Diabetes dataset}
\label{fig:allcostsuada}
\end{figure}

\begin{figure}
\centering
\includegraphics[width=0.48\textwidth]{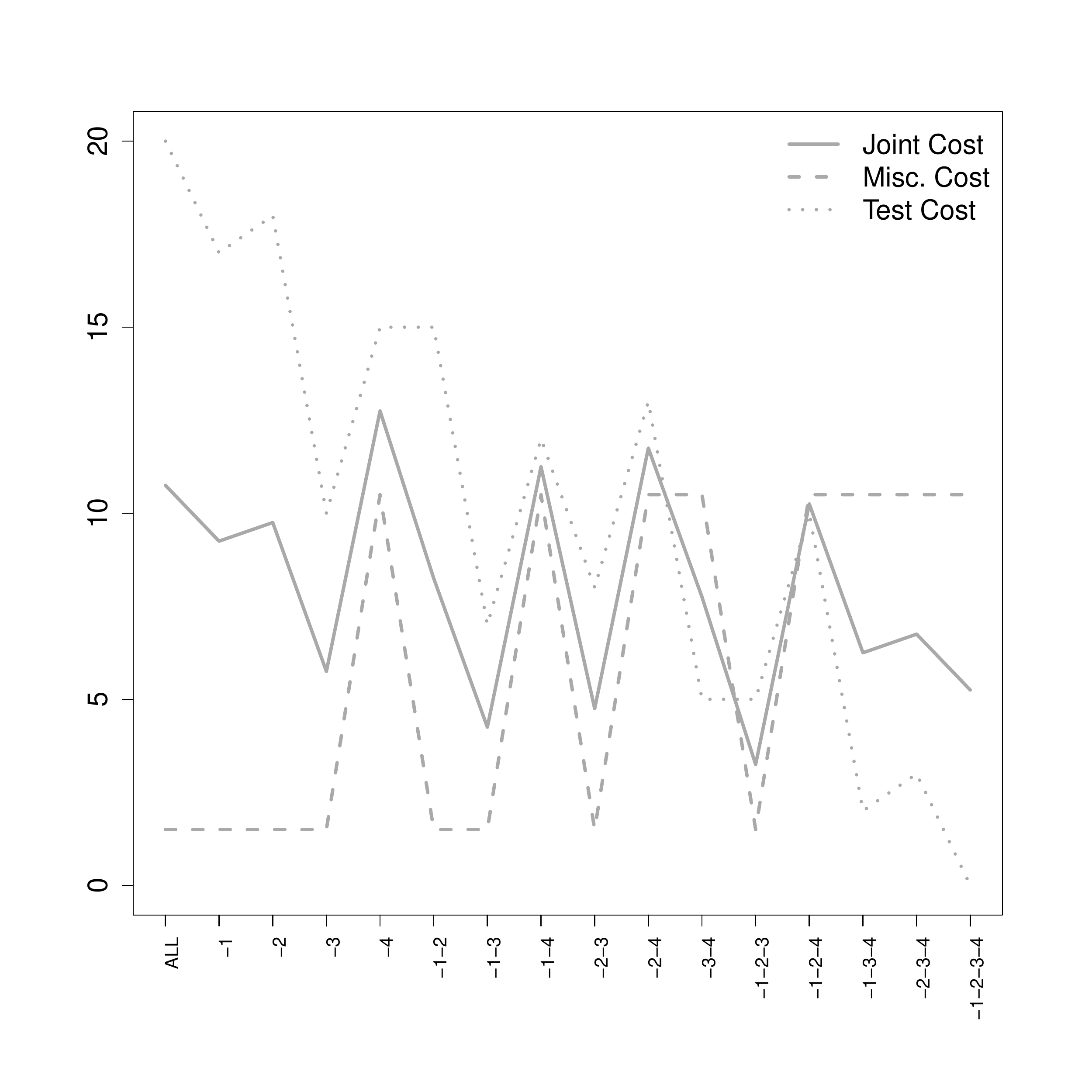} \hfill
\includegraphics[width=0.48\textwidth]{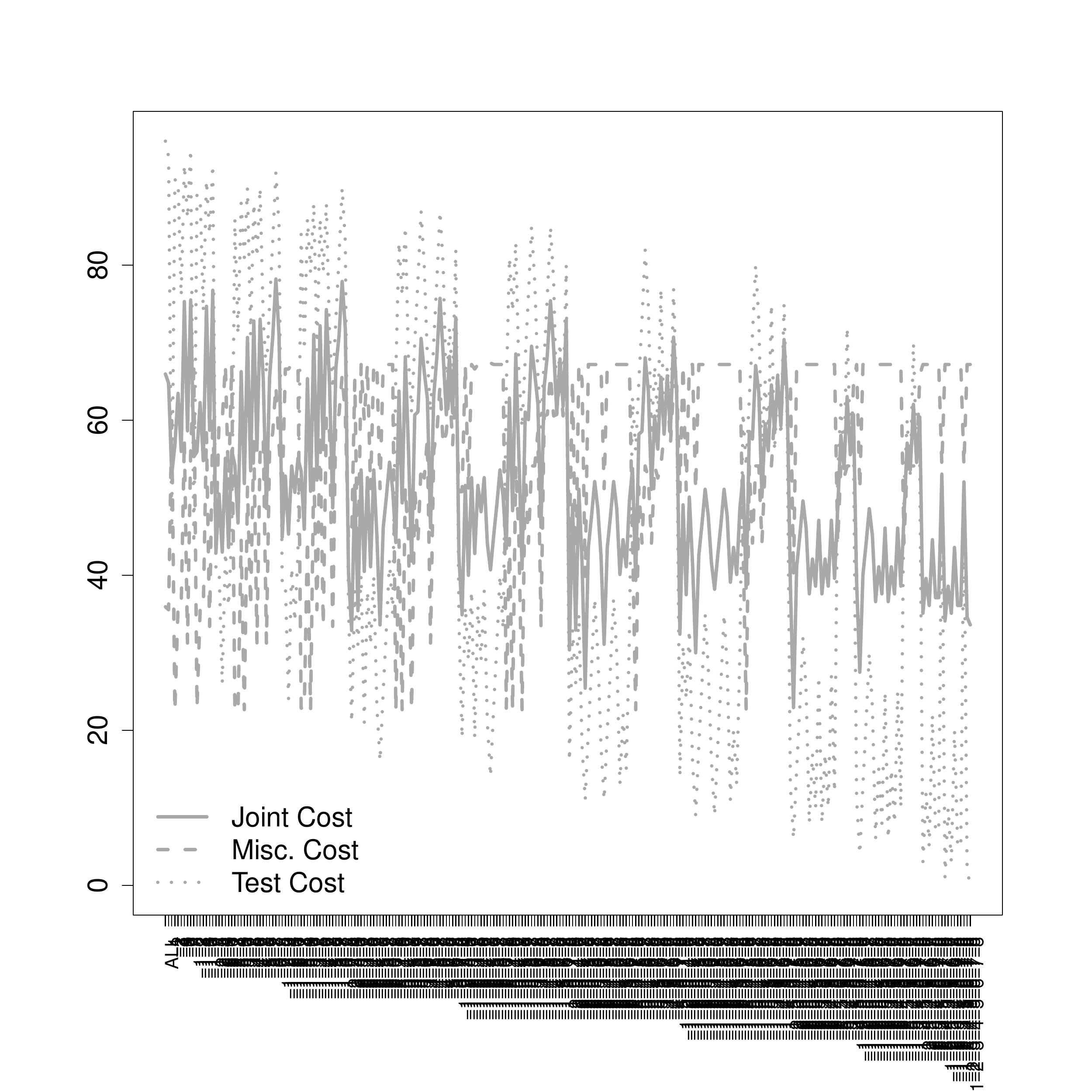} 
\caption {Evolution of $MC$, $TC$ and $JC$ according to attribute selection for a Adaboost model using non-uniform context. Left: Iris dataset, Right: Diabetes dataset}
\label{fig:allcosts0102ada}
\end{figure}

\begin{figure}
\centering
\includegraphics[width=0.48\textwidth]{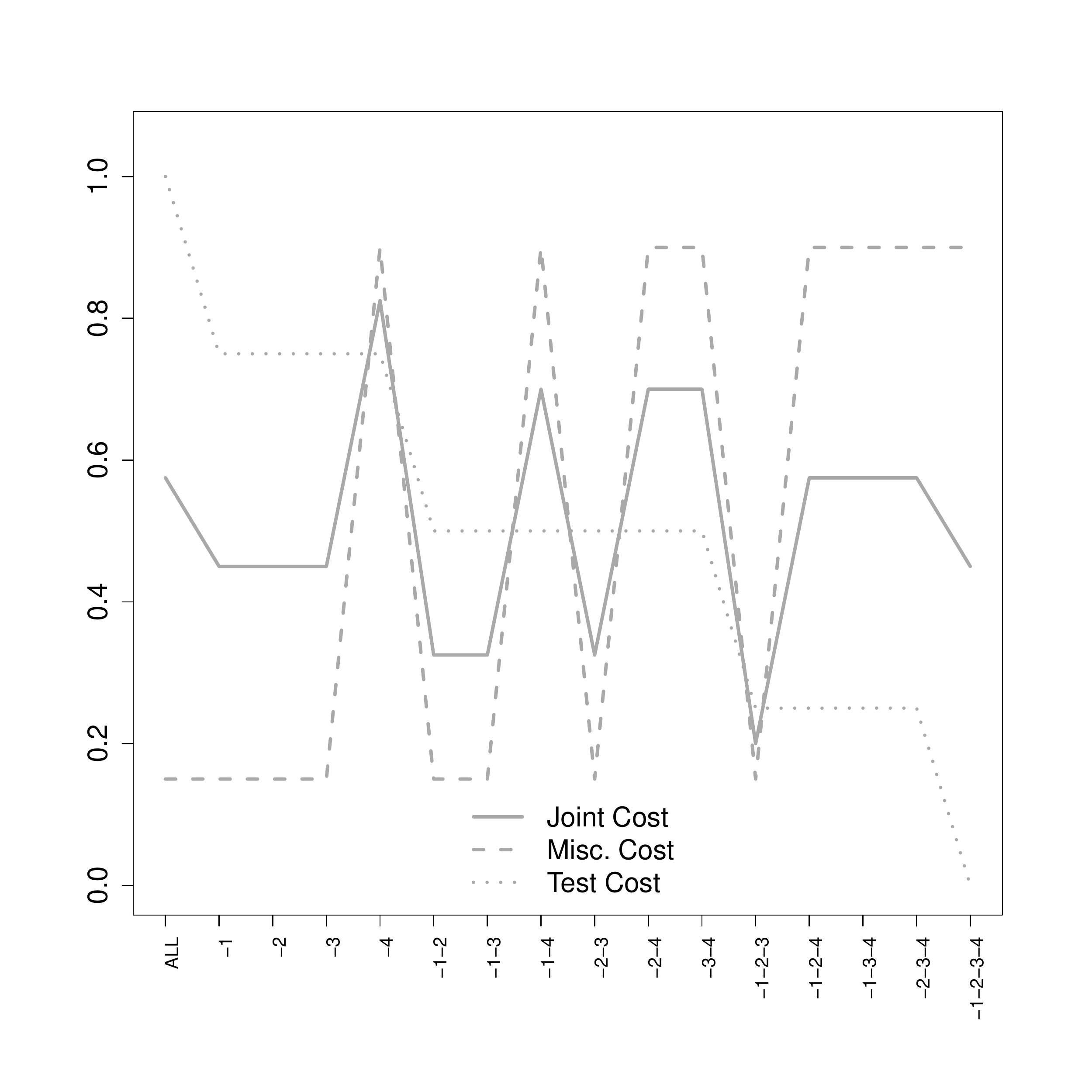} \hfill
\includegraphics[width=0.48\textwidth]{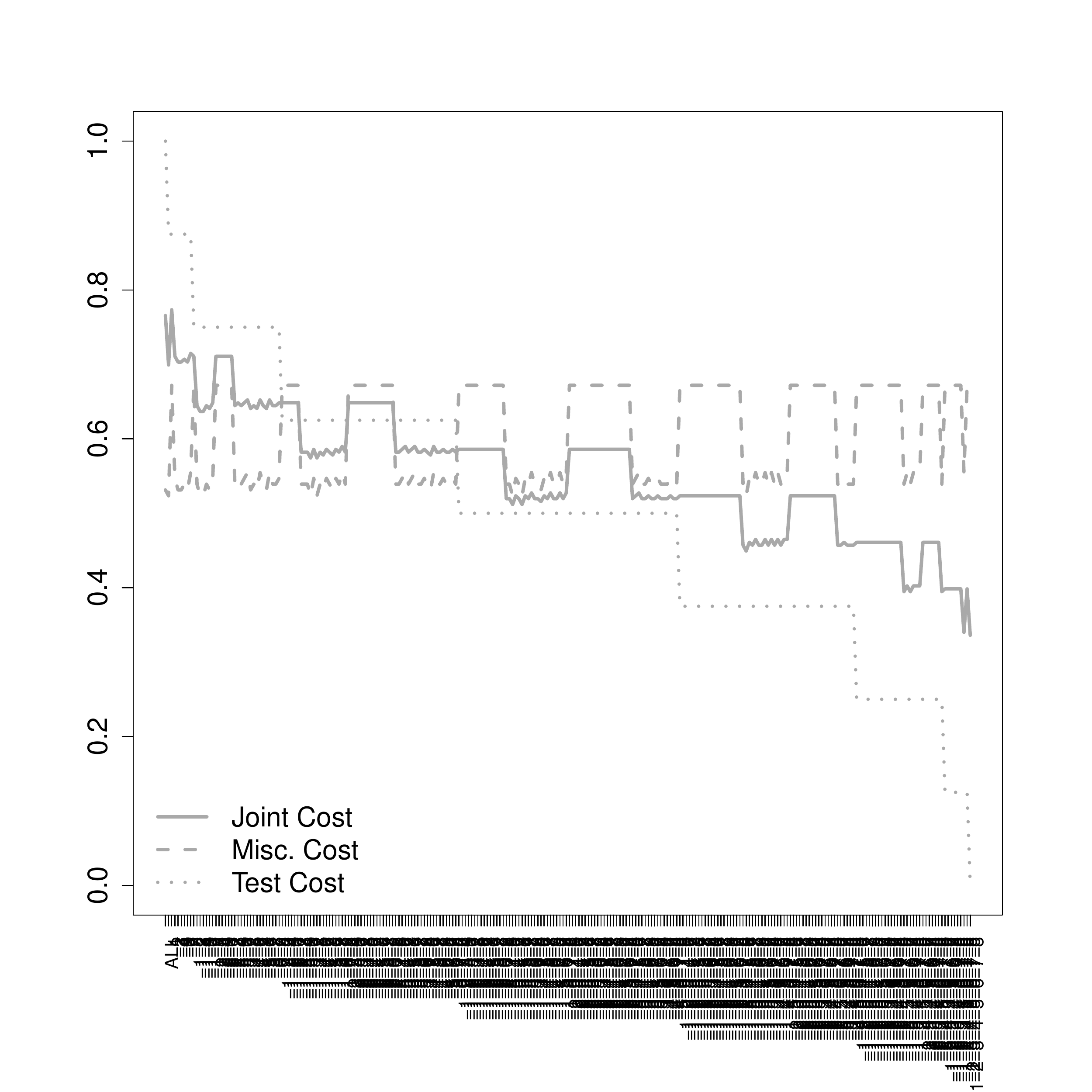} 
\caption {Evolution of $MC$, $TC$ and $JC$ according to attribute selection for a Bagging model using uniform context. Left: Iris dataset, Right: Diabetes dataset}
\label{fig:allcostsubag}
\end{figure}

\begin{figure}
\centering
\includegraphics[width=0.48\textwidth]{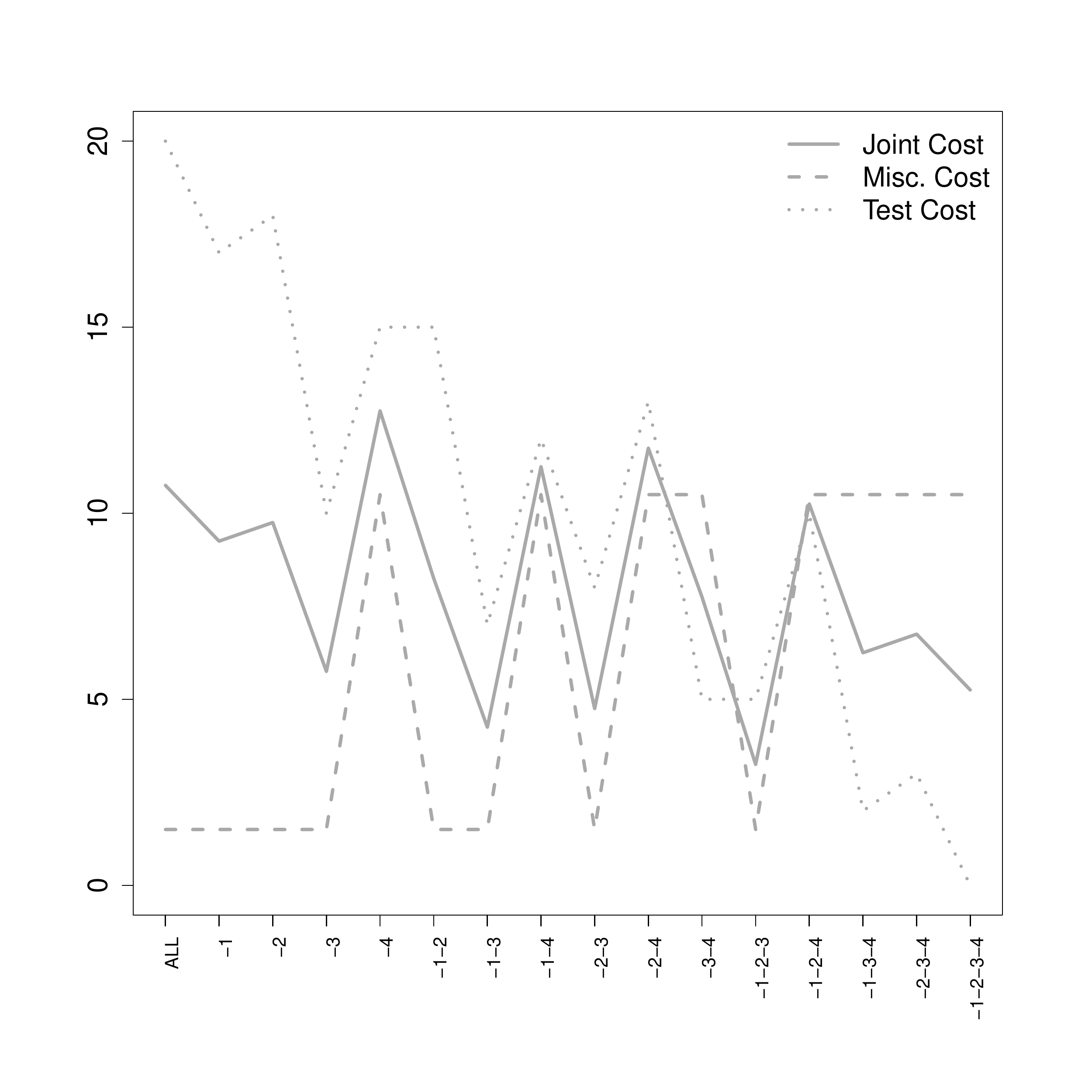} \hfill
\includegraphics[width=0.48\textwidth]{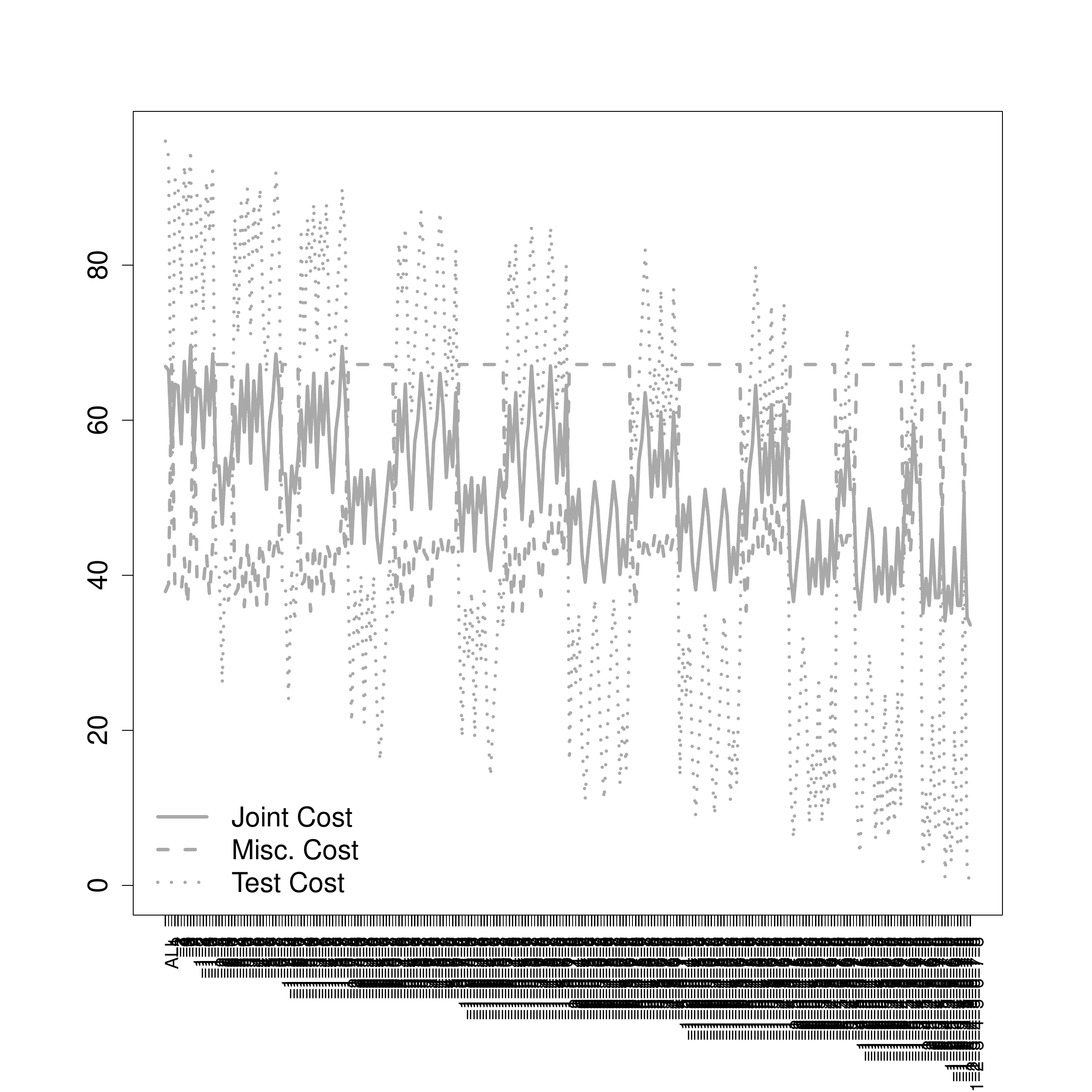} 
\caption {Evolution of $MC$, $TC$ and $JC$ according to attribute selection for a Bagging model using non-uniform context. Left: Iris dataset, Right: Diabetes dataset}
\label{fig:allcosts0102bag}
\end{figure}

All the figures \ref{fig:allcostsuibk} \ref{fig:allcostsuada} \ref{fig:allcostsubag} \ref{fig:allcosts0102ibk} \ref{fig:allcosts0102ada} and \ref{fig:allcosts0102bag} bring us to the same main conclusions: the best joint cost is not always obtained using all the features, or removing all of them. Moreover, the result vary with different machine learning methods when the dataset differs (for the iris dataset the best result is given by the Adaboost, but it is given  by the SMO (SVM) for the Diabetes dataset).

\chapter{Implementation details }\label{implementation}

For the implementation of the approach developed during this research, two principal sources files have been written in the R language, using the RWeka packages and executed in a R platform (see section \ref{tools}).

The first one is used for the representation of all the differents plots shown previously (and more). Given a dataset or a set of datasets and a set of machine learning models (J48, IBk, and SMO in this case), this program first splits each dataset into train and test parts, then trains the train data with each model; after that it proceeds to the reframing of the model and evaluates each reframed model. the result of the accuracy is then represented in a graph, as the trade-off of misclassification and test cost, the JROC convex hull graphs. The result of the approximation of the methods defined in section \ref{hull} can also be shown. 

\begin{lstlisting}

##############################################################################
#
# This is R code for calculating and drawing JROC curves (ROC curves for joint costs) and their convex hulls.
# Joint costs include misclassification costs and test costs
#
#
# This code has been developed by
#   Celestine-Periale Maguedong-Djoumessi, cemadj@posgrado.upv.es
#   Jose Hernandez-Orallo, jorallo@dsic.upv.es
#   UNIVERSITAT POLITECNICA DE VALENCIA, SPAIN
#
#
##############################################################################



########################################
# LIBRARIES
########################################

library(combinat)   # combn
library(RWeka)
\end{lstlisting}

\section{Function definition}
In this section all the functions used in the main code are defined.

\begin{lstlisting}

########################################
# FUNCTION DEFINITION
########################################

# This function sets to null a subset of the columns of a dataframe
settonull <- function(test, vector) {
  if (is.null(vector)) {
    test
  } else {
    for(j in 1:length(vector)) {
      test[,vector[j]] <- rep("?",length(test[,1]))
    }
    test
  }
}

# Calculates the misclassification cost from rweka's textual output.
# This function takes as input the result of an evaluation of a model on a dataset, and gives as result the field of misclassification cost.
misclascost <- function(e) {
	mylist <- strsplit(toString(e), "\n")
	totalcoststring <- mylist[[1]][8]
	totalcoststringnoblanks <- gsub(" ", "", totalcoststring)
	totalcoststringnoblanksonlynumber <- gsub("AverageCost", "", totalcoststringnoblanks)
	totalcostnumber <- as.numeric(totalcoststringnoblanksonlynumber)
	totalcostnumber
}


# This calculate the leftmost intercept line for a set of points given in arrays X and Y
leftmost_intercept <- function(slope, X, Y) {
  best_intercept <- Inf
  for (i in 1:length(X)) {
    x <- X[i]
    y <- Y[i]
    intercept <- y - slope*x    
    if (intercept < best_intercept)
      best_intercept <- intercept
  }
  best_intercept
}


# This calculates the leftmost intercept line for a set of points given in arrays X and Y. 
# But this one returns the points
leftmost_intercept_xy <- function(slope, X, Y) {
  best_intercept <- Inf
  for (i in 1:length(X)) {
    x <- X[i]
    y <- Y[i]
    intercept <- y - slope*x    
    if (intercept < best_intercept) {
      best_intercept <- intercept
      bestx <- x
      besty <- y
    }
  }
  c(bestx,besty)
}

\end{lstlisting}

\section{Main code}

Here, we define all the configuration needed before the experiment (definiton of the path where result will be saved, definition of the dataset(s) to be used and the operational context meaning the different costs.).

\begin{lstlisting}

########################################
# MAIN
########################################


# Sets seed, so all executions do the same and the plots are comparable
set.seed(2)

 WORKDIR <- "C:/Users/periale/Desktop/tesis/images"  # Celestine
#WORKDIR <- "A:/__FAENA__/_TESIS Co-Dirigides/Celestine Maguedong/code"  # Jose
setwd(WORKDIR)


# Dataset selection

#namedataset <- "breast-cancer"
 namedataset <- "diabetes"
path="" # variable wich contains the path  for the directory of the datasets
if (WORKDIR=="A:/__FAENA__/_TESIS Co-Dirigides/Celestine Maguedong/code") {
path=""
}
else {
path="C:/Users/periale/Desktop/datasets/UCI/"
}

if (namedataset != "iris") {
 mydata <- read.arff(paste(path,namedataset,".arff", sep=""))
 #mydata <- read.arff("diabetes.arff")
} else if (namedataset == "iris") {
  data(iris)
  mydata <- iris
}


# Gets length, number of attributes and number of classes
len <- length(mydata[,1])
nattr <- length(mydata[1,])
for (i in 1:len)
{
	for (j in 1:(nattr-1))
	{
		if (is.na(mydata[i,j]))
			mydata[i,j] <- "?"
	}
}

nclasses <- length(unique(mydata[,nattr]))


# Ensure that the name of the class attribute is called "class"
names(mydata)[nattr] <- "class"

# Shuffles datasets and splits it
shufindx <- sample(1:len, len)
mydata <- mydata[shufindx,]

lentrain <- trunc(len*2/3)
train <- mydata[1:lentrain,]
test <- mydata[(trunc(len*2/3)+1):len,]
lentest <- length(test[,1])


# Training

# Defines test cost vector and misc cost matrix
# There are several options depending on the experiment

# Uniform, as for definition 2
testcostvector_U <- rep(1/(nattr-1), nattr-1)   # nattr includes the class, so we need to remove 1
mcvector_U <- matrix(rep(nclasses/(nclasses-1), nclasses^2), nclasses, nclasses)
for (i in 1:nclasses)
  mcvector_U[i,i] <- 0
alpha <- 0.5  

testcostvector_01_iris <- c(3, 2, 10, 5)
mcvector_01_iris <- matrix(c(0,5,30,20,0,15,15,15,0),3,3)  # In Weka, rows are actual values and columns are predicted values

testcostvector_02_diabetes <- c(2, 50, 5, 5, 20, 3, 10, 1)
mcvector_02_diabetes <- matrix(c(0,200,50,0),2,2)  # In Weka, rows are actual values and columns are predicted values

cost_configuration <- "UR"
#cost_configuration <- "01"
# cost_configuration <- "02"

if (cost_configuration == "U") {
  testcostvector <- testcostvector_U
  mcvector <- mcvector_U
} else if (cost_configuration == "UR") {
  testcostvector <- testcostvector_U
  mcvector <- mcvector_U
  for (i in 1:length(testcostvector)) {
    k0 <- runif(1)
    k = exp(2*(k0-0.5)) 
    testcostvector[i] <- testcostvector[i]*k
  }
  for (i in 1:length(mcvector[1,])) {
    for (j in 1:length(mcvector[,1])) {
      k0 <- runif(1)
      k = exp(2*(k0-0.5)) 
      mcvector[i,j] <- mcvector[i,j]*k
    }
  }

} else if (cost_configuration == "01") {
  testcostvector <- testcostvector_01_iris
  mcvector <- mcvector_01_iris
} else if (cost_configuration == "02") {
  testcostvector <- testcostvector_02_diabetes
  mcvector <- mcvector_02_diabetes
}


print(testcostvector)
print(mcvector)

# SOME OPTION TO MAKE THE PLOTS LOOK OKAY

if (namedataset == "diabetes") {
  LEGEND_LOCATION <- "bottomleft"
  LEGEND_LOCATION0 <- LEGEND_LOCATION
} else { 
  LEGEND_LOCATION <- "topright"
  if (cost_configuration == "U") {
    LEGEND_LOCATION0 <- "bottom"
  } else {
    LEGEND_LOCATION0 <- LEGEND_LOCATION
  }
}


# ALPHAS THAT WE WILL USE
ALPHAS <- c(0.03, 0.5, 0.9)



# Training
# define all the models used for the training of the dataset.
models <- list()
models_names <- list()


# models <- c(models, list(J48(class ~ ., data = train)))
# models_names <- c(models_names, "DT")  # "Tree"

models <- c(models, list(SMO(class ~ ., data = train)))
models_names <- c(models_names, "SMO")  # "SVM"

models <- c(models, list(IBk(class ~ ., data = train)))
models_names <- c(models_names, "IBk")  # "kNN"

# models <- c(models, list(AdaBoostM1(class ~ ., data = train, control = Weka_control(W = "DecisionStump"))))
# models_names <- c(models_names, "BstDS")  

models <- c(models, list(AdaBoostM1(class ~ ., data = train, control = Weka_control(W = list(J48, M = 30)))))
models_names <- c(models_names, "BstDT") 

# models <- c(models, list(Bagging(class ~ ., data = train, control = Weka_control(W = "DecisionStump")))) 
# models_names <- c(models_names, "BagDS")  

models <- c(models, list(Bagging(class ~ ., data = train, control = Weka_control(W = list(J48, M = 30)))))
models_names <- c(models_names, "BagDT")  

lenmod <- length(models)


# Initialisation of lists and structures for main loop
mc1=0
mis1=0



# These are the lists where we keep the values for TC, MC, JC and accuracy
testcost <- rep(list(),lenmod)
misccost <- rep(list(),lenmod)
jointcost <- rep(list(),lenmod)
accuracylist <- list()


label <- "-"
labelnames <- rep(list(),lenmod)


modifiedtestvector <- list()

# objects use for the construction of box plot instances for the 3differents methods
boxplotstrg1 <- NULL
bp1 <- rep(list(), lenmod)

bpnames <- NULL

latticecounter <- 0


# objects for the incremental methods

tcpivotset <- NULL      # Attributes removed for tcpivotset (originally none)
tccounter <- rep(0,lenmod)   # corrected
tcselection <- list()

mcpivotset1 <- NULL      # Attributes removed for mcpivotset (originally none)
mccounter1 <- rep(0,lenmod)  # corrected
mcselection <- list()

mcpivotsetnew<-list()
tcpivotsetnew<-list()
jcpivotsetnew<-list()


jcpivotset1 <- NULL      # Attributes removed for jcpivotset (originally none)
jccounter1 <- rep(0,lenmod)  # corrected
jcselection <- list()

# The outer loop is for each row in the lattice (removing i attributes)
for (i in 0:(nattr-1)) {   # it goes from 0 (no attributes removed) to nattr-1 (all attributes removed)
                           # (Note that nattr includes the class, so nattr-1 is the actual number of attributes


  cat("\n****Outer loop. Iteration: ", i, "of ", nattr-1, "****\n\n")

  # What's the meaning of these pivots?
  tcpivot <- Inf
  mcpivot1 <- Inf
  jcpivot1 <- Inf

  pivot <- 1
  bpnames[i+1] <- i
  boxplotstrg1 <- NULL
  label<-"_"

  # This calculates the combinations removing i attributes
  settonullmatrix <- t(combn(1:(nattr-1), i))
  print(settonullmatrix)
  if (i == 0) {     
    len <- 0
  } else {
    len <- length(settonullmatrix[,1])
  }


  # The inner loop is for each element in the row in the lattice
  for (j in (min(1,i):len)) {

    latticecounter <- latticecounter + 1
    cat("\n  ****Inner loop. Iteration: ", j, "of ", len, "   Points in the lattice: ", latticecounter, "of", 2^(nattr-1), "****\n\n")

    tcignore <- FALSE
    mcignore1 <- FALSE
    jcignore1 <- FALSE
	
k=1
for (mod in models)
{
  
  if (j==0) {
      modifiedtest <- test
      label <- "ALL"
      myvector <- NULL   
    } else {
      myvector <- settonullmatrix[j,]
      modifiedtest <- settonull(test, myvector)  

      newattr <- setdiff(myvector, tcpivotset)     # New attributes from the one fixed for tc in the previous lattice row
      if (length(newattr) != 1) {
        tcignore <- TRUE 
      }

      newattr <- setdiff(myvector, mcpivotset1)     # New attributes from the one fixed for mc in the previous lattice row
      if (length(newattr) != 1) {
        mcignore1 <- TRUE 
      }


      newattr <- setdiff(myvector, jcpivotset1)     # New attributes from the one fixed for jc in the previous lattice row
      if (length(newattr) != 1) {
        jcignore1 <- TRUE 
      }

    }

    if (!tcignore)
      tccounter[k] <- tccounter[k] + 1

    if (!mcignore1)
      mccounter1[k] <- mccounter1[k] + 1
   
    if (!jcignore1)
      jccounter1[k] <- jccounter1[k] + 1
    
    # do whatever you need to do with modifiedtest
	
    e <- evaluate_Weka_classifier(mod,newdata=modifiedtest, cost= mcvector, complexity = TRUE,seed = 123, class = TRUE )
    
    # computation of the accuracy with the DT
    # for the first model
    res1 <- predict(mod, newdata = modifiedtest)
#    print("result with DT")
#    print(e1)
#    summary(e1)
    hits <- 0
    for (i in 1:lentest) 
      if (res1[i] == modifiedtest[i,nattr])
        hits <- hits +1
    
    accuracy1 <- hits/ lentest
#    print(accuracy1)
print(k)
	if (length(accuracylist)<lenmod)
		accuracylist <- c(accuracylist, list(accuracy1))
		else
		accuracylist[[k]] <- c(accuracylist[[k]], list(accuracy1))
	mc1 <- misclascost(e[1])
	if (length(misccost)<lenmod)
		misccost <- c(misccost, list(mc1))
		else
		misccost[[k]] <- c(misccost[[k]], list(mc1))
	
	

  
	

    # computation of the total test cost (and creation of labels)
    tc <- 0
    if (j==0) {
      label <- "ALL"
    } else {
      label <- ""
    }
    for (i in 1:(nattr-1)) {
     if (modifiedtest[1,i] != "?") {             
	tc <- tc + testcostvector[i]
     }
    }
	if (length(testcost)<lenmod)
		testcost <- c(testcost, list(tc))
		else
		testcost[[k]] <- c(testcost[[k]],list(tc))
	
    jc1 <- alpha*mc1 + (1-alpha)*tc
  
	if (length(jointcost)<lenmod)
		jointcost <- c(jointcost,list(jc1))
		else
		jointcost[[k]] <- c(jointcost[[k]], list(jc1))



    if (!tcignore) {	# tcignore == TRUE means that this point cannot be derived from the previous row in the lattice
		if (length(tcselection)<lenmod)
		tcselection <- c(tcselection, latticecounter)
		else
		tcselection[[k]] <- c(tcselection[[k]], latticecounter)  # Index of the point

      if (tcpivot > tc) { 
        tcpivot <- tc
        tcpivotsetnew <- myvector   # Keeps track of the attributes of the best point for jc
      }
    }

    if (!mcignore1) {	 # mcignore1 == TRUE means that this point cannot be derived from the previous row in the lattice
     if (length(mcselection)<lenmod)
		mcselection <- c(mcselection, latticecounter)
		else
		mcselection[[k]] <- c(mcselection[[k]], latticecounter)  # Index of the point

      if (mcpivot1 > mc1) { 
        mcpivot1 <- mc1
        mcpivotsetnew <- myvector   # Keeps track of the attributes of the best point for jc
      }
    }

 


    if (!jcignore1) {	 # jcignore1 == TRUE means that this point cannot be derived from the previous row in the lattice
      if (length(jcselection)<lenmod)
		jcselection <- c(jcselection, latticecounter)
		else
		jcselection[[k]] <- c(jcselection[[k]], latticecounter)  # Index of the point

      if (jcpivot1 > jc1) { 
        jcpivot1 <- jc1
        jcpivotsetnew <- myvector   # Keeps track of the attributes of the best point for jc
      }
    }
    

	
    boxplotstrg1[pivot] <- jc1
    
    pivot <- pivot+1  


  	cat("model:", unlist(models_names[k]),"\n")
  tcpivotset <- tcpivotsetnew
  cat("\nBest TC: points explored: ", tccounter[k], "\n")
  cat("Best TC: attributes removed: ", tcpivotset, "\n")


  mcpivotset1 <- mcpivotsetnew
  cat("\nBest MC: points explored: ", mccounter1[k], "\n") # Remains to be adjust with the others one
  cat("Best MC: attributes removed: ", mcpivotset1, "\n")


  jcpivotset1 <- jcpivotsetnew
  cat("\nBest JC: points explored: ", jccounter1[k], "\n")
  cat("Best JC: attributes removed: ", jcpivotset1, "\n")
  k=k+1
  } # end of inner loop

    
	for (i in 1:(nattr-1)) {
     if (modifiedtest[1,i] == "?") {             # computation of the labelnames of the current modified test file.
       label <- paste(label,"-", i, sep="")
     }
    }
	labelnames <- c(labelnames, list(label))
	
	modifiedtestvector <- c(modifiedtestvector, list(modifiedtest))  # computation of the vector which contains all the differents modifiedtests
	


  cat("\n\n")
  

  
	bp1 <-c(bp1,list(boxplotstrg1))
	

}
  

} # end of outer loop


# The labels are converted into a vector x
x=NULL
for (i in 1:length(labelnames))
	x[i] <- labelnames[[i]]
	


# We calculate some max and min for the plots

maxmisccost <- max(unlist(misccost))


minmisccost <- min(unlist(misccost))


maxjointcost <- max(unlist(jointcost))


minjointcost <- min(unlist(jointcost))


maxtestcost <- max(unlist(testcost))
mintestcost <- min(unlist(testcost))

maxcost <- max(maxmisccost, maxjointcost, maxtestcost)
mincost <- min(minmisccost, minjointcost, mintestcost)

maxaccuracy <- max(unlist(accuracylist))


minaccuracy <- min(unlist(accuracylist))

\end{lstlisting}

\section{BMC, BTC, BJC, RND methods implementation:}

This part of the code implements the four others approach (adding to the Full method) which approximate the \JROC hull (see chapter \ref{hull}): the Backward $MC$-guided (BMC), the Backward $TC$-guided (BTC), the Backward $JC$-guided (BJC) and Monte Carlo (RND).

\begin{lstlisting}

misccosttc <- list()
misccostmc <- list()
misccostjc <- list()
testcosttc <- list()
testcostmc <- list()
testcostjc <- list()
misccostrnd <- list()
testcostrnd <- list()

##################### TC method incremental ##########################
for (i in 1:lenmod)
{
misccosttc <- c(misccosttc, list(misccost[[i]][tcselection[[i]]]))
testcosttc <- c(testcosttc,list(testcost[[i]][tcselection[[i]]]))
}





##################### MC method incremental ##########################

for (i in 1:lenmod)
{
misccostmc <- c(misccostmc, list(misccost[[i]][mcselection[[i]]]))
testcostmc <- c(testcostmc,list(testcost[[i]][mcselection[[i]]]))
}




##################### JC method incremental ##########################

for (i in 1:lenmod)
{
misccostjc <- c(misccostjc, list(misccost[[i]][jcselection[[i]]]))
testcostjc <- c(testcostjc,list(testcost[[i]][jcselection[[i]]]))
}




##################### Monte Carlo method ##########################

latticesize <- length(modifiedtestvector)
rndsamplesize <- (nattr-1)*(nattr)/2 + 1           

rndselection <- sample(1:latticesize, rndsamplesize, replace=FALSE)


# lists for the monte carlo method. All the lists end with rnd (rnd)
for (i in 1:lenmod)
{
misccostrnd <- c(misccostrnd, list(misccost[[i]][rndselection]))
testcostrnd <- c(testcostrnd, list(testcost[[i]][rndselection]))
}


labelnamesrnd <- labelnames[rndselection]


  
# The monte carlo labels are converted into a vector Y
Y=NULL
for (i in 1:length(labelnamesrnd))
	Y[i] <- labelnamesrnd[[i]]
  


point_full<-NULL
point_btc<-NULL
point_bmc<-NULL
point_bjc<-NULL
point_rnd<-NULL

alpha_iso<-0.5
  slope <- (1 - alpha_iso) / - alpha_iso
  
# FULL METHOD
  point_full <- leftmost_intercept_xy(slope, c(unlist(testcost)), c(unlist(misccost)) )
  jcfull <- alpha_iso*point_full[2] + (1-alpha_iso)*point_full[1]

# BMC METHOD
  point_bmc <- leftmost_intercept_xy(slope, c(unlist(testcostmc)), c(unlist(misccostmc) ))
  jcbmc <- alpha_iso*point_bmc[2] + (1-alpha_iso)*point_bmc[1]

# BTC METHOD
  point_btc <- leftmost_intercept_xy(slope, c(unlist(testcosttc)), c(unlist(misccosttc) ))
  jcbtc <- alpha_iso*point_btc[2] + (1-alpha_iso)*point_btc[1]

# BTC METHOD
  point_bjc <- leftmost_intercept_xy(slope, c(unlist(testcostjc)), c(unlist(misccostjc) ))
  jcbjc <- alpha_iso*point_bjc[2] + (1-alpha_iso)*point_bjc[1]
  
# BTC METHOD
  point_rnd <- leftmost_intercept_xy(slope, c(unlist(testcostrnd)), c(unlist(misccostrnd) ))
  jcrnd <- alpha_iso*point_rnd[2] + (1-alpha_iso)*point_rnd[1]

  c(jcfull, jcbmc, jcbtc, jcbjc, jcrnd)  # Returns the five joint costs for the five different methods


\end{lstlisting}

\section{Plots representation}
This last part of the code is used for the representation of the results obtained during the experiments in various graphs, which are saved as pdf files.

\begin{lstlisting}

 
################################################
####          PDF FILES                    #####
################################################
 

# PDF options
PDFOPEN <- TRUE           # If the plots are output on a PDF file
PDFCLOSE <- PDFOPEN       # Close the PDF file. This should match PDFOPEN  except when you want to draw several curves before closing.
PDFheight= 10             # 7 is the default, so 14 makes it double higher than wide, 5 makes letters bigger (in proportion) for just one plot
PDFwidth= 10                # (as above for width) 7 by default

# Colours
plot_colors <- c("green", "red", "blue", "orange", "yellow", "violet", "pink")
INTERCEPT_COLOUR <- "darkgrey"






# Figure 1 : Accuracy vs attributes for all methods
pdfname <- paste("accuracyall", namedataset, ".pdf", sep="") # File name

if (PDFOPEN) {
  pdf(pdfname, height= PDFheight, width= PDFwidth)
}
par(mar=c(8,6,5,5) + 0.1)

# A=NULL                                  # corrected
# for (i in 1:length(accuracylist3))      # corrected
#	A[i] <- accuracylist3[[i]]        # corrected 

#for draw the plot of accuracy
for (i in 1:lenmod) {
  if (i==1) {
    plot(1:length(accuracylist[[1]]), accuracylist[[1]],ylab="Accuracy", xlab="",ylim=c(minaccuracy,maxaccuracy),lwd=3, xaxt="n",cex.lab=2, type="l", col=plot_colors[1])
    axis(1, at=1:length(x),srt=45, padj=1,las=2, lab=x)
  } else {
    lines(1:length(accuracylist[[i]]), accuracylist[[i]], type="l",lty=i,lwd=3, col=plot_colors[i])
  }
}
legend(LEGEND_LOCATION,c(unlist(models_names)), cex=1, col=plot_colors, lty=1:lenmod, lwd=3, bty="n")  
  
if (PDFCLOSE) {
  dev.off()
} 




for(i in 1:lenmod)  #for the graphs representing different cost for all methods 
{

# Figure 2 (or 3) : All costs vs attributes for one method
pdfname <- paste("costs",models_names[i], cost_configuration, namedataset, ".pdf", sep="") # File name

if (PDFOPEN) {
  pdf(pdfname, height= PDFheight, width= PDFwidth)
}
par(mar=c(8,6,5,4) + 0.1)

	
 #for the graphs representing different cost for one method 
plot(1:length(misccost[[i]]), misccost[[i]],ylab="", xlab="",ylim=c(mincost,maxcost),lwd=3, xaxt="n",lty=2, type="l",cex.lab=2,cex.axis=1.5, col="dark grey")
axis(1, at=1:length(x),srt=45, padj=1,las=2, lab=x)
lines(1:length(jointcost[[i]]),jointcost[[i]], type="l", lty=1, lwd=3, col="dark grey")
lines(1:length(testcost[[i]]), testcost[[i]], type="l", lty=3, lwd=3, col="dark grey")
legend(LEGEND_LOCATION0, c("Joint Cost","Misc. Cost","Test Cost"), cex=1.5, col="dark grey",    lty=1:3, lwd=3, bty="n")   
  
if (PDFCLOSE) {
  dev.off()
} 

}




#Figure 4 (or 5)
pdfname <- paste("totalplot", cost_configuration, namedataset, ".pdf", sep="") # File name

if (PDFOPEN) {
  pdf(pdfname, height= PDFheight, width= PDFwidth)
}
par(mar=c(5,6,5,2) + 0.1)

for (i in 1:lenmod) {
  if (i==1) {
    plot(testcost[[1]], misccost[[1]],ylab="MC", xlab="TC",ylim=c(minmisccost,maxmisccost),xlim=c(mintestcost,maxtestcost),pch=1,lwd=2,cex=2,cex.lab=2,cex.axis=1.5, col=plot_colors[1])
  } else {
    points(testcost[[i]],misccost[[i]], cex=2, pch=i,lwd=2, col=plot_colors[i])
  }
}

legend(LEGEND_LOCATION, c(unlist(models_names)), cex=1.5, col=plot_colors,    pch=1:lenmod, lwd=3, bty="n") 
  
if (PDFCLOSE) {
  dev.off()
} 





# figure 6
pdfname <- paste("totalplot", cost_configuration,  namedataset, "-iso", ".pdf", sep="") # File name

if (PDFOPEN) {
  pdf(pdfname, height= PDFheight, width= PDFwidth)
}
par(mar=c(5,6,5,2) + 0.1)

for (i in 1:lenmod) {
  if (i==1) {
    plot(testcost[[1]], misccost[[1]],ylab="MC", xlab="TC",ylim=c(minmisccost,maxmisccost),xlim=c(mintestcost,maxtestcost),pch=1,lwd=2,cex=2,cex.lab=2,cex.axis=1.5, col=plot_colors[1])
  } else {
    points(testcost[[i]],misccost[[i]], cex=2, pch=i,lwd=2, col=plot_colors[i])
  }
}
 
legend(LEGEND_LOCATION, c(unlist(models_names)), cex=1.5, col=plot_colors, pch=1:lenmod, lwd=3, bty="n") 


for (i in 1:length(ALPHAS)) {
  alpha1 <- ALPHAS[i]
  slope <- (1 - alpha1) / - alpha1
  intercept <- leftmost_intercept(slope, c(unlist(testcost)), c(unlist(misccost)) )
  abline(a=intercept,b=slope, col = INTERCEPT_COLOUR, lty=4)
}
   

  
if (PDFCLOSE) {
  dev.off()
} 






# figure 7
pdfname <- paste("totalplot", cost_configuration,  namedataset, "-ch", ".pdf", sep="") # File name

if (PDFOPEN) {
  pdf(pdfname, height= PDFheight, width= PDFwidth)
}
par(mar=c(5,6,5,2) + 0.1)

for (i in 1:lenmod) {
  if (i==1) {
    plot(testcost[[1]], misccost[[1]],ylab="MC", xlab="TC",ylim=c(minmisccost,maxmisccost),xlim=c(mintestcost,maxtestcost),pch=1,lwd=2,cex=2,cex.lab=2,cex.axis=1.5, col=plot_colors[1])
  } else {
    points(testcost[[i]],misccost[[i]], cex=2, pch=i,lwd=2, col=plot_colors[i])
  }
}

legend(LEGEND_LOCATION, c(unlist(models_names)), cex=1.5, col=plot_colors,    pch=1:lenmod, lwd=3, lty=1:3, bty="n") 

BIGX <- maxtestcost*10000
BIGY <- maxmisccost*10000

for (i in 1:lenmod) {
  X <- c(unlist(testcost[[i]]), 0, BIGX, BIGX)
  Y <- c(unlist(misccost[[i]]), BIGY, BIGY, 0)
  ch <- chull(X,Y)
  testcostch <- X[ch]
  misccostch <- Y[ch]
  lines(testcostch,misccostch, cex=2, lty=i, pch=i,lwd=2, col=plot_colors[i])
}
  
if (PDFCLOSE) {
  dev.off()
} 











# figure 8
pdfname <- paste("mcincremental", cost_configuration, namedataset, ".pdf", sep="") # File name

if (PDFOPEN) {
  pdf(pdfname, height= PDFheight, width= PDFwidth)
}
par(mar=c(5,6,5,2) + 0.1)
for (i in 1:lenmod) {
if(i==1) {
plot(testcostmc[[1]], misccostmc[[1]],ylab="MC", xlab="TC",ylim=c(minmisccost,maxmisccost),xlim=c(mintestcost,maxtestcost),pch=1,lwd=2,cex=2,cex.lab=2,cex.axis=1.5, col=plot_colors[1])
}
else {
points(testcostmc[[i]], misccostmc[[i]], cex=2, pch=2,lwd=2, col=plot_colors[i])
}
}
legend(LEGEND_LOCATION, c(unlist(models_names)), cex=1.5, col=plot_colors, pch=1:3, lwd=3, bty="n")  
   

BIGX <- maxtestcost*10000
BIGY <- maxmisccost*10000

for (i in 1:lenmod) {
  X <- c(unlist(testcostmc[[i]]), 0, BIGX, BIGX)
  Y <- c(unlist(misccostmc[[i]]), BIGY, BIGY, 0)
  ch <- chull(X,Y)
  testcostch <- X[ch]
  misccostch <- Y[ch]
  lines(testcostch,misccostch, cex=2, lty=i, pch=i,lwd=2, col=plot_colors[i])
}

  
if (PDFCLOSE) {
  dev.off()
} 




# figure 9
pdfname <- paste("jcincremental", cost_configuration, namedataset, ".pdf", sep="") # File name

if (PDFOPEN) {
  pdf(pdfname, height= PDFheight, width= PDFwidth)
}

par(mar=c(5,6,5,2) + 0.1)
for (i in 1:lenmod) {
if(i==1) {
plot(testcostjc[[1]], misccostjc[[1]],ylab="MC", xlab="TC",ylim=c(minmisccost,maxmisccost),xlim=c(mintestcost,maxtestcost),pch=1,lwd=2,cex=2,cex.lab=2,cex.axis=1.5, col=plot_colors[1])
}
else {
points(testcostjc[[i]], misccostjc[[i]], cex=2, pch=2,lwd=2, col=plot_colors[i])
}
}
legend(LEGEND_LOCATION, c(unlist(models_names)), cex=1.5, col=plot_colors,    pch=1:3, lwd=3, bty="n") 
   

BIGX <- maxtestcost*10000
BIGY <- maxmisccost*10000

for (i in 1:lenmod) {
  X <- c(unlist(testcostjc[[i]]), 0, BIGX, BIGX)
  Y <- c(unlist(misccostjc[[i]]), BIGY, BIGY, 0)
  ch <- chull(X,Y)
  testcostch <- X[ch]
  misccostch <- Y[ch]
  lines(testcostch,misccostch, cex=2, lty=i, pch=i,lwd=2, col=plot_colors[i])
}

  
if (PDFCLOSE) {
  dev.off()
} 





#figure 10
pdfname <- paste("tcincremental", cost_configuration, namedataset, ".pdf", sep="") # File name

if (PDFOPEN) {
  pdf(pdfname, height= PDFheight, width= PDFwidth)
}
par(mar=c(5,6,5,2) + 0.1)

for (i in 1:lenmod) {
if(i==1) {
plot(testcosttc[[1]], misccosttc[[1]],ylab="MC", xlab="TC",ylim=c(minmisccost,maxmisccost),xlim=c(mintestcost,maxtestcost),pch=1,lwd=2,cex=2,cex.lab=2,cex.axis=1.5, col=plot_colors[1])
}
else {
points(testcosttc[[i]], misccosttc[[i]], cex=2, pch=2,lwd=2, col=plot_colors[i])
}
}
legend(LEGEND_LOCATION, c(unlist(models_names)), cex=1.5, col=plot_colors,   pch=1:3, lwd=3, bty="n")


BIGX <- maxtestcost*10000
BIGY <- maxmisccost*10000

for (i in 1:lenmod) {
  X <- c(unlist(testcosttc[[i]]), 0, BIGX, BIGX)
  Y <- c(unlist(misccosttc[[i]]), BIGY, BIGY, 0)
  ch <- chull(X,Y)
  testcostch <- X[ch]
  misccostch <- Y[ch]
  lines(testcostch,misccostch, cex=2, lty=i, pch=i,lwd=2, col=plot_colors[i])
}

  
if (PDFCLOSE) {
  dev.off()
} 






# figure 11
pdfname <- paste("rnd", cost_configuration, namedataset, ".pdf", sep="") # File name

if (PDFOPEN) {
  pdf(pdfname, height= PDFheight, width= PDFwidth)
}
par(mar=c(5,6,5,2) + 0.1)
plot(testcostrnd[[1]], misccostrnd[[1]],ylab="MC", xlab="TC",ylim=c(minmisccost,maxmisccost),xlim=c(mintestcost,maxtestcost),pch=1,lwd=2,cex=2,cex.lab=2,cex.axis=1.5, col=plot_colors[1])
points(testcostrnd[[1]],misccostrnd[[2]], cex=2, pch=2,lwd=2, col=plot_colors[2])
points(testcostrnd[[1]],misccostrnd[[3]], cex=2, pch=3,lwd=2, col=plot_colors[3])
legend(LEGEND_LOCATION, c(unlist(models_names)), cex=1.5, col=plot_colors, pch=1:3, lwd=3, bty="n")  
   

BIGX <- maxtestcost*10000
BIGY <- maxmisccost*10000

for (i in 1:lenmod) {
  X <- c(unlist(testcostrnd[[i]]), 0, BIGX, BIGX)
  Y <- c(unlist(misccostrnd[[i]]), BIGY, BIGY, 0)
  ch <- chull(X,Y)
  testcostch <- X[ch]
  misccostch <- Y[ch]
  lines(testcostch,misccostch, cex=2, lty=i, pch=i,lwd=2, col=plot_colors[i])
}

   
if (PDFCLOSE) {
  dev.off()
} 








BOXPLOTS <- FALSE


if (BOXPLOTS) {





# figure boxplot
pdfname <- paste("boxplotDT", cost_configuration, namedataset, ".pdf", sep="") # File name

if (PDFOPEN) {
  pdf(pdfname, height= PDFheight, width= PDFwidth)
}
par(mar=c(5,6,5,2) + 0.1)

boxplot(bp1,names=bpnames, ylab="JC", xlab="Features Removed") 
   

  
if (PDFCLOSE) {
  dev.off()
} 


# figure boxplot
pdfname <- paste("boxplotSMO", cost_configuration, namedataset, ".pdf", sep="") # File name

if (PDFOPEN) {
  pdf(pdfname, height= PDFheight, width= PDFwidth)
}
par(mar=c(5,6,5,2) + 0.1)

boxplot(bp2,names=bpnames, ylab="JC", xlab="Features Removed") 
   

  
if (PDFCLOSE) {
  dev.off()
}


# figure boxplot
pdfname <- paste("boxplotIBk", cost_configuration, namedataset, ".pdf", sep="") # File name

if (PDFOPEN) {
  pdf(pdfname, height= PDFheight, width= PDFwidth)
}
par(mar=c(5,6,5,2) + 0.1)

boxplot(bp3,names=bpnames, ylab="JC", xlab="Features Removed") 
   
  
if (PDFCLOSE) {
  dev.off()
}



} # end if

\end{lstlisting}

\section{Experiment implementation}

The second source implements the macro-function ``ONE-EXPERIMENT'', which performs the experiments described in section \ref{experiments}.
In fact, this function, giving a value for $\alpha $ (the isometric value used to give a certain weight to gauge the relevance to the misclassification and test cost), a dataset and a list of different machine learning methods (predictive model algoritms), computes the best features configuration, i.e., the configuration of features which gives the best Joint Cost for each of the 5 methods (Full, BMC, BTC, BJC, RND).
In the implementation of this macro-function, there are some other functions previously defined which are also used, thus repeated here.

\begin{lstlisting}

##############################################################################
#
# This is R code for computing the ONE-EXPERIMENT macro-function.
# Joint costs include misclassification costs and test costs
#
#
# This code has been developed by
#   Celestine-Periale Maguedong-Djoumessi, cemadj@posgrado.upv.es
#   Jose Hernandez-Orallo, jorallo@dsic.upv.es
#   UNIVERSITAT POLITECNICA DE VALENCIA, SPAIN
#
#
##############################################################################



########################################
# LIBRARIES
########################################

library(combinat)   # combn
library(RWeka)




########################################
# FUNCTION DEFINITION
########################################

# This function sets to null a subset of the columns of a dataframe
settonull <- function(test, vector) {
  if (is.null(vector)) {
    test
  } else {
    for(j in 1:length(vector)) {
      test[,vector[j]] <- rep("?",length(test[,1]))
    }
    test
  }
}

# Calculates the misclassification cost from rweka's textual output
misclascost <- function(e) {
	mylist <- strsplit(toString(e), "\n")
	totalcoststring <- mylist[[1]][8]
	totalcoststringnoblanks <- gsub(" ", "", totalcoststring)
	totalcoststringnoblanksonlynumber <- gsub("AverageCost", "", totalcoststringnoblanks)
	totalcostnumber <- as.numeric(totalcoststringnoblanksonlynumber)
	totalcostnumber
}


# This calculate the leftmost intercept line for a set of points given in arrays X and Y
leftmost_intercept <- function(slope, X, Y) {
  best_intercept <- Inf
  for (i in 1:length(X)) {
    x <- X[i]
    y <- Y[i]
    intercept <- y - slope*x    
    if (intercept < best_intercept)
      best_intercept <- intercept
  }
  best_intercept
}


# This calculates the leftmost intercept line for a set of points given in arrays X and Y. 
# But this one returns the points
leftmost_intercept_xy <- function(slope, X, Y) {
  best_intercept <- Inf
  for (i in 1:length(X)) {
    x <- X[i]
    y <- Y[i]
    intercept <- y - slope*x    
    if (intercept < best_intercept) {
      best_intercept <- intercept
      bestx <- x
      besty <- y
    }
  }
  c(bestx,besty)
}





########################################
# MAIN
########################################


# Sets seed, so all executions do the same and the plots are comparable
set.seed(2)

# WORKDIR <- "C:/Users/periale/Desktop/tesis/images"  # Celestine
#WORKDIR <- "E:/__FAENA__/_TESIS Co-Dirigides/Celestine Maguedong/code"  # Jose
WORKDIR <- "A:/__FAENA__/_TESIS Co-Dirigides/Celestine Maguedong/experiments"  # Jose
setwd(WORKDIR)



























##########################################################################
############## MACRO-FUNCTION FOR ONE EXPERIMENT #########################
##########################################################################

# INPUTS: models_names, namedataset, alpha
 One_experiment<- function(models_names, namedataset, alpha, VERBOSE=FALSE, cost_configuration) {
# OUTPUT:
# c(jcfull, jcbmc, jcbtc, jcbjc, jcrnd)


# we DO re-train the models. 
lenmod <- length(models_names)



path<-""

# Dataset selection

# namedataset <- "iris"
# namedataset <- "diabetes"

if (namedataset != "iris") {
 mydata <- read.arff(paste(path,namedataset,".arff",sep=""))
 #mydata <- read.arff("diabetes.arff")
} else if (namedataset == "iris") {
  data(iris)
  mydata <- iris
}


# Gets length, number of attributes and number of classes
len <- length(mydata[,1])
nattr <- length(mydata[1,])
nclasses <- length(unique(mydata[,nattr]))


# Ensure that the name of the class attribute is called "class"
names(mydata)[nattr] <- "class"

# Shuffles datasets and splits it
shufindx <- sample(1:len, len)
mydata <- mydata[shufindx,]

lentrain <- trunc(len*2/3)
train <- mydata[1:lentrain,]
test <- mydata[(trunc(len*2/3)+1):len,]
lentest <- length(test[,1])




# Defines test cost vector and misc cost matrix
# There are several options depending on the experiment

# Uniform, as for definition 2
testcostvector_U <- rep(1/(nattr-1), nattr-1)   # nattr includes the class, so we need to remove 1
mcvector_U <- matrix(rep(nclasses/(nclasses-1), nclasses^2), nclasses, nclasses)
for (i in 1:nclasses)
  mcvector_U[i,i] <- 0

#testcostvector <- c(2,50,5,5,20,3,10,1)
#mcvector <- matrix(c(0,50,200,0),2,2)

# cost_configuration <- "U"
# cost_configuration <- "01"
# cost_configuration <- "02"

if (cost_configuration == "U") {
  testcostvector <- testcostvector_U
  mcvector <- mcvector_U
} else if (cost_configuration == "UR") {
  testcostvector <- testcostvector_U
  mcvector <- mcvector_U
  for (i in 1:length(testcostvector)) {
    k0 <- runif(1)
#    k = exp(2*(k0-0.5)) 
    k = exp(UR_FACTOR*(k0-0.5))

    testcostvector[i] <- testcostvector[i]*k
  }

  # normalise vector
  testcostvector <- testcostvector / sum(testcostvector)

  for (i in 1:length(mcvector[1,])) {
    for (j in 1:length(mcvector[,1])) {
      k0 <- runif(1)
#      k = exp(2*(k0-0.5)) 
      k = exp(UR_FACTOR*(k0-0.5))
      mcvector[i,j] <- mcvector[i,j]*k
    }
  }

  # normalise matrix
  mcvector <- mcvector / sum(mcvector)

#  print(testcostvector)
#  print(mcvector) 

} else if (cost_configuration == "01") {

  testcostvector_01_iris <- c(3, 2, 10, 5)
  mcvector_01_iris <- matrix(c(0,5,30,20,0,15,15,15,0),3,3)  # In Weka, rows are actual values and columns are predicted values

  testcostvector <- testcostvector_01_iris
  mcvector <- mcvector_01_iris
} else if (cost_configuration == "02") {
  testcostvector_02_diabetes <- c(2, 50, 5, 5, 20, 3, 10, 1)
  mcvector_02_diabetes <- matrix(c(0,200,50,0),2,2)  # In Weka, rows are actual values and columns are predicted values

  testcostvector <- testcostvector_02_diabetes
  mcvector <- mcvector_02_diabetes
}








# Training
models <- list()

for (i in models_names) {
  if (i == "J48") {
    models <- c(models, list(J48(class ~ ., data = train)))
  } else if (i == "SMO") {
    models <- c(models, list(SMO(class ~ ., data = train)))
  } else if (i == "IBk") {
    models <- c(models, list(IBk(class ~ ., data = train)))
  } else if (i == "BstDS") {
    models <- c(models, list(AdaBoostM1(class ~ ., data = train, control = Weka_control(W = "DecisionStump"))))
  } else if (i == "BstDT") {
    models <- c(models, list(AdaBoostM1(class ~ ., data = train, control = Weka_control(W = list(J48, M = 30)))))
  } else if (i == "BagDS") {
    models <- c(models, list(Bagging(class ~ ., data = train, control = Weka_control(W = "DecisionStump")))) 
  } else if (i == "BagDT") {
    models <- c(models, list(Bagging(class ~ ., data = train, control = Weka_control(W = list(J48, M = 30)))))
  } else {
    cat("\n\nERROR: model name unknown\n\n")
    abort()
  }
}




# Initialisation of lists and structures for main loop
mc1=0
mis1=0



# These are the lists where we keep the values for TC, MC, JC and accuracy
testcost <- rep(list(),lenmod)
misccost <- rep(list(),lenmod)
jointcost <- rep(list(),lenmod)
accuracylist <- list()


label <- "-"
labelnames <- rep(list(),lenmod)


modifiedtestvector <- list()

# objects use for the construction of box plot instances for the 3differents methods
boxplotstrg1 <- NULL
bp1 <- rep(list(), lenmod)

bpnames <- NULL

latticecounter <- 0


# objects for the incremental methods

tcpivotset <- NULL      # Attributes removed for tcpivotset (originally none)
tccounter <- rep(0,lenmod)   # corrected
tcselection <- list()

mcpivotset1 <- NULL      # Attributes removed for mcpivotset (originally none)
mccounter1 <- rep(0,lenmod)  # corrected
mcselection <- list()

jcpivotset1 <- NULL      # Attributes removed for jcpivotset (originally none)
jccounter1 <- rep(0,lenmod)  # corrected
jcselection <- list()

mcpivotsetnew<-list()
tcpivotsetnew<-list()
jcpivotsetnew<-list()



# The outer loop is for each row in the lattice (removing i attributes)
for (i in 0:(nattr-1)) {   # it goes from 0 (no attributes removed) to nattr-1 (all attributes removed)
                           # (Note that nattr includes the class, so nattr-1 is the actual number of attributes


  if (VERBOSE) cat("\n****Outer loop. Iteration: ", i, "of ", nattr-1, "****\n\n")

  # What's the meaning of these pivots?
  tcpivot <- Inf
  mcpivot1 <- Inf
  jcpivot1 <- Inf

  pivot <- 1
  bpnames[i+1] <- i
  boxplotstrg1 <- NULL
  label<-"_"

  # This calculates the combinations removing i attributes
  settonullmatrix <- t(combn(1:(nattr-1), i))
  if (VERBOSE) print(settonullmatrix)
  if (i == 0) {     
    len <- 0
  } else {
    len <- length(settonullmatrix[,1])
  }


  # The inner loop is for each element in the row in the lattice
  for (j in (min(1,i):len)) {

    latticecounter <- latticecounter + 1
    if (VERBOSE) cat("\n  ****Inner loop. Iteration: ", j, "of ", len, "   Points in the lattice: ", latticecounter, "of", 2^(nattr-1), "****\n\n")

    tcignore <- FALSE
    mcignore1 <- FALSE
    jcignore1 <- FALSE
	
k=1
for (mod in models)
{
  
  if (j==0) {
      modifiedtest <- test
      label <- "ALL"
      myvector <- NULL   
    } else {
      myvector <- settonullmatrix[j,]
      modifiedtest <- settonull(test, myvector)  

      newattr <- setdiff(myvector, tcpivotset)     # New attributes from the one fixed for tc in the previous lattice row
      if (length(newattr) != 1) {
        tcignore <- TRUE 
      }

      newattr <- setdiff(myvector, mcpivotset1)     # New attributes from the one fixed for mc in the previous lattice row
      if (length(newattr) != 1) {
        mcignore1 <- TRUE 
      }


      newattr <- setdiff(myvector, jcpivotset1)     # New attributes from the one fixed for jc in the previous lattice row
      if (length(newattr) != 1) {
        jcignore1 <- TRUE 
      }

    }

    if (!tcignore)
      tccounter[k] <- tccounter[k] + 1

    if (!mcignore1)
      mccounter1[k] <- mccounter1[k] + 1
   
    if (!jcignore1)
      jccounter1[k] <- jccounter1[k] + 1
    
    # do whatever you need to do with modifiedtest
	
    e <- evaluate_Weka_classifier(mod,newdata=modifiedtest, cost= mcvector, complexity = TRUE,seed = 123, class = TRUE )
    
    # computation of the accuracy with the DT
    # for the first model
    res1 <- predict(mod, newdata = modifiedtest)
#    print("result with DT")
#    print(e1)
#    summary(e1)
    hits <- 0
    for (i in 1:lentest) 
      if (res1[i] == modifiedtest[i,nattr])
        hits <- hits +1
    
    accuracy1 <- hits/ lentest
#    print(accuracy1)
    if (VERBOSE) print(k)
	if (length(accuracylist)<lenmod)
		accuracylist <- c(accuracylist, list(accuracy1))
		else
		accuracylist[[k]] <- c(accuracylist[[k]], list(accuracy1))
	mc1 <- misclascost(e[1])
	if (length(misccost)<lenmod)
		misccost <- c(misccost, list(mc1))
		else
		misccost[[k]] <- c(misccost[[k]], list(mc1))
	
	

  
	

    # computation of the total test cost (and creation of labels)
    tc <- 0
    if (j==0) {
      label <- "ALL"
    } else {
      label <- ""
    }
    for (i in 1:(nattr-1)) {
     if (modifiedtest[1,i] != "?") {             
	tc <- tc + testcostvector[i]
     }
    }
	if (length(testcost)<lenmod)
		testcost <- c(testcost, list(tc))
		else
		testcost[[k]] <- c(testcost[[k]],list(tc))
	
    jc1 <- alpha*mc1 + (1-alpha)*tc
  
	if (length(jointcost)<lenmod)
		jointcost <- c(jointcost,list(jc1))
		else
		jointcost[[k]] <- c(jointcost[[k]], list(jc1))



    if (!tcignore) {	# tcignore == TRUE means that this point cannot be derived from the previous row in the lattice
		if (length(tcselection)<lenmod)
		tcselection <- c(tcselection, latticecounter)
		else
		tcselection[[k]] <- c(tcselection[[k]], latticecounter)  # Index of the point

      if (tcpivot > tc) { 
        tcpivot <- tc
        tcpivotsetnew <- myvector   # Keeps track of the attributes of the best point for jc
      }
    }

    if (!mcignore1) {	 # mcignore1 == TRUE means that this point cannot be derived from the previous row in the lattice
     if (length(mcselection)<lenmod)
		mcselection <- c(mcselection, latticecounter)
		else
		mcselection[[k]] <- c(mcselection[[k]], latticecounter)  # Index of the point

      if (mcpivot1 > mc1) { 
        mcpivot1 <- mc1
        mcpivotsetnew <- myvector   # Keeps track of the attributes of the best point for jc
      }
    }

 


    if (!jcignore1) {	 # jcignore1 == TRUE means that this point cannot be derived from the previous row in the lattice
      if (length(jcselection)<lenmod)
		jcselection <- c(jcselection, latticecounter)
		else
		jcselection[[k]] <- c(jcselection[[k]], latticecounter)  # Index of the point

      if (jcpivot1 > jc1) { 
        jcpivot1 <- jc1
        jcpivotsetnew <- myvector   # Keeps track of the attributes of the best point for jc
      }
    }
    

	
    boxplotstrg1[pivot] <- jc1
    
    pivot <- pivot+1  


   if (VERBOSE) cat("model:", unlist(models_names[k]),"\n")
   tcpivotset <- tcpivotsetnew
   if (VERBOSE) {
     cat("\nBest TC: points explored: ", tccounter[k], "\n")
     cat("Best TC: attributes removed: ", tcpivotset, "\n")
   }


   mcpivotset1 <- mcpivotsetnew
   if (VERBOSE) {
     cat("\nBest MC: points explored: ", mccounter1[k], "\n") # Remains to be adjust with the others one
     cat("Best MC: attributes removed: ", mcpivotset1, "\n")
   }


   jcpivotset1 <- jcpivotsetnew
   if (VERBOSE) { 
     cat("\nBest JC: points explored: ", jccounter1[k], "\n")
    cat("Best JC: attributes removed: ", jcpivotset1, "\n")
   }
   k=k+1
  } # end of inner loop

    
	for (i in 1:(nattr-1)) {
     if (modifiedtest[1,i] == "?") {             # computation of the labelnames of the current modified test file.
       label <- paste(label,"-", i, sep="")
     }
    }
	labelnames <- c(labelnames, list(label))
	
	modifiedtestvector <- c(modifiedtestvector, list(modifiedtest))  # computation of the vector which contains all the differents modifiedtests
	


  if (VERBOSE)   cat("\n\n")
  

  
	bp1 <-c(bp1,list(boxplotstrg1))
	

}
  

} # end of outer loop


# The labels are converted into a vector x
x=NULL
for (i in 1:length(labelnames))
	x[i] <- labelnames[[i]]
	


# We calculate some max and min for the plots

maxmisccost <- max(unlist(misccost))


minmisccost <- min(unlist(misccost))


maxjointcost <- max(unlist(jointcost))


minjointcost <- min(unlist(jointcost))


maxtestcost <- max(unlist(testcost))
mintestcost <- min(unlist(testcost))

maxcost <- max(maxmisccost, maxjointcost, maxtestcost)
mincost <- min(minmisccost, minjointcost, mintestcost)

maxaccuracy <- max(unlist(accuracylist))


minaccuracy <- min(unlist(accuracylist))





misccosttc <- list()
misccostmc <- list()
misccostjc <- list()
testcosttc <- list()
testcostmc <- list()
testcostjc <- list()
misccostrnd <- list()
testcostrnd <- list()

##################### TC method incremental ##########################
for (i in 1:lenmod)
{
misccosttc <- c(misccosttc, list(misccost[[i]][tcselection[[i]]]))
testcosttc <- c(testcosttc,list(testcost[[i]][tcselection[[i]]]))
}





##################### MC method incremental ##########################

for (i in 1:lenmod)
{
misccostmc <- c(misccostmc, list(misccost[[i]][mcselection[[i]]]))
testcostmc <- c(testcostmc,list(testcost[[i]][mcselection[[i]]]))
}




##################### JC method incremental ##########################

for (i in 1:lenmod)
{
misccostjc <- c(misccostjc, list(misccost[[i]][jcselection[[i]]]))
testcostjc <- c(testcostjc,list(testcost[[i]][jcselection[[i]]]))
}




##################### Monte Carlo method ##########################

latticesize <- length(modifiedtestvector)
rndsamplesize <- (nattr-1)*(nattr)/2 + 1                                   

rndselection <- sample(1:latticesize, rndsamplesize, replace=FALSE)


# lists for the monte carlo method. All the lists end with rnd (rnd)
for (i in 1:lenmod)
{
misccostrnd <- c(misccostrnd, list(misccost[[i]][rndselection]))
testcostrnd <- c(testcostrnd, list(testcost[[i]][rndselection]))
}


labelnamesrnd <- labelnames[rndselection]


  
# The monte carlo labels are converted into a vector Y
Y=NULL
for (i in 1:length(labelnamesrnd))
	Y[i] <- labelnamesrnd[[i]]
  


point_full<-NULL
point_btc<-NULL
point_bmc<-NULL
point_bjc<-NULL
point_rnd<-NULL

  slope <- (1 - alpha) / - alpha
  
# FULL METHOD
  point_full <- leftmost_intercept_xy(slope, c(unlist(testcost)), c(unlist(misccost)) )
  jcfull <- alpha*point_full[2] + (1-alpha)*point_full[1]

# BMC METHOD
  point_bmc <- leftmost_intercept_xy(slope, c(unlist(testcostmc)), c(unlist(misccostmc) ))
  jcbmc <- alpha*point_bmc[2] + (1-alpha)*point_bmc[1]

# BTC METHOD
  point_btc <- leftmost_intercept_xy(slope, c(unlist(testcosttc)), c(unlist(misccosttc) ))
  jcbtc <- alpha*point_btc[2] + (1-alpha)*point_btc[1]

# BTC METHOD
  point_bjc <- leftmost_intercept_xy(slope, c(unlist(testcostjc)), c(unlist(misccostjc) ))
  jcbjc <- alpha*point_bjc[2] + (1-alpha)*point_bjc[1]
  
# BTC METHOD
  point_rnd <- leftmost_intercept_xy(slope, c(unlist(testcostrnd)), c(unlist(misccostrnd) ))
  jcrnd <- alpha*point_rnd[2] + (1-alpha)*point_rnd[1]

  c(jcfull, jcbmc, jcbtc, jcbjc, jcrnd)  # Returns the five joint costs for the five different methods


}
##########################################################################
############## END MACRO-FUNCTION FOR ONE EXPERIMENT #####################
##########################################################################












##########################################################################
############## We perform the experiments with several alphas, datasets and repetitions

models_names <- list()

# models_names <- c(models_names, "J48")  # "DT"
models_names <- c(models_names, "SMO")  # "SVM"
models_names <- c(models_names, "IBk")  # "kNN"
# models_names <- c(models_names, "BstDS")  
models_names <- c(models_names, "BstDT") 
# models_names <- c(models_names, "BagDS")  
models_names <- c(models_names, "BagDT")  


# Datasets
#datasetnames <- c("glass")
datasetnames <- c("diabetes")
#datasetnames <- c("iris")
#datasetnames <- c("iris","breast-w","breast-cancer","diabetes","glass","balance-scale") # vector which contains the dataset names (8 datasets at final) on which we will performed the one-experiment function
num_datasets <- length(datasetnames)

# ALPHAS THAT WE WILL USE FOR THE EXPERIMENTS
ALPHAS_ISO <- c(0.5)
#ALPHAS_ISO <- c( 0.1, 0.3, 0.5, 0.7, 0.9)
num_alphas <- length(ALPHAS_ISO)

# Repetitions
REPETITIONS <- 4

# Methods
#methods <- c("jcfull", "jcbmc","jcbtc","jcbjc","jcrnd")
methods <- c("Full", "BMC","BTC","BJC","RND")
num_methods <- length(methods)


# Cost configuration
# cost_configuration <- "U"

cost_configuration <- "UR"
UR_FACTOR <- 10

# cost_configuration <- "01" # Only for iris
# cost_configuration <- "02" # only for diabetes


NUM_EXPERIMENTS <- num_alphas * num_datasets * REPETITIONS

# Matrices
mdat <- matrix(data=NA, nrow = num_methods, ncol = NUM_EXPERIMENTS, byrow = TRUE,
               dimnames = list(methods			   ))

tab <- matrix(data=NA, nrow = num_datasets, ncol = num_methods, byrow = TRUE,    
               dimnames = list(c(1:num_datasets),methods			   ))            # matrix which will contains all the means and sd for the datasets

tab2 <- matrix(data=NA, nrow = num_alphas, ncol = num_methods, byrow = TRUE,    
               dimnames = list(ALPHAS_ISO,methods			   ))                # matrix which will contains all the means and sd for the datasets

tab3 <- matrix(data=NA, nrow = num_alphas*num_datasets, ncol = num_methods, byrow = TRUE,    
               dimnames = list((1:(num_alphas*num_datasets)),methods			   ))                # matrix which will contains all the means and sd for the datasets



# Experiments begin

cat("\n\n\n\n\n\n#################### EXPERIMENTS BEGIN!!! ##############\n\n")


k=0

for (dataset in datasetnames) {
  cat("  Dataset: ", dataset, "\n")
  result=NULL
  for (i in ALPHAS_ISO) {
    cat("    Alpha: ", i, "\n")
    for(j in 1:REPETITIONS) {
      k<- k+1

      cat("      Repetition: ", j, "\n")
      cat("      Experiment num.: ", k, "of", NUM_EXPERIMENTS, "\n")

      result <- One_experiment(models_names, dataset, i, VERBOSE=FALSE, cost_configuration) # performs the one-experiment 10 times with the same value of alpha but with different partition of the data set.
      mdat[,k]<-result
      # remains construct the matrix which contains the sd and medians values
    }
  }

  d=d+1
}
cat("\n\n#################### EXPERIMENTS END!!! ##############\n\n")




cat("We now calculate matrices...")

# First matrix
for (d in 1:num_datasets) {
  for (j in 1: num_methods) {
    range0 <- (d-1)*num_alphas*REPETITIONS + 1
    range1 <- (d)*num_alphas*REPETITIONS
    my_subset <- range0:range1

    v <- as.vector(t(mdat[j,my_subset]))
    my_mean <- mean(v)
    my_sd <- sd(v)
    tab[d,j]<-paste(my_mean,"+/-",my_sd,sep="")
  }
}

# Second matrix
for (a in 1:num_alphas) {
  for (j in 1: num_methods) {
    my_subset <- NULL
    for (d in 1:num_datasets) {
      range0 <- (d-1)*REPETITIONS*num_alphas + (a-1)*REPETITIONS + 1
      range1 <- (d-1)*REPETITIONS*num_alphas + (a)*REPETITIONS
      my_subset <- c(my_subset, range0:range1)
    }
    #   1:4, 9:14 17:21
    print(my_subset)
    v <- as.vector(t(mdat[j,my_subset]))
    my_mean <- mean(v)
    my_sd <- sd(v)
    tab2[a,j]<-paste(my_mean,"+/-",my_sd,sep="")
  }
}

# Third matrix

for (i in 1:(num_alphas*num_datasets)) {
  for (j in 1:num_methods) {
    range0 <- (i-1)*REPETITIONS + 1
    range1 <- (i)*REPETITIONS
    my_subset <- range0:range1

    v <- as.vector(t(mdat[j,my_subset]))
    my_mean <- mean(v)
#    my_sd <- sd(v)
    tab3[i,j]<- my_mean  # paste(my_mean,"+/-",my_sd,sep="")
  }
}

tabmeans <- NULL
# Means row
for (j in 1:num_methods) {
  tabmeans[j] <- mean(tab3[,j])
}

tab3means <- rbind(tab3, Avg= tabmeans)


#save(datasetnames,ALPHAS_ISO, models_names, cost_configuration, mdat, tab, tab2, tab3, tabmeans, tab3means, file="C:/Users/periale/Desktop/tesis/images/resultsFile")
save(datasetnames,ALPHAS_ISO, models_names, cost_configuration, mdat, tab, tab2, tab3, tabmeans, tab3means, file="resultsFile")
# load("resultsFile")

print(mdat)
# write.matrix(mdat,file="C:/Users/periale/Desktop/tesis/images/mdat", sep="\t")
write.table(mdat,file="mdat.csv")
# mdat <- read.table(file="mdat.csv")

print(tab)
#write.matrix(tab,file="C:/Users/periale/Desktop/tesis/images/tab", sep="\t")
write.table(tab,file="tab.csv")
# tab <- read.table(file="tab.csv")

print(tab2)
#write.matrix(tab2,file=paste("C:/Users/periale/Desktop/tesis/images/tab2", sep="\t"))
write.csv(tab2,file="tab2.csv")
# tab2 <- read.csv(file="tab2.csv")


\end{lstlisting}
\subsection{Statistical tests}
This part of code implemented the computation of the average ranks of the result obtained in the previous part, from which the Friedman and Nemenyi statistic tests are calculated.

\begin{lstlisting}
 
################################################
####         STATISTICAL TESTS             #####
################################################
 






library(xtable)     # for xtable
library(SuppDists)  # for Friedman test (qFriedman)




# Calculate ranks for the columns and derives the rank
CalculateRanks <- function(v) {

  # Calculate the ranks
  ranks <- v
  lenv <- length(v[,1])
  widthv <- length(v[1,])
  for (i in 1:(lenv-1)) {  # -1 to exclude the mean
    my_row <- as.vector(t(v[i,]))  # This is done in case v is not a matrix but a data frame
#    s <- sort(my_row, index.return=TRUE)   # Doesn't work well with ties
#    ranks[i,s$ix] <- 1:widthv
    ranks[i,] <- rank(my_row)
  }
  for (j in 1:widthv) {
    ranks[lenv,j] <- mean(ranks[1:(lenv-1),j])
  }

 ranks
}


# Add the avg rank row to the matrix
AddRanks <- function(v, ranks) {

  lenv <- length(ranks[,1])
  vnew <- rbind(v, ranks[lenv,])
  rn <- rownames(v)
  rn <- c(rn, "AR") # "AvgRk")
  rownames(vnew) <- rn

  # Return matrix
  vnew
}  


# Calculate Friedman and Nemenyi tests
Tests <- function(ranks, alpha) {
  # alpha <- 0.05  # significance level

  lenv <- length(ranks[,1])  # This includes the mean at the end
  widthv <- length(ranks[1,])

  # Calculate Friedman statistic
  R <- mean(ranks[lenv,])
  n <- lenv-1 # Number of datasets. (-1) because the matrix also includes the average row
  k <- widthv  # Number of algorithms
  Sum1 <- n * sum((ranks[lenv,]-R)^2)
  Sum2 <- 1/(n*(k-1)) * sum((ranks[1:n,1:k]-R)^2)
  Friedman.stat <- Sum1/Sum2
  #CriticalValue <- qFriedman(0.975, 3, 10)  # Example: alpha= 0.05 (two-tail), k=3, n=10
  CriticalValue <- qFriedman(1-(alpha/2), k, n) 
  
  df <- (n-1)*(k-1)  # degrees of freedom
#  qa <- qtukey(0.95, 3, 18) / sqrt(2) # Example: alpha: 0.05, k=3, df= 18
  qa <- qtukey(alpha, k, df) / sqrt(2)

  Nemenyi_CD <- qa * sqrt((k*(k+1))/(6*n))

  if (Friedman.stat > CriticalValue) {
    text1 <- paste("Friedman statistic: ", format(Friedman.stat, digits=4), " >  Critical Value: ", format(CriticalValue, digits=4), "  Null hypothesis rejected (significance level: ", toString(alpha), "). Algorithms do not perform equally.", sep="")
  } else {
    text1 <- paste("Friedman statistic: ", format(Friedman.stat, digits=4), " <  Critical Value: ", format(CriticalValue, digits=4), "  Null hypothesis not rejected (significance level: ", toString(alpha), "). Algorithms may perform equally.", sep="")
  }
  text2 <- paste("Critical difference for the Nemenyi post-hoc test: ", format(Nemenyi_CD, digits=4), sep="")

  c(text1, text2)
}


mytab <- tab3means
# mytab <- tab3means[,2:num_methods]  # We eliminate the perfect method
# mytab <- tab3means[,1:(num_methods-1)]  # We eliminate the worst method


ranks1 <- CalculateRanks(mytab)
tab3complete <- AddRanks(mytab, ranks1)
resTests1 <- Tests(ranks1, 0.05)  # 0.05 confidence level

tab3complete
resTests1


Tab <- xtable(tab3complete, digits=4, caption= paste("This figure shows the JC means for the 4 repetitions for each of the 5 methods (Full, BMC, BTC, BJC, RND). The 30 rows are given by 6 datasets and 5 possible values of $alpha$. The `Avg' row shows the averages of the first 30 rows. Finally, the `AR' shows the average rank for each method. With these ranks the Friedman test is applied: ", paste(resTests1, collapse = ''))) 
print(Tab, file= paste("",  "table3.tex", sep=""))



\end{lstlisting}

\addcontentsline{toc}{chapter}{Bibliography}
{
\bibliographystyle{plain}
\bibliography{biblio}
}
\end{document}